\documentclass[10pt,twocolumn,twoside,romanappendices]{IEEEtran}

\usepackage{amsmath}
\usepackage{amssymb}
\usepackage{graphicx}
\usepackage{subcaption}
\usepackage{amsthm}
\usepackage{cite}
\usepackage{float}
\usepackage{bm}
\usepackage{bbm}
\usepackage{url}
\usepackage{algorithmicx}
\usepackage{algorithm}
\usepackage[noend]{algpseudocode}
\usepackage{authblk} 
\usepackage{setspace}

\usepackage{color}

\def\T{\mathsf{T}}
\def\H{\mathsf{H}}

\def\Cbb{\mathbb{C}}
\def\Rbb{\mathbb{R}}

\def\xbf{{\mathbf{x}}}
\def\ybf{{\mathbf{y}}}
\def\zbf{{\mathbf{z}}}
\def\gbf{{\mathbf{g}}}

\def\rbf{{\mathbf{r}}}

\def\sbf{{\mathbf{s}}}

\def\nbf{{\mathbf{n}}}
\def\hbf{{\mathbf{h}}}

\def\Dbm{{\bm{D}}}
\def\Rbm{{\bm{R}}}
\def\ebm{{\bm{e}}}
\def\dbm{{\bm{d}}}
\def\xbm{{\bm{x}}}
\def\rbm{{\bm{r}}}
\def\lbm{{\bm{l}}}

\def\e{{\mathrm{e}}}

\def\z{{\, \mathrm{z}}}

\def\s{{\, \mathrm{s}}}
\def\u{{\, \mathrm{u}}}
\def\v{{\, \mathrm{v}}}

\def\Xbf{{\mathbf{X}}}

\def\Bbf{{\mathbf{B}}}

\def\Dbf{{\mathbf{D}}}
\def\Cbf{{\mathbf{C}}}
\def\Fbf{{\mathbf{F}}}
\def\Ebf{{\mathbf{E}}}

\def\Abf{{\mathbf{A}}}

\def\Acal{\mathcal{A}}
\def\Bcal{\mathcal{B}}

\def\Fcal{\mathcal{F}}

\def\Scal{\mathcal{S}}

\theoremstyle{theorem}
\newtheorem{proposition}{Proposition}

\begin{document}


\title{Sparse Blind Deconvolution for Distributed Radar Autofocus Imaging}


\author{Hassan~Mansour,~\IEEEmembership{Senior~Member,~IEEE}, Dehong~Liu,~\IEEEmembership{Senior~Member,~IEEE}, Ulugbek~S.~Kamilov,~\IEEEmembership{Member,~IEEE}
and Petros~T.~Boufounos,~\IEEEmembership{Senior~Member,~IEEE}
\thanks{H.~Mansour, D. Liu and P. Boufounos are with Mitsubishi Electric Research Laboratories (MERL), 201 Broadway, Cambridge,
MA 02139, USA (email: mansour@merl.com; liudh@merl.com; petrosb@merl.com).\\
U.~S.~Kamilov is with the Washington University in St. Louis, St. Louis, MO 63130, USA (kamilov@wustl.edu). U.~S.~Kamilov contributed to this work while he was at MERL.}
\thanks{A preliminary version of this work appeared in ICASSP 2018~\cite{MKLB:2018}.}
}%

\markboth{Radar Autofocus Using Sparse Blind Deconvolution}%
{Mansour et al.}

\maketitle

\begin{abstract}
A common problem that arises in radar imaging systems, especially
those mounted on mobile platforms, is antenna position ambiguity. Approaches to
resolve this ambiguity and correct position errors are generally known
as radar autofocus. Common techniques that attempt to resolve the antenna ambiguity generally assume an unknown gain and phase error afflicting the radar measurements. However, ensuring identifiability and tractability of the unknown error imposes strict restrictions on the allowable antenna perturbations. Furthermore, these techniques are often not applicable in near-field
imaging, where mapping the position ambiguity to phase
errors breaks down.

In this paper, we propose an alternate formulation where the position error of each antenna is mapped to a spatial shift
operator in the image-domain. Thus, the radar autofocus problem
becomes a multichannel blind deconvolution problem, in which the radar
measurements correspond to observations of a static radar image that
is convolved with the spatial shift kernel associated with each
antenna. To solve the reformulated problem, we also develop a block
coordinate descent framework that leverages the sparsity and
piece-wise smoothness of the radar scene, as well as the one-sparse
property of the two dimensional shift kernels. We evaluate the
performance of our approach using both simulated and experimental
radar measurements, and demonstrate its superior performance compared
to state-of-the-art methods.
\end{abstract}


\graphicspath{{./figures/}}

\IEEEpeerreviewmaketitle


\section{Introduction}
\label{sec:Intro}

High resolution radar imaging has become essential in a variety of
remote sensing applications. While the resolution of the imaging
platform in the down-range direction is primarily a function of the
frequency and bandwidth of the transmitted pulse, the resolution along
the cross-range (azimuth) direction depends on the aperture, i.e., the
size of the radar array. In order to achieve the large aperture
required in modern applications, practical systems deploy and
distribute over a large area one or more, often mobile, antennas or
antenna arrays, each having a relatively small aperture. The larger
aperture is achieved by the virtual array formed over the large area
of deployment and motion of the antennas.

A multi-antenna distributed setup further reduces the operational and
maintenance costs, allows for flexibility of platform placement, and
provides robustness to sensor failures. Leveraging prior knowledge of
the scene, such as sparsity, the precise knowledge of the antenna
positions along with a full synchronization of received signals has
been shown to significantly improve the radar imaging
resolution~\cite{HermanStrohmer:2009,YPP:2010,BergerMoura:2011,LKB:2015}.

A fundamental challenge that arises in distributed array imaging is
the uncertainty in the exact positions of the transmitting and
receiving antennas. Advanced positioning and navigation systems, such
as the global navigation satellite system (GPS/GNSS) and the inertial
navigation system (INS), generally provide reasonably accurate but not
exact location information. The remaining uncertainty in the true
antenna positions can still span multiple wavelengths. As a result, if
the inexact antenna positions are used as-is in the imaging
process, the received signal is distorted, effectively contaminated
with a gain and phase ambiguity. Consequently, if the position
perturbation is not compensated for, conventional reconstruction
techniques produce out-of-focus radar images. A detailed analysis of
antenna position errors in bistatic synthetic arrays can be found
in~\cite{WangYaziciYanik:2015}.

There is extensive literature addressing the radar autofocus problem
by developing tools that compensate for antenna position
errors~\cite{WEGJ:1994,XGN:1999,YYB:1999,ChoMunson:2008,LiuMunson:2011,NA:2013}. In
some cases, the underlying structure of the radar image, such as its
sparsity, is exploited to limit the solution space and produce higher
quality
reconstructions~\cite{OnhonCetin:2012,DDH:2013,KellyYaghoobiDavies:2014,YHTJZ:2014,LKB:2016,ZHAOetal:2016,Hasankhan2017}. Instead
of estimating the position error, most techniques estimate an
equivalent set of frequency-domain gain and phase errors in the
measured signal, i.e., model the effect of the position error as a
linear time-invariant filter. One advantage of these techniques is
that they can often be directly combined with conventional imaging
methods in the processing pipeline. However, converting the problem to
a phase recovery one often poses severe restrictions on the
applicability of solutions in practical applications.

\begin{figure}[t]
\centering
\includegraphics[width=3.5in]{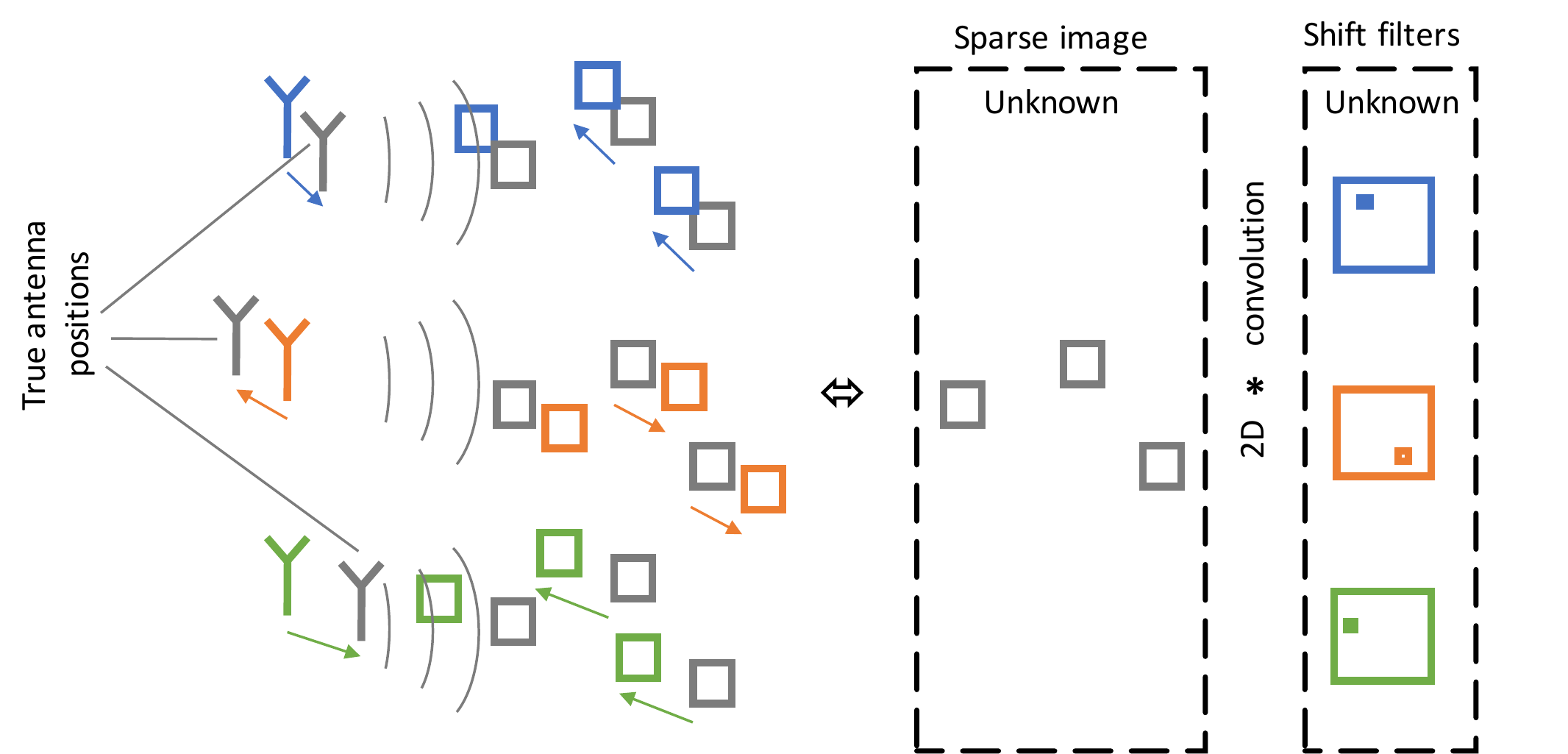}
\caption{\small Position ambiguity of the radar antennas induces an image-domain convolution model.}\label{fig:2DconvIllust}\vspace{-0.2in}
\end{figure}

\subsection{Main contributions}
In contrast to earlier approaches, our work fundamentally re-examines
the acquisition model. Specifically, we study the general high
resolution radar imaging problem of recovering a sparse stationary
image of a scene under position ambiguity of the antennas. We
demonstrate in Section~\ref{sec:Problem} that determining the antenna position ambiguity is
equivalent to a blind deconvolution problem in the image
domain. In particular, we assert that measurements acquired using an
antenna in an assumed position, but with unknown position error, are
equal to measurements of the scene convolved with a two-dimensional
unknown shift kernel, using an antenna without position
error. Figure~\ref{fig:2DconvIllust} illustrates the key
intuition of our formulation. The unknown kernel shifting the scene
mirrors the unknown shift error of the antenna, relative to its
assumed position. In addition, we prove that the image-domain
convolution model is exact if the transmitting and receiving antennas
are affected by the same position error, a common occurrence in
practice, for example if transmission and reception uses the same
antenna or if both antennas are mounted on the same platform.

In addition to our novel formulation, we also propose in Section~\ref{sec:ProposedSolution} a block-coordinate descent algorithm to efficiently
solve the sparse blind deconvolution problem in the image domain. We
validate our approach using numerical simulations and experimental
results that demonstrate their
effectiveness. Figure~\ref{fig:imaging_comparison} provides an example
of the improvements due to our approach. In this example, the scene in
Figure~\ref{fig:imaging_comparison}(a), including three reflective
objects, is imaged using a distributed array with 4 mobile antennas
and position perturbations. Ignoring position errors produces
significant artifacts, as shown in
Figure~\ref{fig:imaging_comparison}(b). While conventional methods
improve imaging performance, as shown in
Figure~\ref{fig:imaging_comparison}(c), there are still missing targets
and artifacts. Our approach is able to recover the scene, as
demonstrated in Figure~\ref{fig:imaging_comparison}(d). The performance
improvement is consistent in a variety of sensing conditions and very
large position errors. More details on this and other experiments are
provided in Section~\ref{sec:Results}.
\begin{figure}[t]
\centering
\mbox{
\begin{subfigure}[b]{0.2\textwidth}
                \centering
                \includegraphics[width=\textwidth]{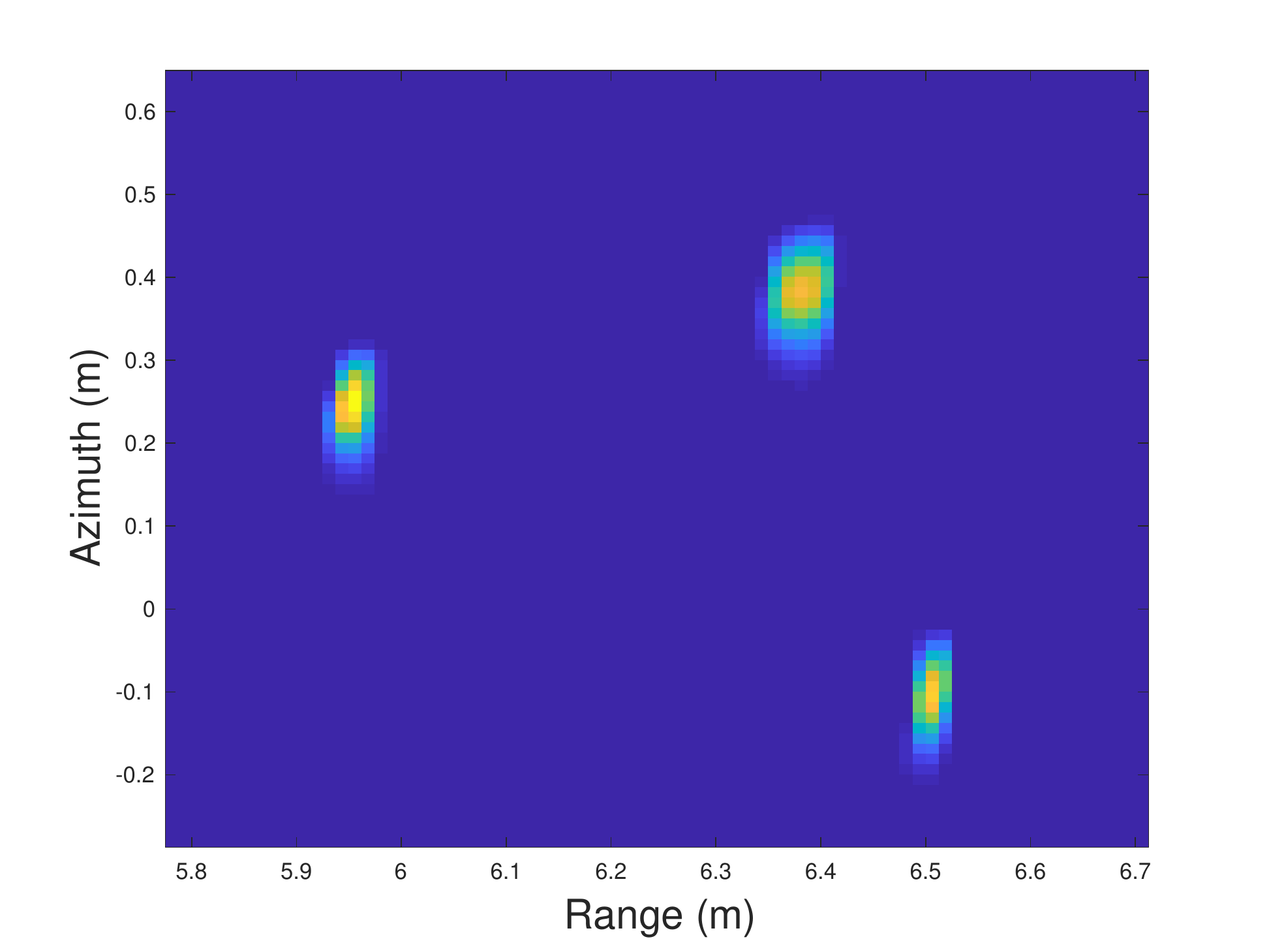}
                \caption{}
        \end{subfigure}%
        \begin{subfigure}[b]{0.2\textwidth}
                \centering
                \includegraphics[width=\textwidth]{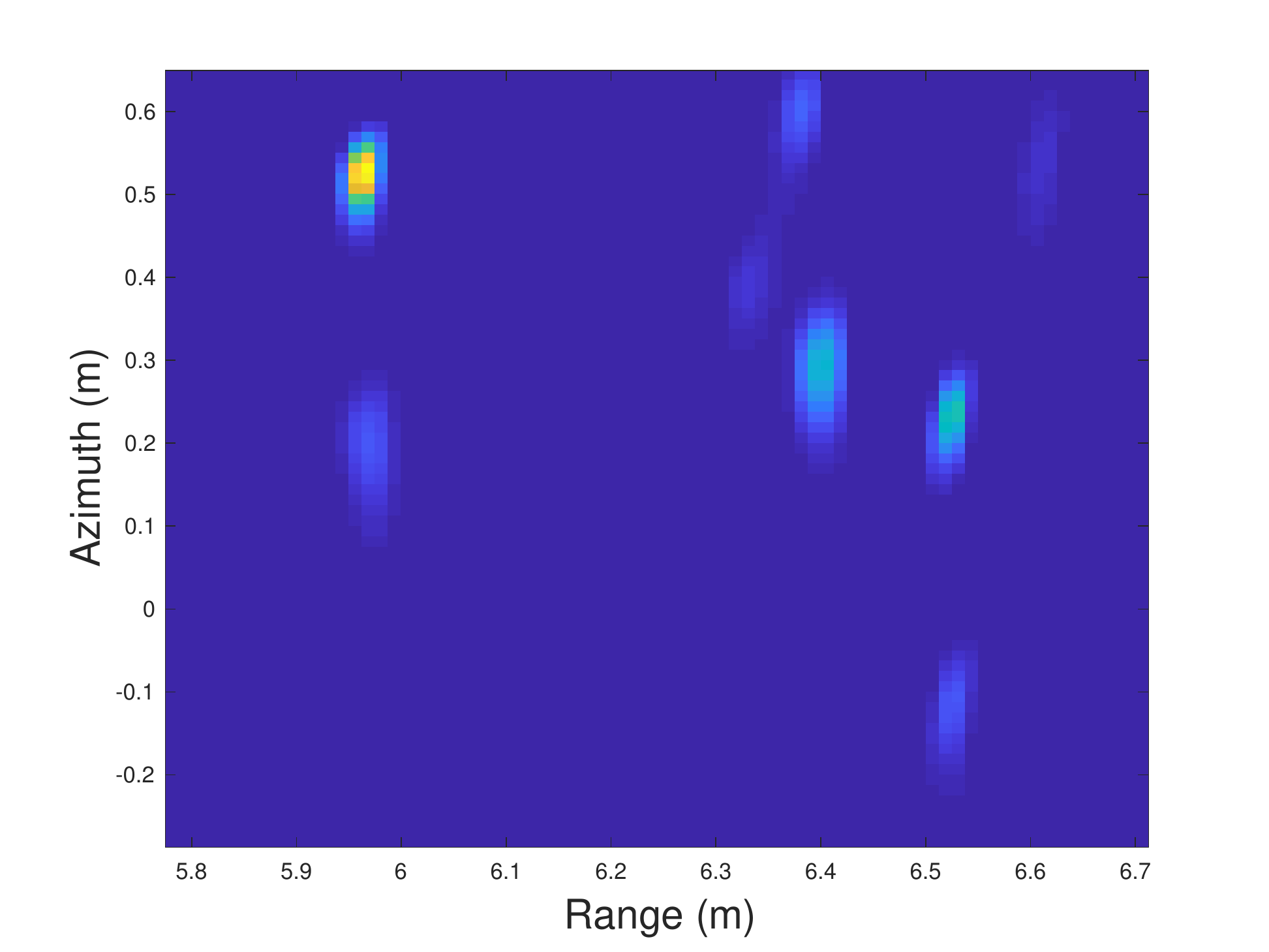}
                \caption{}
        \end{subfigure}%
}
\mbox{
\begin{subfigure}[b]{0.2\textwidth}
                \centering
                \includegraphics[width=\textwidth]{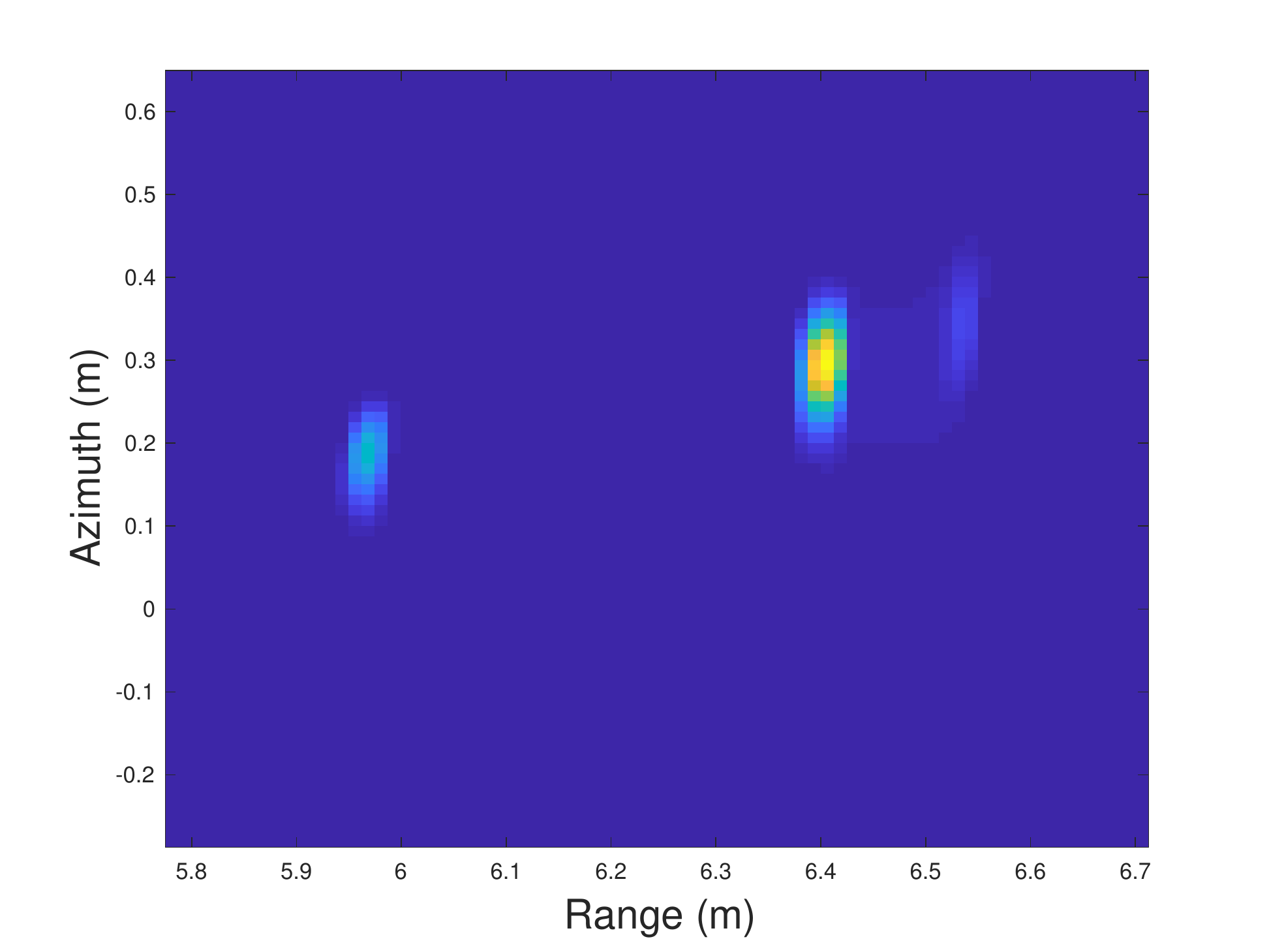}
                \caption{}
        \end{subfigure}%
        \begin{subfigure}[b]{0.2\textwidth}
                \centering
                \includegraphics[width=\textwidth]{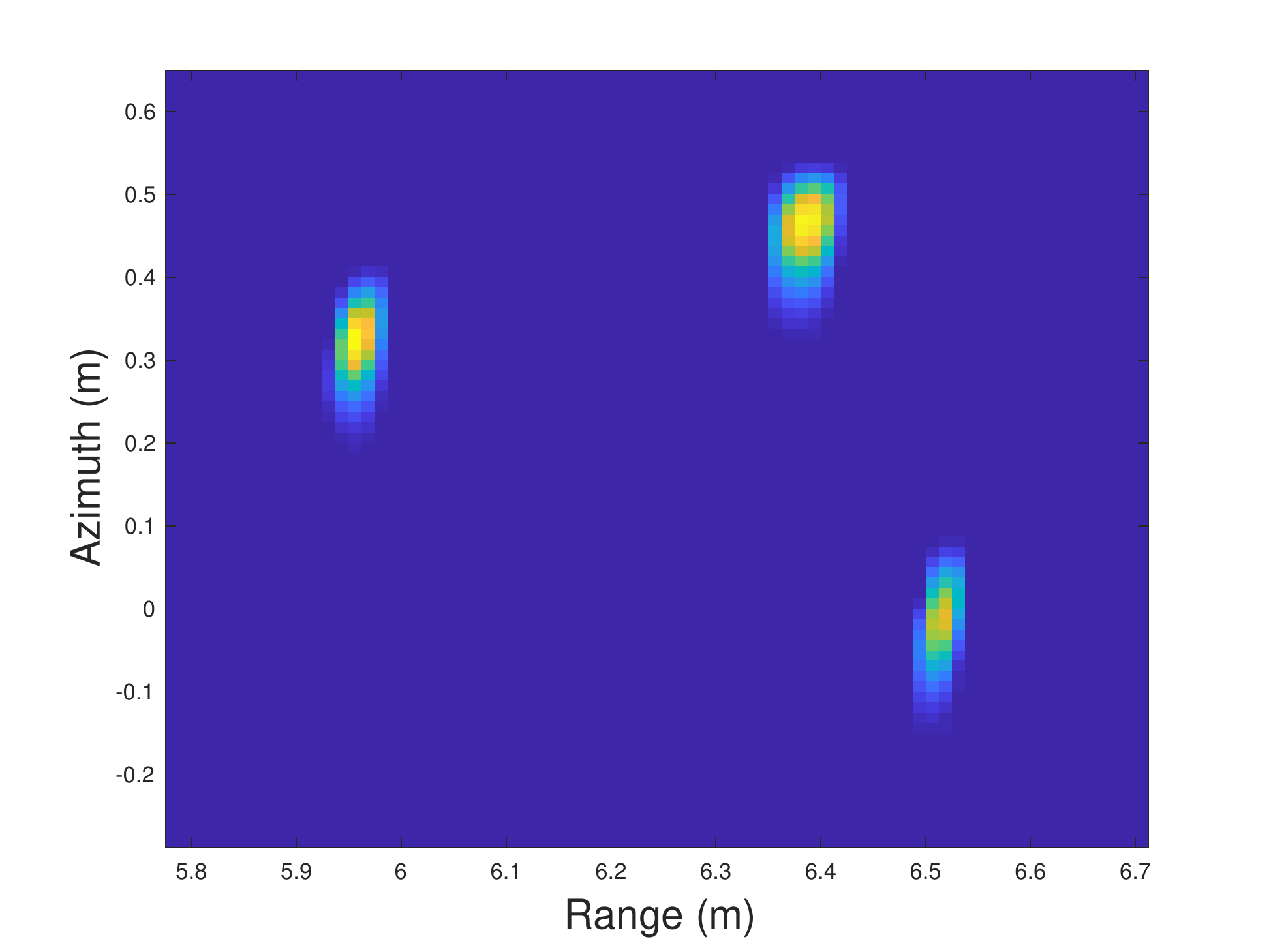}
                \caption{}
        \end{subfigure}%
}
\caption{\small(a)--(b) Sparse recovery of radar images without autofocusing obtained using (a) the correct antenna positions, and (b) the incorrect antenna positions. (c)--(d) Autofocus results using blind deconvolution according to (c) the time-domain convolution model; and (d) using the proposed sparse blind deconvolution and the image-domain convolution model.}\label{fig:imaging_comparison}\vspace{-0.2in}
\end{figure}

\subsection{Background and Related Work}
In its most general formulation, using an arbitrary gain and phase
error to model the data error in the frequency domain, the radar
autofocus problem is ill-posed. The overwhelming majority of the
literature addresses this ill-posedness by imposing constraints on the
dimensionality of the error. These constraints include constant phase error
assumptions or restricted subspace
assumptions~\cite{YYB:1999,LiuWieselMunson:2012,ZXXetal:2013,UUR2015,EZJ:2017},
and sparsity
assumptions~\cite{OnhonCetin:2012,KellyYaghoobiDavies:2014,ZHAOetal:2016,LKB:2016,Hasankhan2017}. While
some of these assumptions may be valid under specific conditions of
the measurement process, they do not necessarily apply to the general
radar imaging problem. 

In the context of autofocus using synthetic arrays, well established
approaches include phase gradient autofocus (PGA)~\cite{WEGJ:1994} and
sharpness optimization autofocus
techniques~\cite{XGN:1999,MDM:2007,BNCC:2008,ZYXLH:2016}. The PGA
approach estimates the derivative of the phase error using a minimum
variance estimator. The underlying assumption is that individual
targets in the scene exhibit the same phase error. Alternatively,
contrast optimization techniques, optimize a sharpness function of the
reconstructed image. The radar imaging operators of each measurement
are assumed to be invertible and the phase errors are assumed to
impact individual columns in the radar image. 

Another approach championed
in~\cite{LiuMunson:2011,LiuWieselMunson:2012,LiuWieselMunson:2013} is
the multichannel autofocus (MCA) and Fourier domain multichannel
autofocus (FMCA). Both techniques develop solutions that map the
autofocus problem into a constant modulus quadratic program (CMQP)
that is solved using subspace identification methods or semidefinite
relaxation. In these approaches, it is assumed that the measurements
are acquired at specific look angles. For each look angle, the
measurements are impacted by an unknown delay, resulting in a constant
unknown phase shift under a narrow band assumption. The delays, and
their associated phases, change between different look angles. Such a
model has two limitations for high resolution radar imaging. The first
limitation is the assumption that a radar measurement from one antenna
position receives reflections from a single ray along a single scan
angle. In the general radar imaging setting, a single radar
measurement can receive reflections from a large spatial area in the
scene. The second limitation is related to the first and materializes
in the form of restricting the phase ambiguity to a spatial shift
along the radial line. In fact, we prove in this paper that, in the
general radar setting, the phase ambiguity is realized instead as a
two-dimensional spatial shift of the radar image.

\if 
The overwhelming majority of the literature on radar autofocus address
the ill-posedness of~\eqref{eq:Model1intro} by imposing constraints
that limit the dimensionality of the vectors $\widehat{\gbf}_m$. These
constraints include constant phase assumptions or restricted subspace
assumptions~\cite{YYB:1999,LiuWieselMunson:2012,ZXXetal:2013,UUR2015,EZJ:2017},
and sparsity
assumptions~\cite{OnhonCetin:2012,KellyYaghoobiDavies:2014,
  ZHAOetal:2016,LKB:2016,Hasankhan2017}. While some of these
assumptions may be valid under specific conditions of the measurement
process, they do not necessarily apply to the general radar imaging
problem. In the context of autofocus using synthetic arrays, well
established approaches include phase gradient autofocus
(PGA)~\cite{WEGJ:1994} and sharpness optimization autofocus
techniques~\cite{XGN:1999,MDM:2007,BNCC:2008,ZYXLH:2016}. The PGA
approach estimates the derivative of the phase error using a minimum
variance estimator. The underlying assumption is that individual
targets in the scene exhibit the same phase error. Contrast
optimization techniques, on the other hand, optimize a sharpness
function of the reconstructed image. The radar imaging operators of
each measurement are assumed to be invertible and the phase errors are
assumed to impact individual columns in the radar image. Another
approach championed by Munson et
al.~\cite{LiuMunson:2011,LiuWieselMunson:2012,LiuWieselMunson:2013} is
the multichannel autofocus (MCA) and Fourier domain multichannel
autofocus (FMCA). Both techniques develop solutions that map the
autofocus problem into a constant modulus quadratic program (CMQP)
that is solved using subspace identification methods or semidefinite
relaxation. In these approaches, it is assumed that the measurements
are acquired at specific look angles. For each look angle, the
measurements are impacted by an unknown delay, resulting in a constant
unknown phase shift under a narrow band assumption. The delays, and
their associated phases, change between different look angles. Such a
model has two limitations for high resolution radar imaging. The first
limitation is the assumption that a radar measurement from one antenna
position receives reflections from a single ray along a single scan
angle. In the general radar imaging setting, a single radar
measurement can receive reflections from a large spatial area in the
scene. The second limitation is related to the first and materializes
in the form of restricting the phase ambiguity to a spatial shift
along the radial line. We prove in this paper that, in the general
radar setting, the phase ambiguity is realized instead as a
two-dimensional spatial shift of the radar image.
\fi

More generally, the radar autofocus problem can be viewed as a
multichannel blind deconvolution problem, where the distortions due to
position errors can be considered as unknown channels acting on the
signals that are received in each antenna. Blind deconvolution attempts to
separate a signal and a channel from their convolved mixture, i.e.,
their product in the frequency domain, using assumptions on each of
the two. The multi-channel version considers multiple channels
affecting the same signal. Solutions to the general multichannel blind deconvolution problem have
existed in the literature since the
1990s~\cite{XLTK:1995,MDCM:1995,GN:1995,AQH:1997,TP:1998}. 

More recent research has focused on convex formulations and recovery
guarantees. The seminal work of Ahmed, Recht, and
Romberg~\cite{BlindDeconvConvex_ARR2013} proposed a convex formulation
of the single-channel problem, recovering the two signals by lifting
the problem to a frequency-domain outer-product space. Under this
formulation, the acquired data are linear measurements of an unknown
rank-one matrix to be recovered. The
unknown channel and signal vectors are obtained by factoring this
rank-one matrix to its components. The authors presented a theoretical
analysis of the recovery guarantees of the lifted problem when both
unknowns exist in low dimensional subspaces, and for a generic linear
acquisition system. The analysis was later generalized by Li and
Strohmer~\cite{SelfCalibratBiconvex_LS2015} using the
\emph{SparseLift} technique to the case where the signal is sparse and
exists on a union of subspaces spanned by an unknown subset of the
columns of a dictionary that has measurement concentration
properties. Along a similar track, Li, Lee, and
Bresler~\cite{LiLeeBresler:2016,LiLeeBresler:2017} analyzed the
identifiability properties of these bilinear inverse problems. The
problem of recovering a sparse signal from diverse convolutive
measurements was also studied by Ahmed and Demanet~\cite{AD:2016}
under the assumption that the convolution filters are members of known
low-dimensional random subspaces.

Another convex approach that has been studied
in~\cite{BN:2007,Balzano2008,GCD:2012,BPGL:2014,BlindDeconvLLS:2016,WC:2016}
recasts blind deconvolution as a linear least squares problem, thus
enjoying lower computational complexity compared to the lifting
approach when certain conditions are satisfied. Finally, we point out
that other nonconvex formulations with recovery guarantees have been
studied in~\cite{LLSW:2016,LTR:2016,CJ:2016,CJarxiv:2016}. 
A key difference in this prior work and our model is that, in the
terminology of blind deconvolution, both our signals and our channels
are sparse in the same domain, thus violating most common assumptions
necessary for recovery guarantees. On the other hand, our algorithm
is able to exploit some additional diversity due to the multi-channel
nature of our formulation.

\if

To summarize, none of the aforementioned works address the general
model in~\eqref{eq:Model2intro} for the case when both $\xbf$ and
$\hbf_m$'s are sparse, and when the matrices $\Abf_m$ are not generic
Gaussian or sub-Gaussian random matrices.


We are interested in the problem of recovering an image of a stationary scene composed of a small number of targets. Formally, let $\xbf \in \Cbb^N$ be the vector representation of the sparse radar image. The image is to be recovered by processing $F-$dimensional frequency-domain measurements $\{\widetilde{\ybf}_m\}_{m=1}^M$ from $M$ distributed antennas that suffer from position ambiguity. We focus our attention to two-dimensional radar imaging in this paper but our analysis extends to three-dimensional imaging as well. We develop an image reconstruction framework wherein a perturbation in the antenna positions results in a measured signal that corresponds to an image-domain convolution model as illustrated in Figure~\ref{fig:2DconvIllust}. More precisely, if we denote the radar propagation matrix at the \emph{correct} antenna positions by $\widetilde{\Abf}_m$, and denote by $\Abf_m$ the corresponding matrix at the \emph{incorrect} positions, then we have $\widetilde{\ybf}_m = \widetilde{\Abf}_m \xbf \neq \Abf_m \xbf$. Unfortunately, we are only provided the measurements $\widetilde{\ybf}_m$ and the matrices $\Abf_m$. The position ambiguity of radar antennas can be modeled as a time-domain convolution with the measurements, or equivalently, as a gain and phase ambiguity in the frequency-domain of the radar signal, that is,
\begin{equation}\label{eq:Model1intro}
	\widetilde{\ybf}_m = \Dbf_{\widehat{\gbf}_m} \Abf_m \xbf,
\end{equation} 
where $\Dbf_{\widehat{\gbf}_m}$ is a diagonal matrix with the phase correction vector $\widehat{\gbf}_m \in \Cbb^F$ on its diagonal entries. Notice that the system in~\eqref{eq:Model1intro} is ill-posed in general since for any $M$ measurements, we are left with $MF$ equations and $MF+N$ unknowns. Alternatively, we propose in this paper to represent the gain and phase ambiguity as an image-domain convolution where a two-dimensional spatial shift kernel $\hbf_m$ is applied to the radar image $\xbf$, i.e.,
\begin{equation}\label{eq:Model2intro}
	\widetilde{\ybf}_m = \Abf_m \left(\xbf * \hbf_m \right).
\end{equation}
Under the new model, the shift kernels are one-sparse vectors of length $N_h^2$ with unknown support locations, thereby reducing the unknown degrees of freedom to an order of $M\log(N_h) + N$. The $\log(N_h)$ is the number of unknowns needed to encode the position of the nonzero entry in a vector of length $N_h$. Figure~\ref{fig:introLayout_v2} illustrates an example of a radar scene acquired by 16 distributed antennas with average position error around $2\lambda$ and maximum error at $3.5\lambda$, where $\lambda$ is the wavelength of the center frequency of a Gaussian pulse centered at 6 GHz with a 9 GHz bandwidth. The received signals are contaminated with white Gaussian noise at 30 dB peak-signal-to-noise-ratio (PSNR) after matched-filtering with the transmitted pulse. Figure~\ref{fig:imaging_comparison}(a) shows the ground truth recovery where the correct antenna positions are used. Contrast that with Figure~\ref{fig:imaging_comparison}(b) where the incorrect antenna positions are used in constructing the radar propagation matrix. The benefit of the image-domain convolution model becomes evident when we compare the recovery performance in Figures~\ref{fig:imaging_comparison}(c) and \ref{fig:imaging_comparison}(d). While the measurement-domain convolution model can provide an approximation of the radar scene in Figure~\ref{fig:imaging_comparison}(c), the image-domain convolution model produces an exact reconstruction of the scene in Figure~\ref{fig:imaging_comparison}(d) up to a global shift ambiguity. 

\begin{figure}[t]
\centering
\includegraphics[width=3in]{introLayout_v2-eps-converted-to.pdf}
\caption{\small Example of a distributed radar acquisition system with position ambiguity. The round dots indicate the assumed but erroneous antenna positions, while the $\times$'s indicate the true positions.}\label{fig:introLayout_v2}
\end{figure}

\fi

\subsection{Notation}
Throughout the text, lower case bold face letters $\xbf$ denote vectors and upper case bold face letters $\Xbf$ denote matrices. Sets are denoted by calligraphic upper case letters $\Bcal$ or upper case Greek letters $\Omega$. Italicized bold upper case letters $\Rbm(\cdot)$ refer to functions that act on vector spaces.  We use $\Fbf_1$ and $\Fbf_2$ to refer to the one-dimensional and two-dimensional Fourier transform matrices, respectively. The matrix $\Dbf_{\xbf}$ is used to indicate a diagonal matrix with the vector $\xbf$ on its main diagonal. Superscripts $^\T$ and $^\H$ refer to the transpose and the Hermitian transpose of matrices, respectively. The function $\delta(\lbm  + \ebm)$ refers to the Dirac delta function of the variable $\lbm$ delayed by the quantity $\ebm$. The $\ell_1$ norm of a vector $\xbf$ is defined as the sum of its absolute values, i.e., $\|\xbf\|_1 = \sum\limits_j |\xbf(j)|$. The $\ell_2$ norm of a vector is equal to the square root of the sum of squares of the vector, i.e., $\|\xbf\|_2 = \sqrt{\sum\limits_j |\xbf(j)|^2}$.

\section{Problem Formulation}
\label{sec:Problem}

\subsection{Signal model}

We consider a two-dimensional radar imaging scenario in which $M$
distributed antennas are used to detect $K$ targets. A
three-dimensional generalization is straightforward, but we avoid it
here for simplicity of exposition. The targets are located within a
spatial region of interest that is discretized on a grid $\Omega
\subset \Rbb^2, |\Omega | = N,$ and $N = N_x \times N_y$ with $N_x$
and $N_y$ specifying the number of grid points in the horizontal and
vertical directions. We use $\lbm \in \Omega$ to denote the index of
spatial positions, i.e., of grid-points in $\Omega$. We assume that the
grid is sufficiently fine, that off-grid errors are negligible, and
that there exists a single reflector in each grid point.

Let $\Gamma \subset \Rbb^2, |\Gamma| = M$ be the set of all the
spatial locations of the $M$ antennas. Note that the antennas may be
located anywhere, inside or outside of the grid, and not necessarily
on a grid point. Without loss of generality, we shall assume that a
subset of the antennas act as transmitter/receivers while the
remaining antennas are only receivers. A transmitting antenna at
position $\rbm \in \Gamma$ emits a time-domain pulse $p(t)$ with
frequency spectrum $P(\omega)$, where $\omega = 2\pi f$ is the angular
frequency and $f \in \Bcal$ is the ordinary frequency in the signal
bandwidth $\Bcal, |\Bcal | = F$. The received signal at antenna
position $\rbm' \in \Gamma$ due to the scattering of the transmitted
pulse by a target located at position $\lbm \in \Omega$ is given
by~\cite{LKLCVJLKC:2005}
\begin{equation}\label{eq:RcvSig}
	Y(\omega, \lbm, \rbm, \rbm') = P(\omega) G(\omega, \rbm, \rbm',\lbm) X(\lbm) + N(\omega), 
\end{equation}
where $X(\lbm) \in \Cbb$ is the scene reflectivity at location $\lbm$,
$N(\omega)$ is a noise component, and $G(\omega, \rbm, \rbm', \lbm)$
is the propagation gain characterized by
\begin{equation}\label{eq:G}
	G(\omega, \rbm, \rbm', \lbm) = a(\rbm, \rbm', \lbm) \e^{-i \omega \frac{\|\rbm - \lbm\|_2 + \|\rbm' - \lbm\|_2}{c}},
\end{equation}
where $a(\rbm, \rbm', \lbm)$ is the magnitude attenuation, $\e^{-i
  \omega \frac{\|\rbm - \lbm\|_2 + \|\rbm' - \lbm\|_2}{c}}$ is the
phase change due to the transmission delay, and $c$ denotes the speed
of light.

A typical scene comprises of multiple reflectors, at different
locations $\lbm\in\Omega$. If we assume no shadowing and no multiple
reflections, the received data for receiver-transmitter pair is the
sum of~\eqref{eq:RcvSig} over all $\lbm$ in which reflectors are
present. To compact notation, we use $\xbf \in \Cbb^N$ to denote the
vectorized reflectivity of the scene at all grid points in $\Omega$,
with empty grid points having zero reflectivity. Thus, the received
signal in \eqref{eq:RcvSig} at all frequencies $\omega$ can then be
written in vector form $\ybf(\rbm, \rbm') \in \Cbb^{F}$ as follows
\begin{equation}\label{eq:VecRcvSig}
	\ybf(\rbm, \rbm') = \Abf(\rbm, \rbm') \xbf + \nbf(\rbm, \rbm'),
\end{equation}
where $\Abf(\rbm, \rbm') \in \Cbb^{F \times N}$ includes $P(\omega)$
and $G(\omega, \rbm, \rbm',\lbm)$ and denotes the radar imaging
operator corresponding to the transmitter and receiver pair at
positions $\rbm$ and $\rbm'$, respectively, and $\nbf(\rbm, \rbm')$ is
the noise component.

\subsection{Imaging under position uncertainty}

Using \eqref{eq:VecRcvSig} to estimate the scene reflectivity $\xbf$
from measurements $\ybf(\rbm,\rbm')$ requires exact knowledge of the
imaging operator $\Abf(\rbm, \rbm')$, and consequently, the antenna
positions $\rbm$ and $\rbm'$. However, positioning errors commonly
occur in practice, especially in distributed radar settings which rely
on inaccurate global positioning systems or inertial navigation
systems.

To model the effect of position errors, we consider
transmitter-receiver pairs indexed by $m$, positioned at $(\rbm_m,
\rbm'_m)$. We denote the measurement vector at the assumed position of
the pair and the corresponding imaging operator using $\ybf_m :=
\ybf(\rbm_m, \rbm'_m)$ and $\Abf_m := \Abf(\rbm_m, \rbm'_m)$,
respectively. We use $\widetilde{\rbm}_m = \rbm_m + \ebm_m$ and
$\widetilde{\rbm}'_m = \rbm'_m + \ebm'_m$ to denote the actual positions
of the transmitter and receiver, respectively, where $\ebm_m$ and
$\ebm'_m$ denote the corresponding positioning errors.

The actual received antenna measurement $\widetilde{\ybf}_m :=
\ybf(\widetilde{\rbm}_m, \widetilde{\rbm}'_m)$ observes the scene reflectivity
$\xbf$ through the perturbed imaging operator $\widetilde{\Abf}_m :=
\Abf(\widetilde{\rbm}_m, \widetilde{\rbm}'_m)$, i.e.,
\begin{equation}
\widetilde{\ybf}_m = \widetilde{\Abf}_m \xbf + \nbf_m.
\end{equation} 
Since the operator $\widetilde{\Abf}_m$ is unknown, we need to model
the received measurements $\widetilde{\ybf}_m$ as a function of $\Abf_m$
and $\xbf$.

\subsubsection{Convolution in the measurement-domain}
Standard approaches for radar autofocus use a gain and phase
correction in the measurements' frequency domain to describe
$\widetilde{\ybf}_m$ in terms of $\Abf_m$ and $\xbf$. More precisely,
let $\widehat{\gbf}_m \in \Cbb^F$ be a complex valued vector
corresponding to the Fourier transform of a time-domain kernel $\gbf_m
\in \Rbb^M$, i.e, $\gbf_m = \Fbf_1^\H \widehat{\gbf}_m$. The received
frequency-domain measurements are expressed as
\begin{equation}\label{eq:Model1}
	\widetilde{\ybf}_m = \Dbf_{\widehat{\gbf}_m} \Abf_m \xbf + \nbf_m,
\end{equation}
where $\Dbf_{\widehat{\gbf}_m}$ is a diagonal matrix with
$\widehat{\gbf}_m$ on its diagonal entries, i.e., modulates the
acquired frequency-domain data. In other words, the position error is
assumed to affect the received signal through a time-domain
convolution with $\gbf_m$. Given $M$ measurements $\widetilde{\ybf}_m,
m \in \{1\dots M\}$, the radar autofocus problem is regarded as a
bilinear inverse problem in both the reflectivity image $\xbf$ and the
frequency-domain phase correction vectors $\widehat{\gbf}_m$ for all
$m$.

Notice that the system in~\eqref{eq:Model1} has $F$ equations with
$F+N$ unknowns, which makes it severely ill-posed. Even in the case
where $\xbf$ is sparse, the problem remains ill-posed since a general
phase correction vector $\widehat{\gbf}_m$ continues to have $F$
degrees of freedom. In order to make the problem tractable, the
kernels $\gbf_m = \Fbf_1^\H \widehat{\gbf}_m$ are often assumed to be
shift kernels, which reduces the degrees of freedom to a singe phase
angle per transmitter-receiver pair. However, the approximation that
$\gbf_m$ is a shift operator is only valid in the far field regime,
where the position error can be approximated by a one dimensional
shift in the down-range direction of the virtual antenna array, or if
the scene only contains a single reflector.

Most work in the self-calibration and radar autofocus literature
considers the measurement-domain convolution model
in~\eqref{eq:Model1} with some important simplifications. For
instance, the convolution kernels $\gbf_m$ are assumed to be
approximate shift kernels and therefore sparse in the
time-domain. Specifically, the simplification allows the kernels
$\gbf_m$ to be restricted to a lower $d$-dimensional subspace spanned
by the columns of a basis matrix $\Bbf \in \Cbb^{F\times d}, d <
F$. Combining all measurements $M$ into a single system of equations,
we obtain the following bilinear inverse problem
\begin{equation}\label{eq:CombinedSystem}
	\overline{\ybf} = \Dbf_{\widehat{\overline{\gbf}}} \overline{\Abf} \xbf + \overline{\nbf},
\end{equation}
where the over-line notation indicates stacking the $M$ vectors (or matrices) vertically into a single vector (or matrix), for example $\overline{\ybf} = \left[ \begin{array}{c}\ybf_1 \\ \vdots \\  \ybf_M \end{array} \right]$ and $\overline{\Abf} = \left[ \begin{array}{c}\Abf_1 \\ \vdots \\  \Abf_M \end{array} \right]$. Note that with this formulation, the vector $\overline{\gbf}$ belongs to the $Md$-dimensional subspace in $\Cbb^{MF}$ spanned by the basis matrix 
$$\overline{\Bbf} = \left[ \begin{array}{cccc} \Bbf & \mathbf{0} & \hdots & \mathbf{0} \\ \mathbf{0} &\Bbf & & \mathbf{0} \\ \vdots & & \ddots & \\ \mathbf{0} & \mathbf{0}& \hdots & \Bbf \end{array}\right] \in \Cbb^{MF \times Md},$$
where $\mathbf{0}$ is the $F\times d$ all zero matrix. Under this scenario, Ling and Strohmer~\cite{SelfCalibratBiconvex_LS2015} proposed the \emph{SparseLift} problem which casts the blind deconvolution problem as a convex sparse recovery program that estimates the sparse matrix $\Xbf = \overline{\gbf} \xbf^\H$ in the lifted space of the outer products of $\overline{\gbf}$ and $\xbf$. However, as the number of measurements and the dimensions of the parameters increase, the \emph{SparseLift} formulation can quickly become intractable to solve. Alternatively, Mansour et al.~\cite{MKLOBPA:2016} proposed a stochastic gradient descent approach for solving the nonconvex variation of~\eqref{eq:CombinedSystem} in the context of through-the-wall-radar-imaging. 
 
Other simplifications in the literature assume the operators $\Abf_m = \Abf$ to be fixed for all $m$. In this case, the measurements $\widetilde{\ybf}_m$ represent multiple measurements of a single vector $\zbf := \Abf\xbf$ observed through diverse channels $\widehat{\gbf}_m$, i.e., 
\begin{equation}\label{eq:BD_DiverseInputs}
	\widetilde{\ybf}_m = \Dbf_{\zbf}\widehat{\gbf}_m.
\end{equation}
Existing solutions to~\eqref{eq:BD_DiverseInputs} range from solving for a low rank matrix in the lifted space of the variables $\zbf$ and $\widehat{\gbf}_m$~\cite{BlindDeconvConvex_ARR2013,SelfCalibratBiconvex_LS2015,LiLeeBresler:2016,AD:2016}; reformulating the problem as a linear least squares problem~\cite{BN:2007,Balzano2008,GCD:2012,BPGL:2014,BlindDeconvLLS:2016,WC:2016}; or utilizing subspace identification techniques~\cite{LiuMunson:2011,LiuWieselMunson:2012,LiuWieselMunson:2013}.

While the simplifications described above can apply to special cases of autofocusing and self-calibration problems, they are not necessarily satisfied in general radar autofocus problems as we demonstrated in the previous section. We discuss next the image-domain blind deconvolution model and propose a reconstruction algorithm to solve the problem.

\subsubsection{Convolution in the image-domain}
A key contribution of our work is moving the convolution with the
shift kernel from the measurement domain to the image domain.  More
precisely, let $\hbf_m \in \Rbb^{N_h^2}, N_h \leq \min\{N_x,N_y\}$ be
a vectorized two-dimensional shift kernel of size $N_h \times
N_h$. Under the new model, the received signal of the antenna pair
indexed by $m$ is written as
\begin{equation}\label{eq:Model2}
	\widetilde{\ybf}_m = \Abf_m \left(\xbf * \hbf_m \right) + \nbf_m,
\end{equation}
where $*$ here denotes the two-dimensional convolution of the image
with the kernel.

We prove in Proposition~\ref{prop:ImageDomainConvolution} that when the transmitting and receiving antennas are affected by the same position ambiguity, the convolution kernel $\hbf_m$ is strictly a spatial shift kernel with a single nonzero entry equal to one. This situation is prevalent in systems where the transmitting and receiving antennas are collocated. The system in~\eqref{eq:Model2} may still be underdetermined with $F$ equations and $N_h^2 + N$ unknowns. However, given enough measurements, it should be possible to recover $\xbf$ and all shift kernels $\hbf_m$ by utilizing an appropriate regularization for each. In the next section, we demonstrate the appropriateness of the image-domain convolution in~\eqref{eq:Model2} compared to the measurement-domain convolution in~\eqref{eq:Model1} through an illustrative example.
\begin{proposition}\label{prop:ImageDomainConvolution}
Let $\widetilde{\ybf}_m := \widetilde{\Abf}_m \xbf,$ where $\xbf$ is a radar image defined over a spatial domain $\Omega$. Denote by $\ebm_m$ and $\ebm'_m$ the position ambiguities for the transmitter and receiver antenna pair indexed by $m$. 

If $\ebm'_m = \ebm_m$ and $\xbf$ is zero valued within a boundary of width $\ebm_m$ inside $\Omega$, then there exists a spatially shifted image $\widetilde{\xbm}(\lbm) = \delta(\lbm + \ebm_m)*\xbm(\lbm), \forall \lbm \in \Omega$ such that $\widetilde{\ybf}_m = \Abf_m \widetilde{\xbf},$ where $\delta(\lbm + \ebm_m)$ is the two dimensional shift kernel.

Otherwise, if $\ebm'_m = \ebm_m + \dbm_m$ with $\|\dbm_m\|_2 \leq \Delta$, then the approximation $\widetilde{\ybf}_m \approx \Abf_m \widetilde{\xbf}$ incurs a phase error bounded by $\e^{\pm i\omega \Delta/c}$ for each frequency $\omega$.
\end{proposition}
\begin{proof}
Consider the received signal $\widetilde{\ybf}(\omega)$ at an arbitrary frequency $\omega$ in the bandwidth of the radar pulse. Ignoring the amplitude attenuation, the received signal at perturbed antenna positions $\widetilde{\rbm}, \widetilde{\rbm}'$ is given by
\begin{equation}
	\widetilde{\ybf}(\omega) = \widetilde{\Abf}(\omega)\xbf = \sum\limits_{\lbm \in \Omega} \e^{-i\omega \frac{\|\widetilde{\rbm} - \lbm\|_2 + \|\widetilde{\rbm}' - \lbm\|_2}{c}} \xbf(\lbm),
\end{equation}
where $\widetilde{\Abf}_m(\omega)$ is the row of $\widetilde{\Abf}$ corresponding to the frequency $\omega$.
Recall that $\widetilde{\rbm} = \rbm + \ebm,$ $\widetilde{\rbm}' = \rbm' + \ebm'$. When $\ebm = \ebm'$, we have
$$
\begin{array}{ll}
	\|\widetilde{\rbm} - \lbm\|_2 + \|\widetilde{\rbm}' - \lbm\|_2 &= \|\rbm + \ebm - \lbm\|_2 + \|\rbm' + \ebm' - \lbm\|_2 \\
										&= \|\rbm - \widetilde{\lbm}\|_2 + \|\rbm' - \widetilde{\lbm}\|_2,
\end{array}
$$
where $\widetilde{\lbm} = \lbm - \ebm$. Then the received signal $\widetilde{\ybf}(\omega)$ satisfies
\begin{equation}
\begin{array}{ll}
	\widetilde{\ybf}(\omega) & = \sum\limits_{\lbm \in \Omega} \e^{-i\omega \frac{\|\rbm - \widetilde{\lbm}\|_2 + \|\rbm' - \widetilde{\lbm}\|_2}{c}} \xbf(\lbm) \\
				& = \sum\limits_{\widetilde{\lbm} \in \widetilde{\Omega}} \e^{-i\omega \frac{\|\rbm - \widetilde{\lbm}\|_2 + \|\rbm' - \widetilde{\lbm}\|_2}{c}} \xbf(\widetilde{\lbm} + \ebm) \\
				& = \sum\limits_{\widetilde{\lbm} \in \widetilde{\Omega}} \e^{-i\omega \frac{\|\rbm - \widetilde{\lbm}\|_2 + \|\rbm' - \widetilde{\lbm}\|_2}{c}} \left(\delta(\widetilde{\lbm} + \ebm) *\xbf(\widetilde{\lbm} ) \right),
\end{array}
\end{equation}
where $\widetilde{\Omega}$ is a displacement of the set $\Omega$ by $\ebm$. Denote by $\widetilde{\xbf}(\widetilde{\lbm}) = \delta(\widetilde{\lbm} + \ebm) *\xbf(\widetilde{\lbm})$ and since the support of $\xbf$ is surrounded by a zero valued boundary of width $\ebm$ inside $\Omega$, we get
\begin{equation}\label{eq:ytilde_Axtilde}
\begin{array}{ll}
	\widetilde{\ybf}(\omega) & = \sum\limits_{\widetilde{\lbm} \in \widetilde{\Omega}} \e^{-i\omega \frac{\|\rbm - \widetilde{\lbm}\|_2 + \|\rbm' - \widetilde{\lbm}\|_2}{c}} \widetilde{\xbf}(\widetilde{\lbm} ) \\
				& = \sum\limits_{\lbm \in \Omega} \e^{-i\omega \frac{\|\rbm - \lbm\|_2 + \|\rbm' - \lbm\|_2}{c}} \widetilde{\xbf}(\lbm) \\
				& = \Abf(\omega) \widetilde{\xbf}(\lbm ).
\end{array}
\end{equation}

If $\ebm' = \ebm + \dbm$ for some offset $\dbm, \|\dbm\|_2 \leq \Delta$, then 
$$
\begin{array}{l}
\|\rbm - \widetilde{\lbm}\|_2 + \|\rbm' - \widetilde{\lbm}\|_2 - \Delta \\ 
	\hspace{0.5in} \leq \|\widetilde{\rbm} - \lbm\|_2 + \|\widetilde{\rbm}' - \lbm\|_2 \leq \\
	\hspace{1in} \|\rbm - \widetilde{\lbm}\|_2 + \|\rbm' - \widetilde{\lbm}\|_2 + \Delta
\end{array}
$$
Consequently, the expression in \eqref{eq:ytilde_Axtilde} incurs a maximum phase error equal to $\e^{\pm i\omega \Delta/c}$ when $\ebm' \neq \ebm$.
\end{proof}

\subsubsection{An illustrative example}
We simulate a radar scene with three targets inside a region of interest $\Omega$ and generate measurements corresponding to three transmitter-receiver antennas as shown in Figure~\ref{fig:example_three_layout}. The blue crosses and red circles indicate the true positions and the assumed (erroneous) positions of the antennas, respectively.

Consider first the antennas lying inside the dashed ellipse. The time-domain measurements corresponding to each of the antenna positions are shown in the top plot of Figure~\ref{fig:example_single_time}. According to the model in~\eqref{eq:Model1}, there exists a convolutional kernel $\gbf$ shown in the bottom plot of Figure~\ref{fig:example_single_time} that maps the red curve corresponding to $\Fbf_1^\H \Abf_m \xbf$ to the blue curve corresponding to $\Fbf_1^\H \widetilde{\Abf}_m \xbf$, where $\Fbf_1$ is the one-dimensional Fourier transform. However, it is clear from the figure that $\gbf$ cannot be a simple shift kernel. On the other hand, consider the reflectivity images in Figure~\ref{fig:example_imaging_result} (a) and (b) obtained from multiplying the received measurement $\widetilde{\ybf}_m = \widetilde{\Abf}_m \xbf$ by the adjoint of each of the true imaging operator $\widetilde{\Abf}_m$ and the erroneous imaging operator $\Abf_m$. Notice that there does exist a simple two-dimensional shift kernel that can be applied to the true reflectivity image $\xbf$ to produce the measurements $\widetilde{\ybf}_m$. To better illustrate this fact, we apply the same shift to two other antennas as shown in Figure~\ref{fig:example_three_layout} and generate the reflectivity images in Figure~\ref{fig:example_imaging_result} (c) and (d) using each of the imaging operators $\widetilde{\Abf}_m$ and $\Abf_m$, respectively. The arcs visible in Figure~\ref{fig:example_imaging_result} (a) and (b) become focused on the shifted target locations since the resulting virtual array has a wider aperture, which in turn reduces the null space of the resulting imaging operator.

\begin{figure}[ht]
\centering
\includegraphics[width= 2.5in]{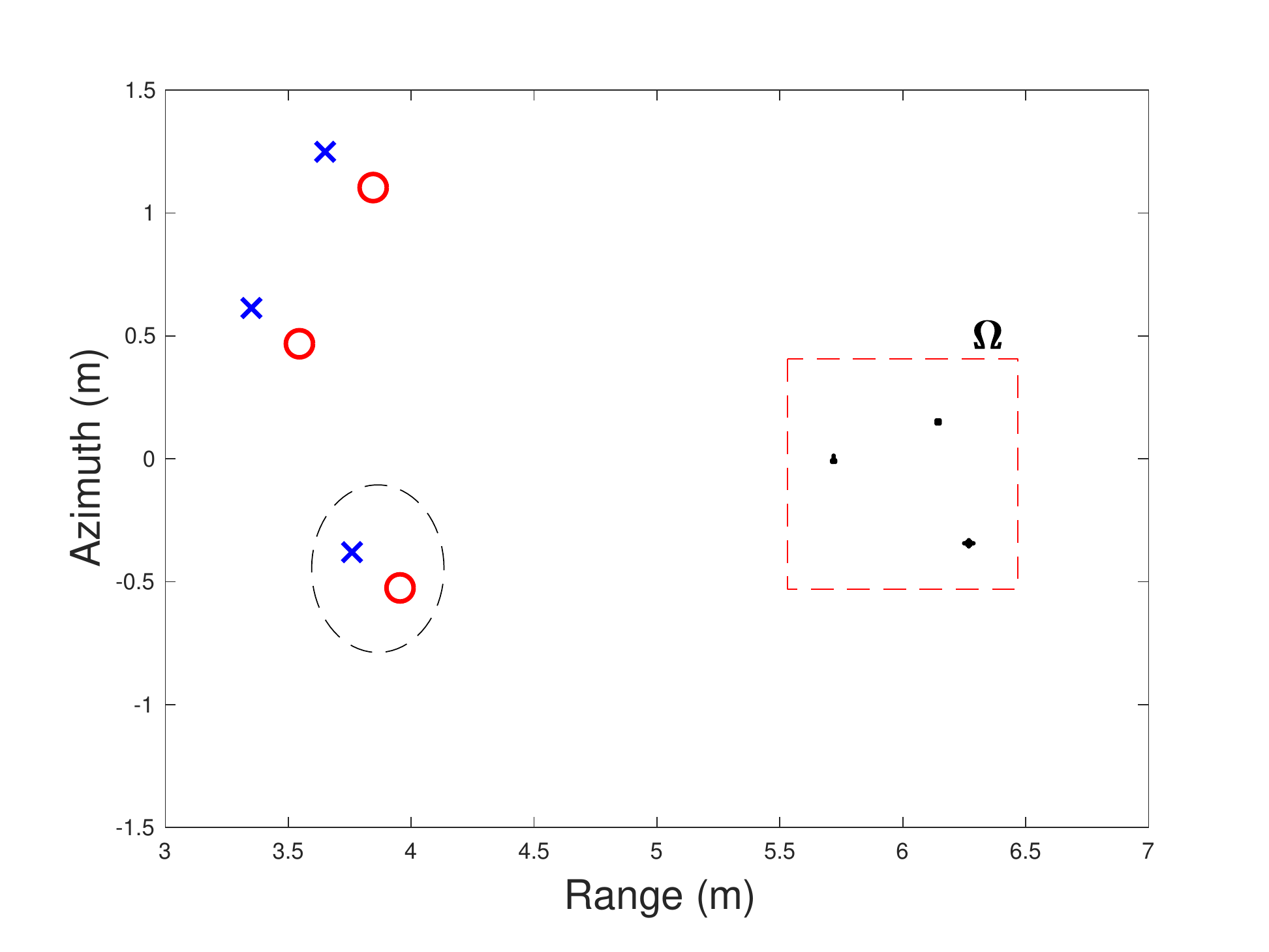}
\caption{\small Illustration of an example radar scene with three antennas at the true positions marked by the blue $\times$'s and the assumed positions marked by the red circles. Three targets are observed inside the region of interest bounded by the dashed red line.}\label{fig:example_three_layout}
\end{figure}

\begin{figure}[ht]
\centering
\includegraphics[width= 3in]{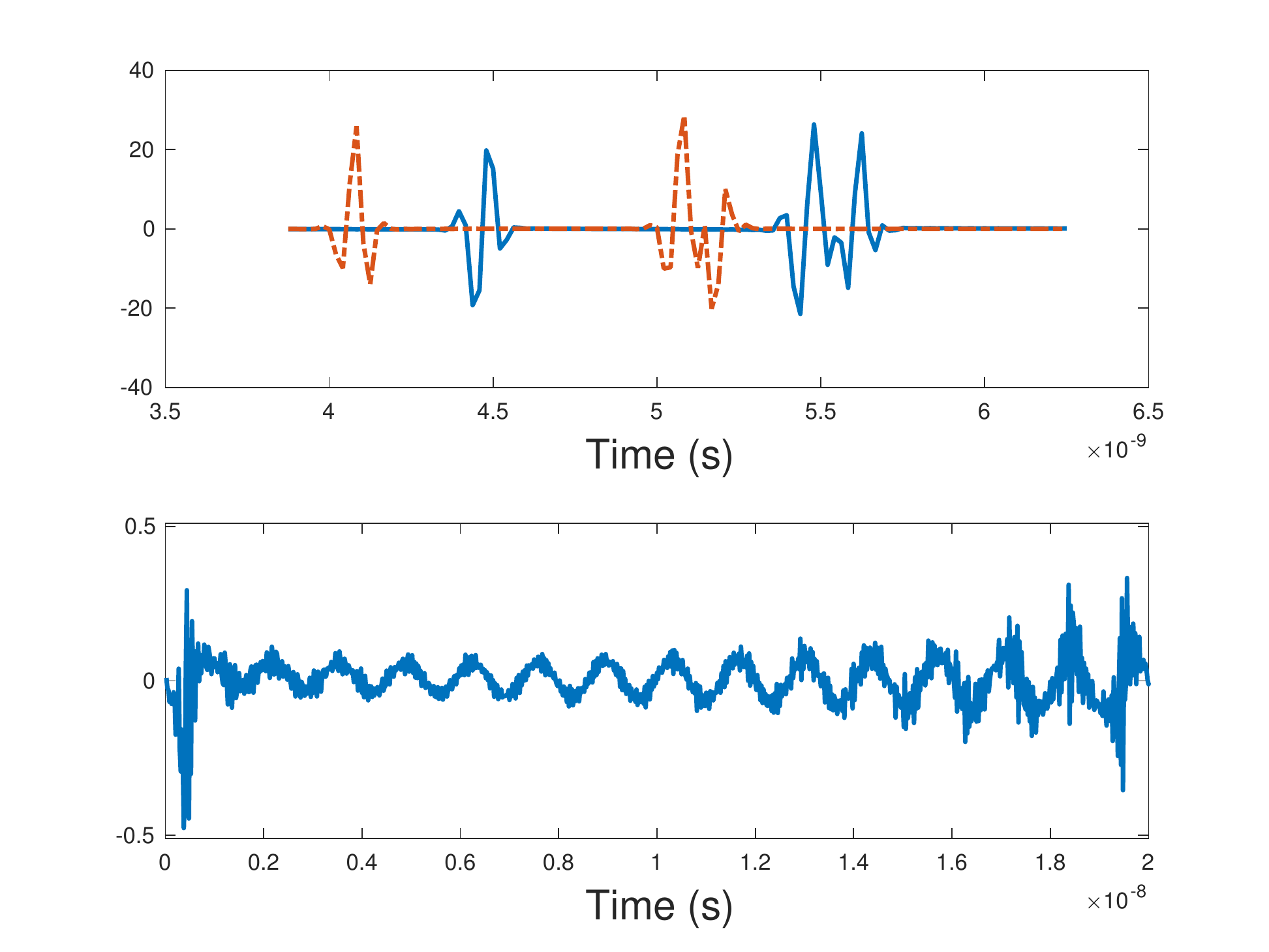}
\caption{\small Top: time-domain signals observed by the third set of antennas shown in Figure~\ref{fig:example_three_layout} and highlighted by the dashed ellipse. The solid blue line corresponds to the signal observed by the location of the $\times$. The dashed red line corresponds to the signal observed by the location of the circle. Bottom: Convolution kernel that maps the red curve to the blue curve.}\label{fig:example_single_time}
\end{figure}

\begin{figure}[ht]
\centering
\mbox{
\begin{subfigure}[b]{0.22\textwidth}
                \centering
                \includegraphics[width=\textwidth]{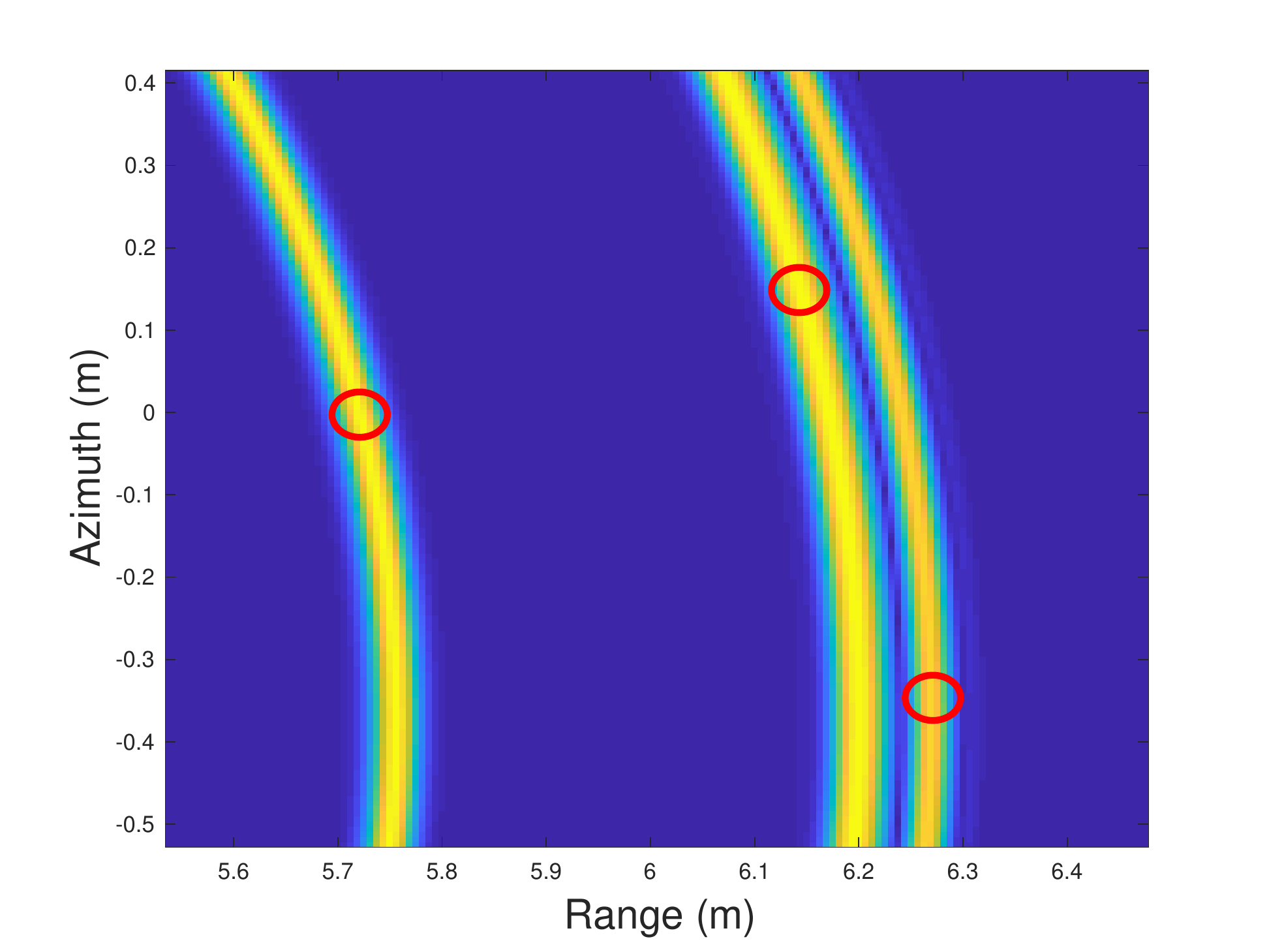}
                \caption{}
        \end{subfigure}%
        \begin{subfigure}[b]{0.22\textwidth}
                \centering
                \includegraphics[width=\textwidth]{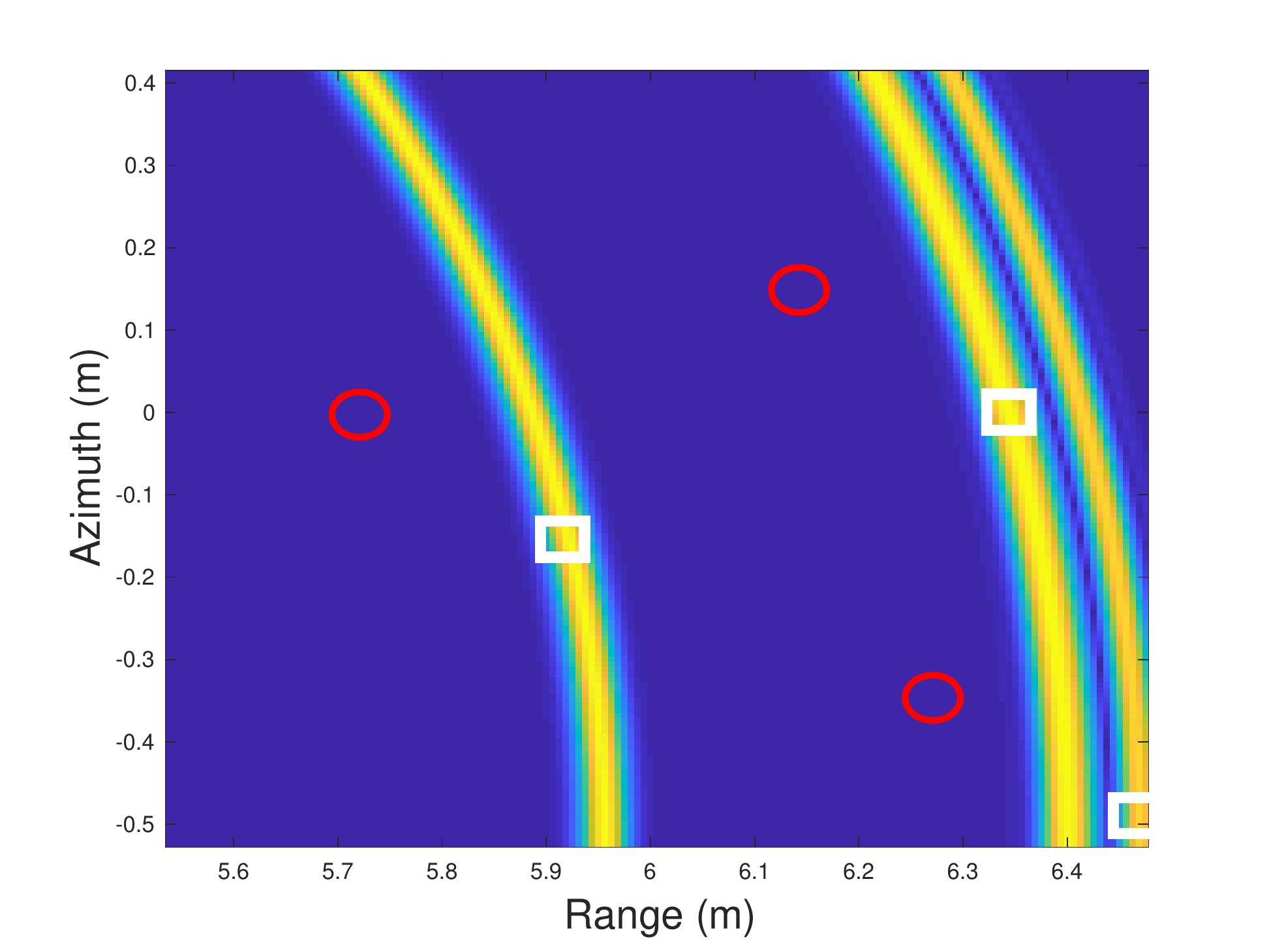}
                \caption{}
        \end{subfigure}%
}
\mbox{
\begin{subfigure}[b]{0.22\textwidth}
                \centering
                \includegraphics[width=\textwidth]{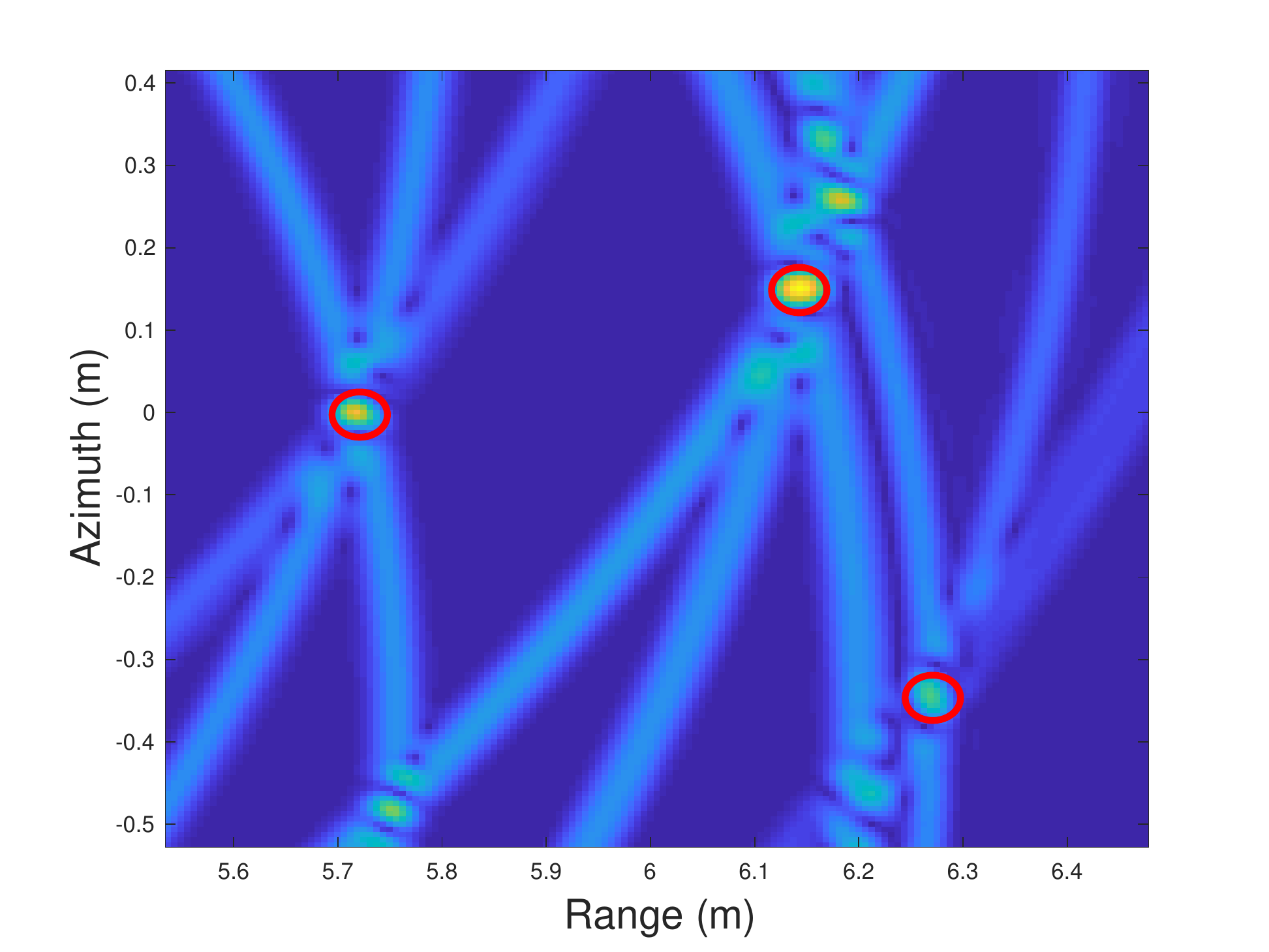}
                \caption{}
        \end{subfigure}%
        \begin{subfigure}[b]{0.22\textwidth}
                \centering
                \includegraphics[width=\textwidth]{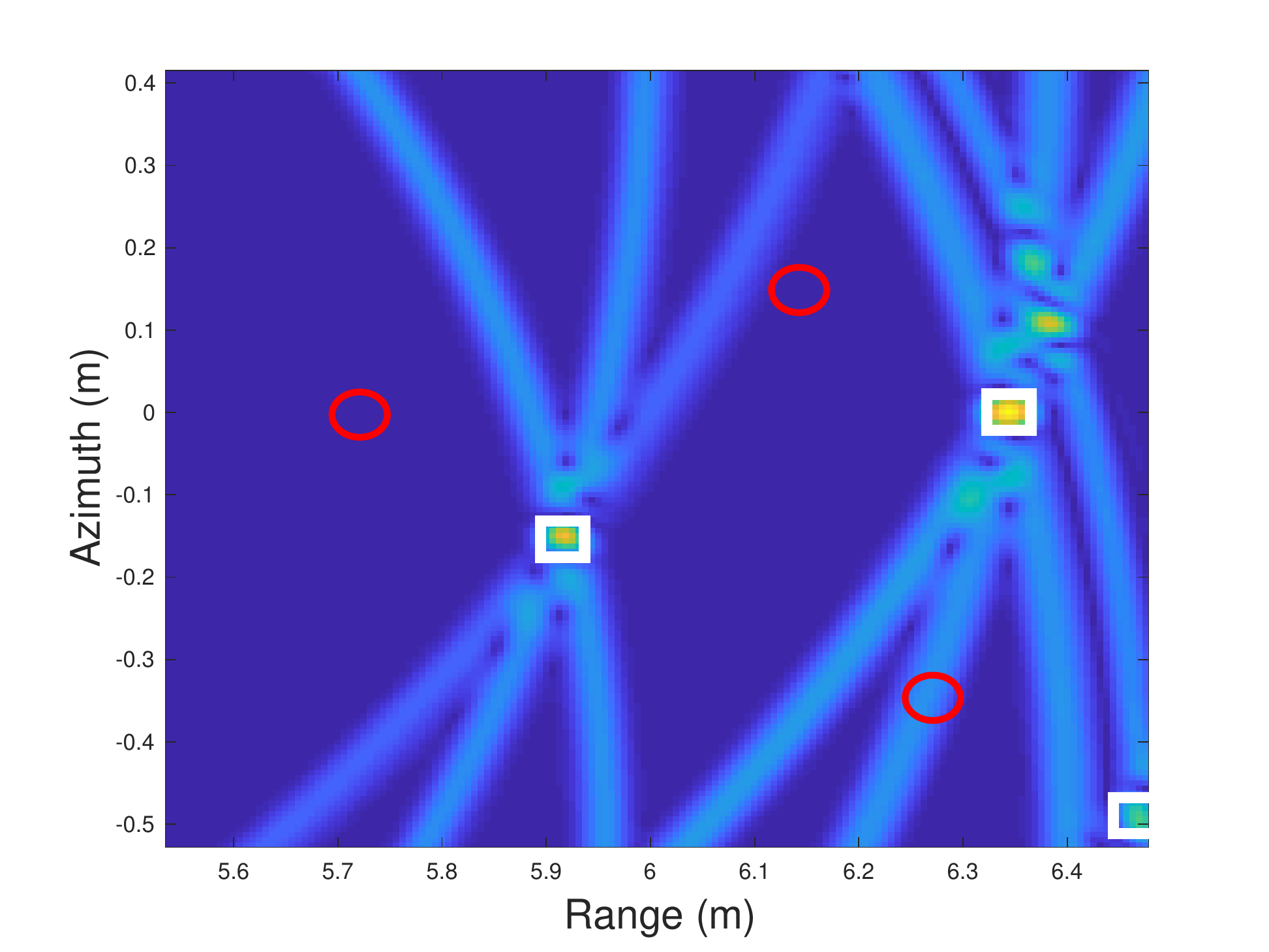}
                \caption{}
        \end{subfigure}%
}
\caption{\small Imaging results obtained by applying the adjoint of the true imaging operator $\widetilde{\Abf}$ in (a) and (c); and the wrong imaging operator $\Abf$ in (b) and (d) to the received measurements. (a) and (b) show the imaging result from the single antenna measurement marked by the $\times$ in the ellipse shown in Figure~\ref{fig:example_three_layout}. (c) and (d) show the imaging result from using the antenna positions marked by all three $\times$. The red circles and white squares indicate the true positions and the shifted positions of the targets, respectively.}\label{fig:example_imaging_result}
\end{figure}

Before ending this section, we consider again the image-domain convolution model described in~\eqref{eq:Model2} and expressed in the spatial Fourier domain below
\begin{equation}\label{eq:Model2_F}
\begin{array}{ll}
	\widetilde{\ybf}_m & = \Abf_m \left(\xbf * \hbf_m \right) + \nbf_m\\
			& = \Abf_m \Fbf_2^\H \Dbf_{\widehat{\hbf}_m} \widehat{\xbf} + \nbf_m, 
\end{array}
\end{equation}
where $\widehat{\hbf}_m = \Fbf_2 \hbf_m$ and $\widehat{\xbf} = \Fbf_2 \xbf$ denote the two-dimensional Fourier transforms of $\hbf_m$ and $\xbf$, respectively, and $\Dbf_{\widehat{\hbf}_m}$ is the diagonal matrix with $\widehat{\hbf}_m$ on the diagonal. Notice that the convolved signal $\left(\xbf * \hbf_m \right)$ in~\eqref{eq:Model2_F} is now observed through a linear operator $\Abf_m$ that has a large null space, making the variety of convex optimization methods in the literature developed for blind deconvolution inapplicable to this scenario. We may consider convexifying problem~\eqref{eq:Model2_F} by lifting into the variable representation $\Xbf_m = \hbf \xbf^\H,$ such that $\widetilde{\ybf}_m = \Abf_m \Cbf \ \mathrm{vec}(\Xbf_m),$ where the linear operator $\Cbf$ maps the vectorized outer-product matrix $\mathrm{vec}(\Xbf_m)$ to the output of the convolution $\left(\xbf * \hbf_m \right)$. However, the resulting linear operator $\Abf_m\Cbf$ contains repeated columns and lacks many of the properties required for successfully recovering the sparse and rank-one lifted variables $\Xbf_m$.

\section{Proposed approach}\label{sec:ProposedSolution}

The approach we propose in this paper is based on using block
coordinate descent to compute the radar reflectivity image $\xbf$ and
the spatial convolution filters $\hbf_m$ from noisy measurements
$\widetilde{\ybf}_m$.

\subsection{The global model}

We first incorporate into the model in~\eqref{eq:Model2_F} the prior
information that the image $\xbf$ is sparse and piecewise continuous
and that the kernels $\hbf_m$ are two dimensional shift
operators. Therefore, we use a \emph{fused
  Lasso}\cite{fusedLasso:2005} penalty function $\Rbm_{\xbf}(\cdot)$
for $\xbf$, and an $\ell_1$ norm regularizer $\Rbm_{\hbf}(\hbf_m) =
\|\hbf_m\|_1$ for the convolution filters $\hbf_m$. The overall
optimization problem is described as follows
\begin{equation}\label{eq:FullOptimization}
\begin{array}{ll}
	\min\limits_{\scriptsize \begin{array}{l}\xbf \in \Cbb^N,\\ \hbf_m \in \Rbb_+^{N_h^2} \end{array}} &  \sum\limits_{m=1}^M \frac{1}{2}\|\widetilde{\ybf}_m - \Abf_m \Fbf_2^\H \Dbf_{ \widehat{\hbf}_m}  \widehat{\xbf}\|_2^2  + \mu\Rbm_{\hbf}(\hbf_m) \\ \\
	\text{subject to} & \Rbm_{\xbf}(\xbf) \leq \tau, \\ \\
	&  \mathbf{1}^T\hbf_m = 1, \forall \ m \in \{1\dots M\},
\end{array}
\end{equation}
where $\mathbf{1}$ is the all one vector, and as before, $\widehat{\hbf}_m = \Fbf_2 \hbf_m$ and $\widehat{\xbf} = \Fbf_2 \xbf$. We use the regularization parameter $\mu$ to control the tradeoff between the sparse prior and the data mismatch cost. We also use an upper bound $\tau$ to constrain the penalty function $\Rbm_{\xbf}(\xbf)$. We describe the detailed procedure for computing $\tau$ later in this section.

The \emph{fused Lasso} regularizer $\Rbm_{\xbf}(\xbf)$ combines the $\ell_1$ norm and the total variation (TV) norm of a signal: 
\begin{equation}\label{eq:RegX}
	\Rbm_{\xbf}(\xbf) = \|\xbf\|_1 +  \gamma\|\xbf\|_{TV}, 
\end{equation}
where the total variation norm $\|\xbf\|_{TV}$ is defined as the sum of the $\ell_2$ norms of groups of elements in the gradient vector $\sbf = \Ebf\xbf,$ where $\Ebf: \Cbb^N \rightarrow \Cbb^{2N}$ is the two dimensional finite difference operator, such that, the first $N$ entries of $\sbf$ contain the horizontal gradient coefficients, and the second $N$ entries contain the vertical gradient coefficients. Therefore, the total variation norm of $\xbf$ is expressed in terms of the $\ell_{2,1}$ mixed norm of $\sbf$ as follows:
\begin{equation}
	\|\xbf\|_{TV} := \|\sbf\|_{2,1} = \sum\limits_{j=1}^N \sqrt{\sbf^2(j) + \sbf^2(N+j)}.
\end{equation}

The minimization in~\eqref{eq:FullOptimization} is nonconvex and our
aim is to find a stationary point to the problem. Therefore, we
present in Algorithm~\ref{alg:BlockCoordDescent} a block coordinate
descent approach that alternates between descent steps for each of
$\xbf$ and $\hbf_m,$ for all $m$. The shift kernels $\hbf_m$ are all
initialized to the no-shift kernel $\hbf^0$, an $N_h\times N_h$
zero-valued matrix with the central entry set equal to one. For each
descent step, we apply a small number of iterations of the fast
iterative shrinkage/thresholding algorithm (FISTA)~\cite{FISTA:2009}
for updating $\hbf_m$, and a similar number of iterations of an
accelerated projected gradient descent algorithm (FPGD) inspired by
FISTA and~\cite{BergFriedlander:2008,BergFriedlander:2011} for
updating $\xbf$. The optimization subroutines shown in
Algorithm~\ref{alg:BlockCoordDescent} use the data fidelity cost
function
\begin{equation}\label{eq:Dcal}
\Dbm(\u) := \sum\limits_{m=1}^M \frac{1}{2}\|\widetilde{\ybf}_m - \Acal^m \u\|_2^2,
\end{equation} 
where $\u$ refers to either the image $\xbf$ or the sequence of convolution kernels $\hbf_m$. The forward operator with the respect to $\xbf$ given the estimates of the kernels $\hbf_m^t$ at iteration $t$ is defined as 
\begin{equation}\label{eq:forwardX}
	\Acal^m_{\xbf}(\hbf_m^t) := \Abf_m \Fbf_2^\H \Dbf_{ \Fbf_2 \hbf_m^t}  \Fbf_2.
\end{equation}
Similarly, the forward operator with respect to $\hbf_m$ given the estimate of the image $\xbf^t$ at iteration $t$ is defined as
\begin{equation}\label{eq:forwardH}
	\Acal^m_{\hbf}(\xbf^t) := \Abf_m \Fbf_2^\H \Dbf_{ \Fbf_2 \xbf^t}  \Fbf_2.
\end{equation}
Moreover, every descent step of $\hbf_m,$ produces an estimate $\widetilde{\hbf}_m$ which does not necessarily satisfy the shift kernel properties. Therefore, we use a projector $\bm{P}(\widetilde{\hbf}_m)$ onto the space of shift kernels which sparsifies $\widetilde{\hbf}_m$ by setting to one its largest entry and setting the remaining entries to zero. When the largest value is shared among more than one entry, we choose the one that is closest to the center of the kernel and set the remaining entries to zero.

\begin{algorithm}[ht]
\caption{ Block coordinate descent for solving~\eqref{eq:FullOptimization}}\label{alg:BlockCoordDescent}
\begin{algorithmic}[1]
\Statex{\textbf{input: }measurements $\{\widetilde{\ybf}_m\}_{m=1}^M$, initial guess $\xbf^0, \hbf^0$, maximum subroutine iterations $T$, and parameters $\tau, \mu$.}
\Statex{\textbf{set: } $j \leftarrow 1$; $\widetilde{\hbf}_m^0, \hbf_m^0 \leftarrow \hbf^0$ for all $m$}
	\Repeat
	\State $\Acal^m_{\xbf} \leftarrow \Acal^m_{\xbf}(\hbf_m^{j-1})$ for all $m$
	\State Compute $\tau^j$ according to \eqref{eq:tau_update} or \eqref{eq:tau_update2}
	\State $\xbf^j \leftarrow \texttt{fpgd}(\{\Acal^m_{\xbf}\}_{m=1}^M,  \Rbm_{\xbf}, \{\widetilde{\ybf}_m\}_{m=1}^M, \tau^j, \xbf^{j-1},T)$
	\For{$m \gets 1 \textrm{ to } M$}
	\State $\Acal^m_{\hbf} \leftarrow \Acal^m_{\hbf}(\xbf^j)$
	\State $\widetilde{\hbf}_m^j \leftarrow \texttt{fista}(\Acal^m_{\hbf},  \mu\Rbm_{\hbf}, \widetilde{\ybf}_m, \widetilde{\hbf}_m^{j-1},T)$
	\State $\hbf_m^j \leftarrow \bm{P}(\widetilde{\hbf}_m^j)$
	\EndFor
	\State $j \leftarrow j+1$
	\Until{stopping criterion}
	\Statex{\textbf{return:} estimate of the radar image $\xbf^j$}
\end{algorithmic}
\end{algorithm}

\subsection{\texttt{fista} subroutine for updating $\hbf_m$}

In general, FISTA can be used to solve convex optimization problems of the form
\begin{equation}\label{eq:FISTA_OptimProb}
	\min\limits_{\u \in \Scal} \Dbm(\u) + \lambda \Rbm(\u),
\end{equation}
where $\Dbm(\u)$ is a smooth data fidelity cost function and $\Rbm$ is a penalty function which can be non-smooth. The iterative procedure involves a proximal gradient update with a Lipschitz step size, in addition to a momentum term. The proximal operator is defined as
\begin{equation}\label{eq:GeneralProx}
	\textrm{prox}_{\eta \Rbm} (\z) := \arg\min\limits_{\u \in \Scal} \left\{ \frac{1}{2} \|\u - \z\|_2^2 + \eta \Rbm(\u)\right\}.
\end{equation}
More recently, it was shown in~\cite{KMW:2017} that the proximal operator can be replaced by a general nonexpansive denoiser without affecting the convergence performance of the FISTA routine.

Note that the expression for $\Dbm$ in~\eqref{eq:Dcal} is separable in $\hbf_m$ for every $m$. Therefore, the FISTA subroutine for updating $\hbf_m$ reduces to a standard non-negative sparse recovery problem described in Algorithm~\ref{alg:fistaH}. The function $\mathcal{T}_+\left( \z; \eta \right)$ in step 3 of the algorithm is the element-wise non-negative soft-thresholding operator induced by the $\ell_1$ proximal shrinkage function that modifies the entries $\z(j), \forall j \in \{1\dots J\}$ of a vector $\z \in \Rbb^J$ as follows:
\begin{equation}\label{eq:NNsoft-thresholding}
	\mathcal{T}_+\left( \z; \beta \right) = \left\{
	\begin{array}{ll}
		\z(j) - \beta, & \textrm{if} \ \z(j) > \beta \\
		0,& \textrm{otherwise.}
	\end{array}
	\right.
\end{equation}
Finally, we enforce the unit sum constraint by scaling the vectors $\u^t$ as shown in step 5 of the algorithm.
\begin{algorithm}[ht]
\caption{\texttt{fista} subroutine for updating $\hbf_m$}\label{alg:fistaH}
\begin{algorithmic}[1]
\Statex{\textbf{input: } $\Acal^m_{\hbf},  \mu\Rbm_{\hbf}, \widetilde{\ybf}_m, \widetilde{\hbf}_m^{j-1},T$.}
\Statex{\textbf{set: } $q_0 = 1$, $\u^0 = \s^0 = \widetilde{\hbf}_m^{j-1}$}
	\State $\alpha \leftarrow$ inverse of maximum eigenvalue of $\Acal^{m \H}_{\hbf} \Acal^m_{\hbf}$
	\For{$t \gets 1 \textrm{ to } T$}
	\State $\z^t \leftarrow s^{t-1} + \alpha \Acal^{m \H}_{\hbf}\left(  \widetilde{\ybf}_m - \Acal^m_{\hbf} \s^{t-1} \right)$
	\State $\u^t \leftarrow \mathcal{T}_+\left( \z^t; \alpha\mu \right)$
	\State $\u^t \leftarrow \frac{1}{\mathbf{1}^T\u^t} \u^t$
	\State $q_t \leftarrow \frac{1 + \sqrt{1 + 4q_{t-1}^2}}{2}$
	\State $\s^t \leftarrow \u^t + \frac{q_{t-1} - 1}{q_t}(\u^t - \u^{t-1})$
	\EndFor
	\Statex{\textbf{return:} $\u^T$}
\end{algorithmic}
\end{algorithm}

\subsection{\texttt{fpgd} subroutine for updating $\xbf_m$}
Next, we describe the rules for computing $\tau^j$ and the fast projected gradient descent subroutine \texttt{fpgd} for updating $\xbf$ in Algorithm~\ref{alg:fpgdX}. 

Given the updated $\hbf_m$ kernels, our aim at this stage is to find the image $\xbf$ that minimizes the \emph{fused Lasso} penalty term $\Rbm_{\xbf}$ subject to the data fidelity condition $\|\Dbm(\xbf)\|_2 \leq \sigma$, where $\sigma$ is an upper bound on the noise level, i.e., $\|\nbf\|_2 \leq \sigma$. In~\cite{BergFriedlander:2008,BergFriedlander:2011}, van den Berg and Friedlander developed a framework for general gauge function minimization with least squares constraints, where the problem is ``flipped" by solving a sequence of simpler least squares subproblems with gauge constraints using projected gradient descent. The sequence of subproblems is determined by specifying the upper bound $\tau$ on the gauge penalty as well as an update rule for computing the optimal $\tau$.

\subsubsection{Computing $\tau^j$}
We follow a similar approach to~\cite{BergFriedlander:2011} and define the following subproblem for updating $\xbf$
\begin{equation}\label{eq:x_subproblem}
	\min\limits_{\xbf} \sum\limits_m \frac{1}{2}\|\widetilde{\ybf}_m - \Acal^m_{\xbf} \xbf\|_2^2 \ \textrm{s.t.} \ \Rbm_{\xbf}(\xbf) \leq \tau^j,
\end{equation}
where the upper bound $\tau^j$ is computed in terms of the residual $\widetilde{\rbf}_m = \widetilde{\ybf}_m - \Acal^m_{\xbf} \xbf$ and the polar function $\Rbm_{\xbf}^o$ to the gauge as follows
\begin{equation}\label{eq:tau_update}
	\tau^j = \frac{\sum\limits_m\|\widetilde{\rbf}_m\|^2_2 - \sigma\sqrt{\sum\limits_m\|\widetilde{\rbf}_m\|^2_2}}{\Rbm_{\xbf}^o(\sum\limits_m \Acal^{m \H}_{\xbf}\widetilde{\rbf}_m)}.
\end{equation}
The polar function of the \emph{fused Lasso} penalty in \eqref{eq:RegX} is given by
\begin{equation}\label{eq:polar_fusedLasso}
	\Rbm_{\xbf}^o(\u) = \max\{\|\u\|_{\infty}, \gamma\|\s\|_{2, \infty}\},
\end{equation}
where $\s = \Ebf\u$ is the gradient vector of $\u$, the $\ell_{2, \infty}$ norm $\|\s\|_{2, \infty} = \max\limits_{j \in \{1\dots N\}} \sqrt{\s^2(j) + \s^2(N+j)}$, and the $\ell_{\infty}$ norm $\|\u\|_{\infty}$ of a vector $\u$ selects its maximum entry in absolute value.

We initialize the residual vector $\widetilde{\rbf}_m = \widetilde{\ybf}_m$ and fix it for all iterations $j$ of Algorithm~\ref{alg:BlockCoordDescent} until $\xbf$ converges to a stationary point $\xbf^*$ at some iteration $j^*$. After that, we update $\widetilde{\rbf}_m = \widetilde{\ybf}_m - \Acal^m_{\xbf} \xbf^*$ and the bound
\begin{equation}\label{eq:tau_update2}
	\tau^j = \tau^{j^*} + \frac{\sum\limits_m\|\widetilde{\rbf}_m\|^2_2 - \sigma\sqrt{\sum\limits_m\|\widetilde{\rbf}_m\|^2_2}}{\Rbm_{\xbf}^o(\sum\limits_m \Acal^{m \H}_{\xbf}\widetilde{\rbf}_m)}, \ \forall j > j^*
\end{equation}
and continue running the iterations of Algorithm~\ref{alg:BlockCoordDescent} until $\xbf$ converges to a new stationary point or until a maximum number of iterations has been reached, at which point we consider the exit condition has been met for our application. Alternatively, the iterative procedure of Algorithm~\ref{alg:BlockCoordDescent} may be continued until the sequence of $\tau^{j^*}$'s converges.

\subsubsection{Updating $\xbf$}
The fast projected gradient descent subroutine for updating $\xbf$ is summarized in Algorithm~\ref{alg:fpgdX}. The approach capitalizes on the momentum term used in FISTA to speed up convergence but differs from FISTA in that the proximal gradient update is replaced with a projected gradient update in order to satisfy the constraint $\Rbm_{\xbf}(\xbf) \leq \tau^j$. What remains is to specify the procedure for projecting onto the constrained \emph{fused Lasso} penalty.

Proposition~\ref{prop:projX} shows that the projection of a point $\z \in \Cbb^N$ onto $\Rbm_{\xbf}(\z) = \tau$ can be obtained using the proximal shrinkage of the \emph{fused Lasso} penalty for an appropriate regularization parameter $\lambda$ that can be computed using Newton's method.
\begin{proposition}\label{prop:projX}
Let $\z$ be any point in $\Cbb^N$, and denote by $\Rbm_{\xbf}(\z) = \|\z\|_1 + \gamma\|\z\|_{TV}$ the \emph{fused Lasso} penalty function, such that, $\Rbm_{\xbf}(\z) > \tau$ for some scalar $\tau$. Then the orthogonal projection of $\z$ onto the ball $\Fcal$ defined by $\{\u \in \Cbb^N : \Rbm_{\xbf}(\u) = \tau\}$ is obtained using the proximal shrinkage operation $\mathrm{prox}_{\lambda^* \Rbm_{\xbf}}(\z)$ defined as:
$$
	\u^* = \arg\min\limits_{\u} \frac{1}{2}\|\u - \z\|_2^2 + \eta^* \Rbm_{\xbf}(\u),
$$
where $\eta^*$ is a regularization parameter computed using Newton's root finding method.
\end{proposition}
\begin{proof}
The orthogonal projection of $\z$ onto $\Fcal$ is given by the solution to the problem
\begin{equation}\label{eq:proj}
\begin{array}{ll}
	&\min\limits_{\u} \frac{1}{2}\|\u - \z\|_2^2 \ \textrm{subject to} \ \Rbm_{\xbf}(\u) = \tau \\
	\Leftrightarrow & \max\limits_{\eta} \min\limits_{\u} \frac{1}{2}\|\u - \z\|_2^2 + \eta (\Rbm_{\xbf}(\u) - \tau) \\
	\Leftrightarrow & \max\limits_{\eta} \ \lambda\Rbm_{\xbf}(\u_{\eta})- \tau\lambda,
\end{array}
\end{equation}
where $\u_{\eta} = \mathrm{prox}_{\eta\Rbm_{\xbf}}(\z)$. Then, we need to find the $\eta$ that achieves the root of the function $f(\eta) = \Rbm_{\xbf}(\u_{\eta}) - \tau$. However, the solution $\u_{\eta}$ does not have an analytic form. Therefore, we use the variational expression 
$$
\Rbm_{\xbf}(\u_{\eta}) = \|\z_{\Omega_{\eta}}\|_1 - \|\u_{\eta}\|_0 \eta + \gamma(\|(\Ebf \z)_{\Gamma_{\eta}}\|_{2,1} - \|\Ebf\u_{\eta}\|_{2,0}\eta),
$$ 
where $\Omega_{\eta}$ is the support set of $\u_{\eta}$, and $\Gamma_{\eta}$ is the support set of the row norm vector of $\Ebf\u_{\eta}$. Consequently, we can evaluate the gradient of $f(\eta)$ with respect to $\eta$ as $g = - \|\u_{\eta}\|_0 - \gamma \|\Ebf\u_{\eta}\|_{2,0}$. Finally, the parameter $\eta$ is updated using Newton's method as $\eta = \max\left\{0, \eta - \frac{f(\eta)}{g}\right\}$ until it converges to $\eta^*$.
\end{proof}

The proximal shrinkage operation of the combined $\ell_1$ norm and total variation regularizers of $\xbf$ is performed by splitting the proximal operators into the two stages shown in steps 2 and 3 of Algorithm~\ref{alg:projX}. In the first stage, the soft-thresholding operator $\mathcal{T}\left( \z^t; \alpha\eta \right)$ is used to sparsify the signal $\z^t$, where
\begin{equation}\label{eq:soft-thresholding}
	\mathcal{T} \left( \z; \beta \right) = \left\{
	\begin{array}{ll}
		\z(j) - \beta, & \textrm{if} \ \z(j) > \beta \\
		\z(j) + \beta, & \textrm{if} \ \z(j) < -\beta \\
		0,& \textrm{otherwise.}
	\end{array}
	\right.
\end{equation} 
A second proximal operator is then applied in step 3 of the algorithm to enforce the total variation regularization. We implement this proximal operator using the alternating direction method of multipliers (ADMM) algorithm~\cite{Douglas1956NumerSol,ADMM_Boyd:2011}.

\begin{algorithm}[ht]
\caption{\texttt{fpgd} subroutine for updating $\xbf$}\label{alg:fpgdX}
\begin{algorithmic}[1]
\Statex{\textbf{input: } $\{\Acal^m_{\xbf}\}_{m=1}^M,  \Rbm_{\xbf}, \{\widetilde{\ybf}_m\}_{m=1}^M, \tau^j, \xbf^{j-1},T$.}
\Statex{\textbf{set: } $q_0 = 1$, $\u^0 = \s^0 = \xbf^{j-1}$}
	\State $\alpha \leftarrow$ inverse of maximum eigenvalue of $\sum\limits_{m=1}^M\Acal^{m \H}_{\xbf} \Acal^m_{\xbf}$
	\For{$t \gets 1 \textrm{ to } T$}
	\State $\z^t \leftarrow s^{t-1} + \alpha \sum\limits_{m=1}^M\Acal^{m \H}_{\xbf}\left(  \widetilde{\ybf}_m - \Acal^m_{\xbf} \s^{t-1} \right)$
	\State $\u^t \leftarrow \textrm{proj}_{\Rbm_{\xbf}}(\z^t, \tau^j)$
	\State $q_t \leftarrow \frac{1 + \sqrt{1 + 4q_{t-1}^2}}{2}$
	\State $\s^t \leftarrow \u^t + \frac{q_{t-1} - 1}{q_t}(\u^t - \u^{t-1})$
	\EndFor
	\Statex{\textbf{return:} $\u^T$}
\end{algorithmic}
\end{algorithm}

\begin{algorithm}[ht]
\caption{$\textrm{proj}_{\Rbm_{\xbf}}:$ constrained \emph{fused Lasso} projector}\label{alg:projX}
\begin{algorithmic}[1]
\Statex{\textbf{input: } $\z,  \textrm{finite difference operator } \Ebf,  \tau, \alpha$.}
\Statex{\textbf{set: } $\lambda = 0, \u = \z$}
	\While{$\Rbm_{\xbf}(\u) > \tau$}
	\Statex Compute proximal shrinkage
	\State $\v \leftarrow \mathcal{T}\left( \z; \alpha\lambda \right)$
	\State $\u \leftarrow \arg\min\limits_{\widetilde{\u} \in \Cbb^N} \left\{ \frac{1}{2} \|\widetilde{\u} - \v\|_2^2 + \alpha\lambda\gamma \|\Ebf \widetilde{\u}\|_{2,1})\right\}$
	\Statex Update $\lambda$
	\State $g \leftarrow -\|\u\|_0 - \gamma\|\Ebf\u\|_{2,0}$
	\State $\lambda \leftarrow \max\left\{0, \lambda - \frac{\Rbm_{\xbf}(\u) - \tau}{\alpha g}\right\}$
	\EndWhile
	\Statex{\textbf{return:} $\u$}
\end{algorithmic}
\end{algorithm}

\section{Performance Evaluation}\label{sec:Results}
In this section we evaluate the performance of our approach using both
simulated data and real experimental radar data.

\subsection{Simulated data}
We simulate a radar scene acquired by 32 distributed antennas divided into four arrays as shown in Figure~\ref{fig:introLayout}. The true antenna positions are indicated by the $\times$'s whereas the erroneous assumed positions are indicated by the dots. The average absolute value of the position error of the antennas is around $2\lambda$ with a maximum error of $3.5\lambda$, where $\lambda$ is the wavelength of the center frequency of a differential Gaussian pulse centered at 6 GHz with a 9 GHz bandwidth. The received signals are contaminated with white Gaussian noise at 4dB, 6dB, 8dB, 10dB, 15dB and 20dB peak signal to noise ratio (PSNR) after matched-filtering with the transmitted pulse. 

\begin{figure}[ht]
\hspace{-0.2in}\includegraphics[width=4.05in]{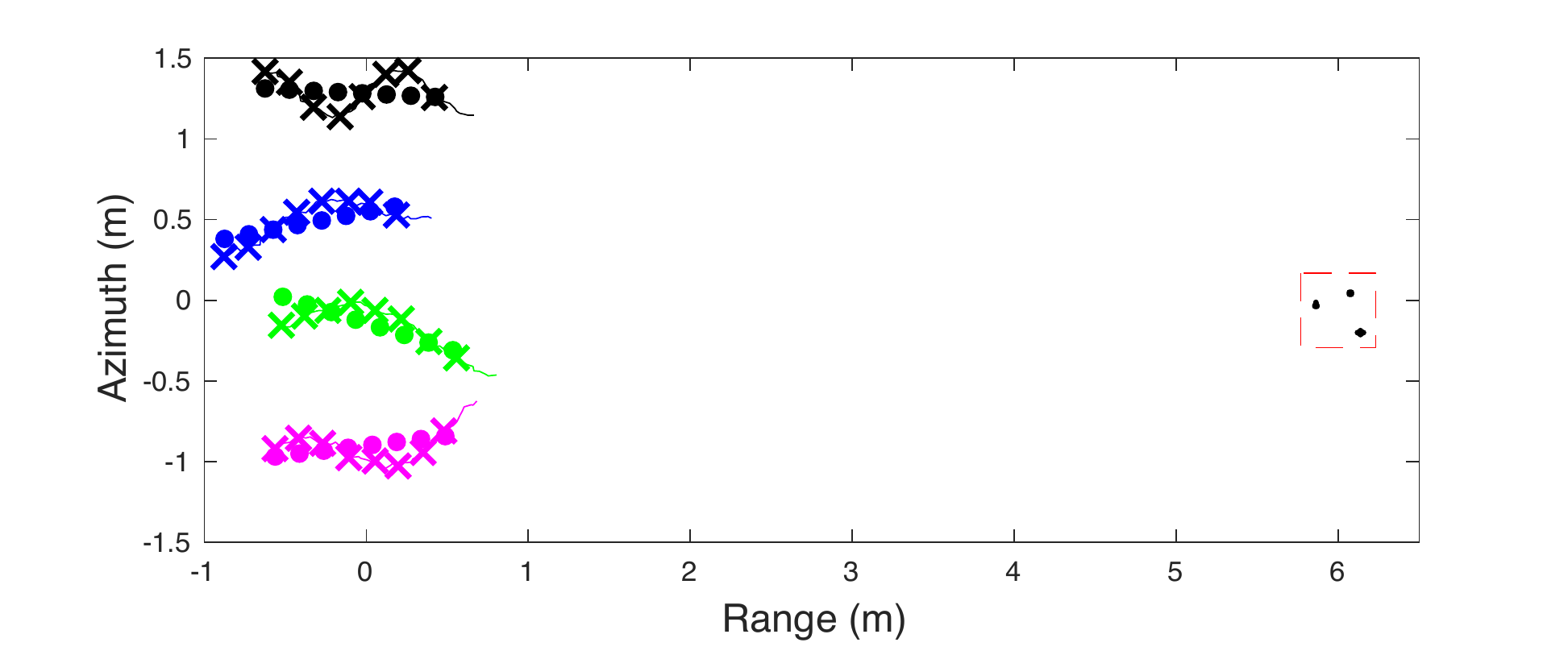}
\caption{\small A distributed radar acquisition system with position ambiguity. The round dots indicate the assumed but erroneous antenna positions, while the $\times$'s indicate the true positions.}\label{fig:introLayout}\vspace{-0.2in}
\end{figure}

We generate ten different simulation layouts by randomly perturbing the positions of the antennas in the four arrays and constructing the five different target layouts shown in Figure~\ref{fig:ImageingResults}. We compute the reconstructed images using the \emph{fused Lasso} regularized least squares problem for both the ground truth and erroneous antenna positions. We also generate the imaging result from the measurement-domain blind deconvolution model in~\eqref{eq:CombinedSystem} by implementing a solver for a sparsity regularized version of the linear least squares problem of~\cite{BN:2007,BlindDeconvLLS:2016}. Figure~\ref{fig:ImageingResults} shows the reconstructed images using the four methods above when the measurements are contaminated with 15dB PSNR noise. The figure illustrates that both the measurement-domain blind deconvolution method and our proposed image-domain blind deconvolution method are successful at recovering the target image. However, it can also be seen that our proposed method is much more robust to noise. For further validation, we present target detection receiver-operating-characteristic (ROC) curves in Figure~\ref{fig:Sims_ROC} that demonstrate the superior performance of our method under the different noise levels. 
\begin{figure*}
\centering
\mbox{
\includegraphics[width=0.18\textwidth]{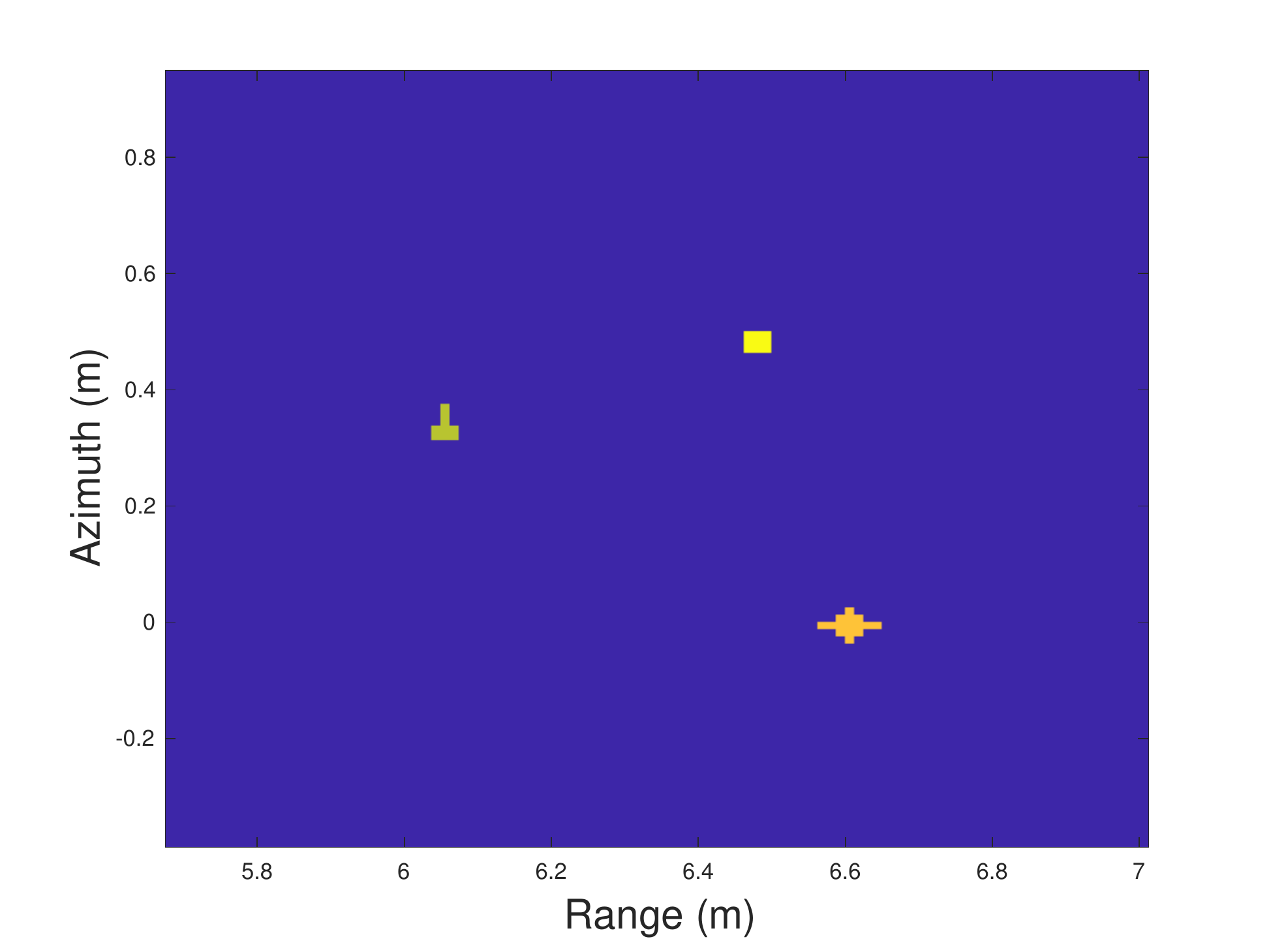}
\includegraphics[width=0.18\textwidth]{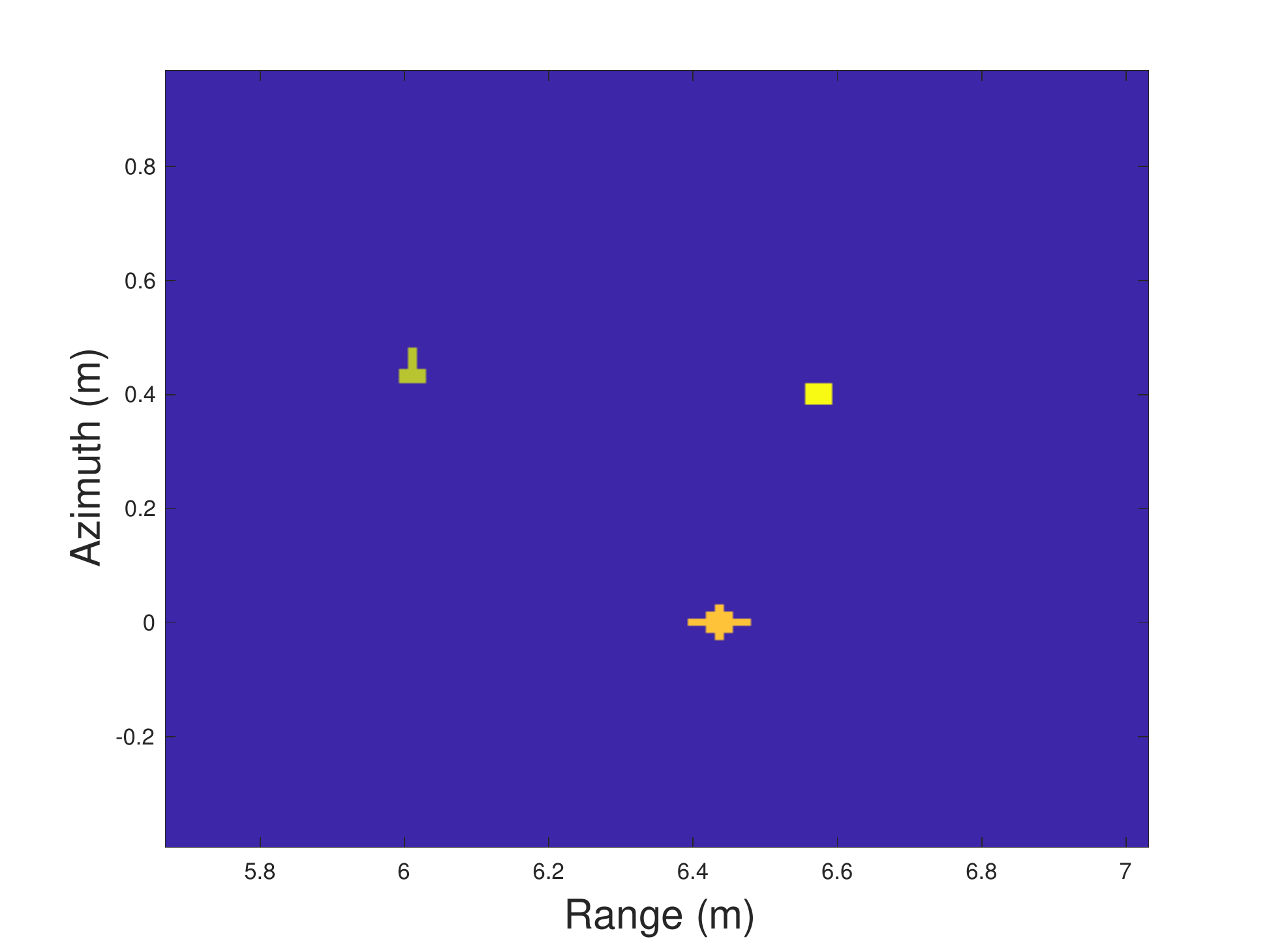}
\includegraphics[width=0.18\textwidth]{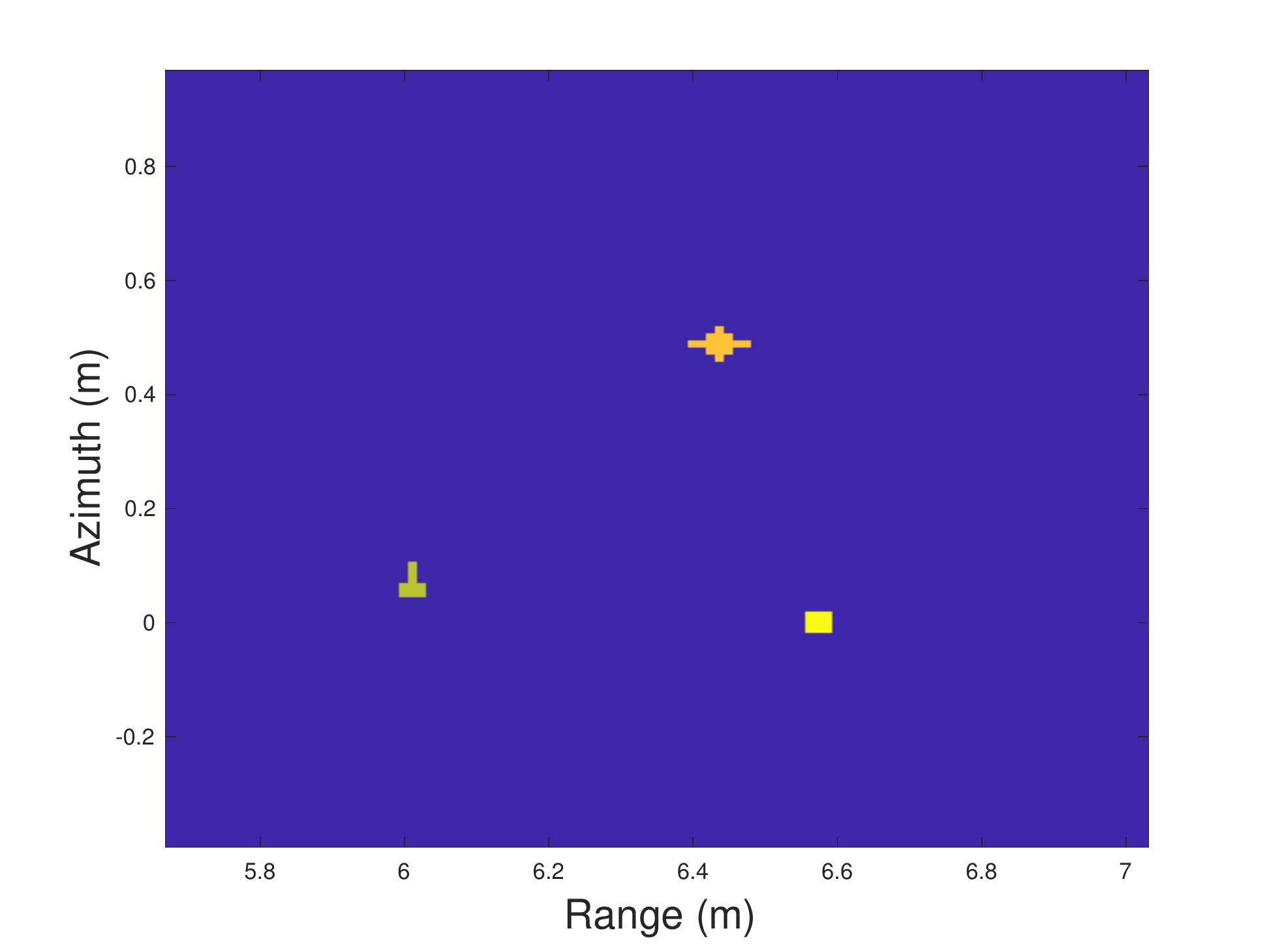}
\includegraphics[width=0.18\textwidth]{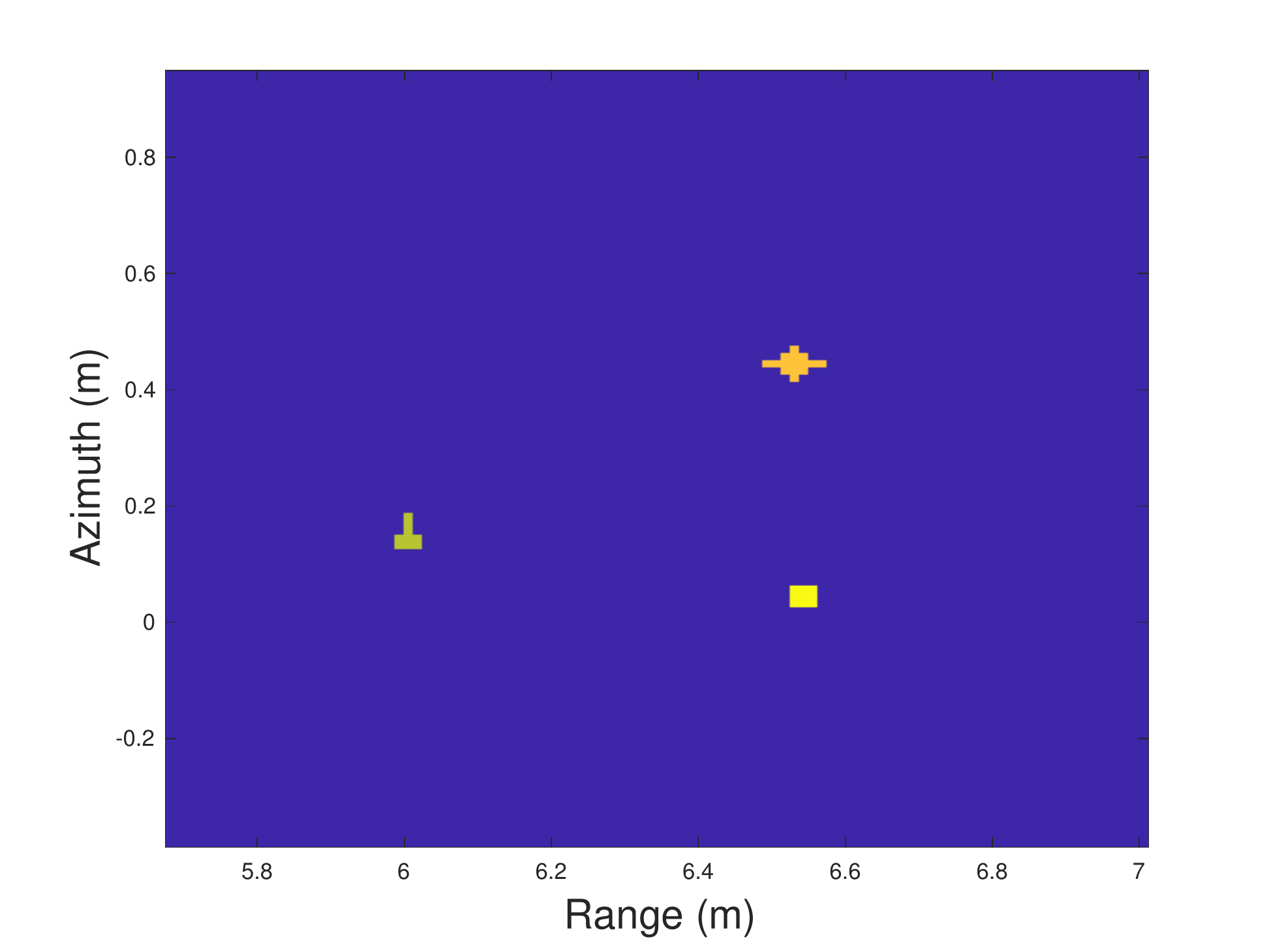}
\includegraphics[width=0.18\textwidth]{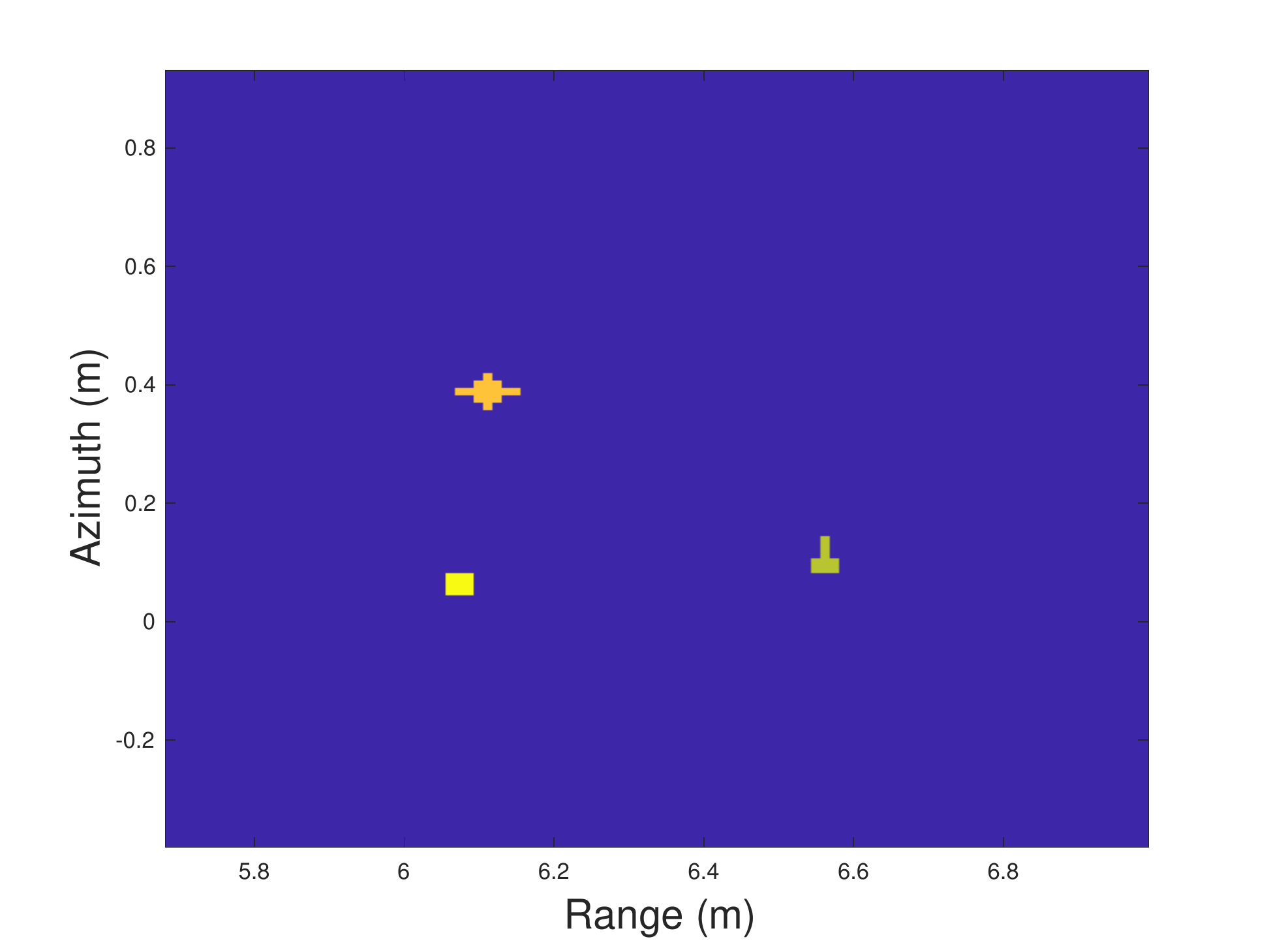}
}
\mbox{
\includegraphics[width=0.18\textwidth]{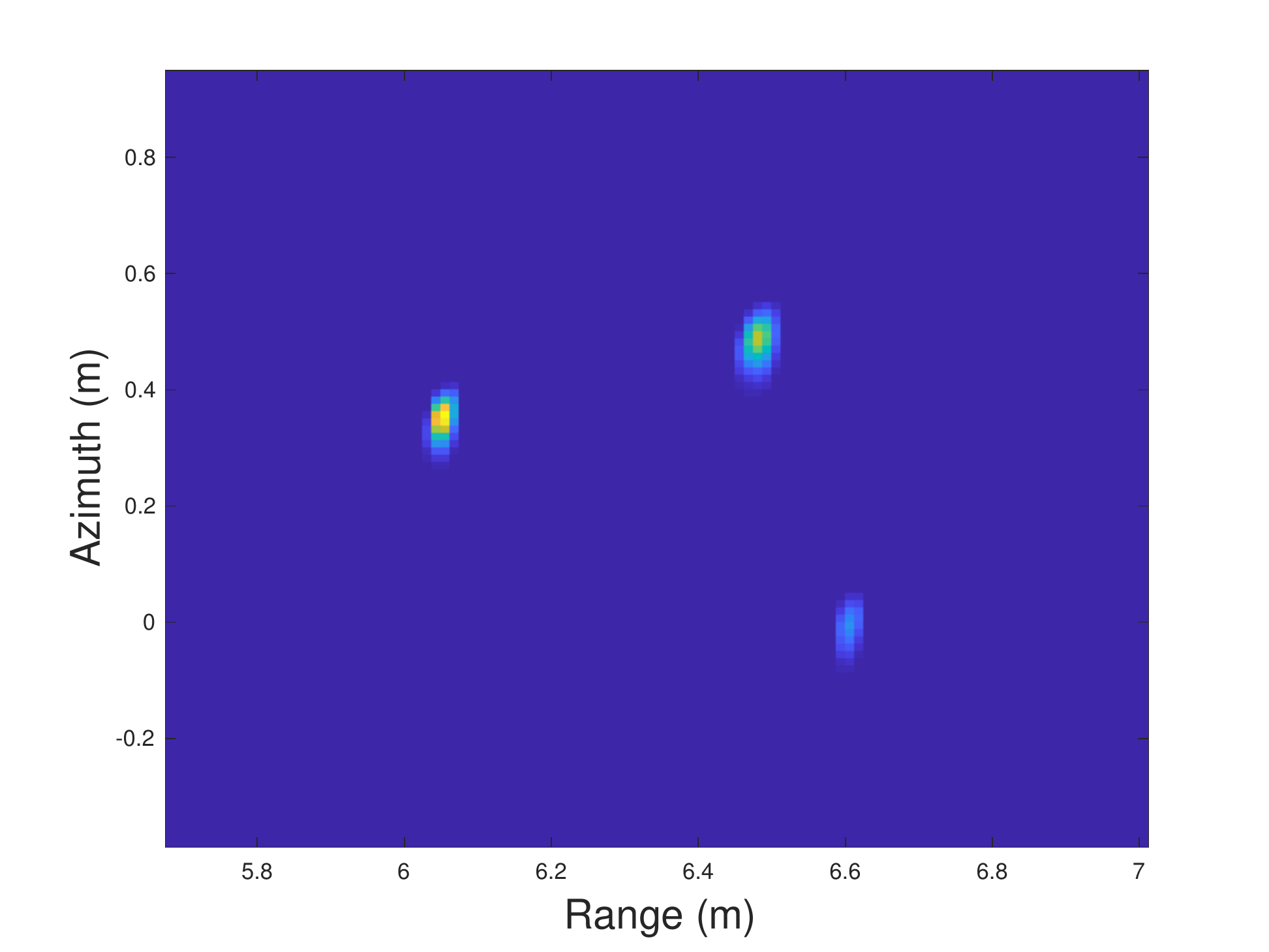}
\includegraphics[width=0.18\textwidth]{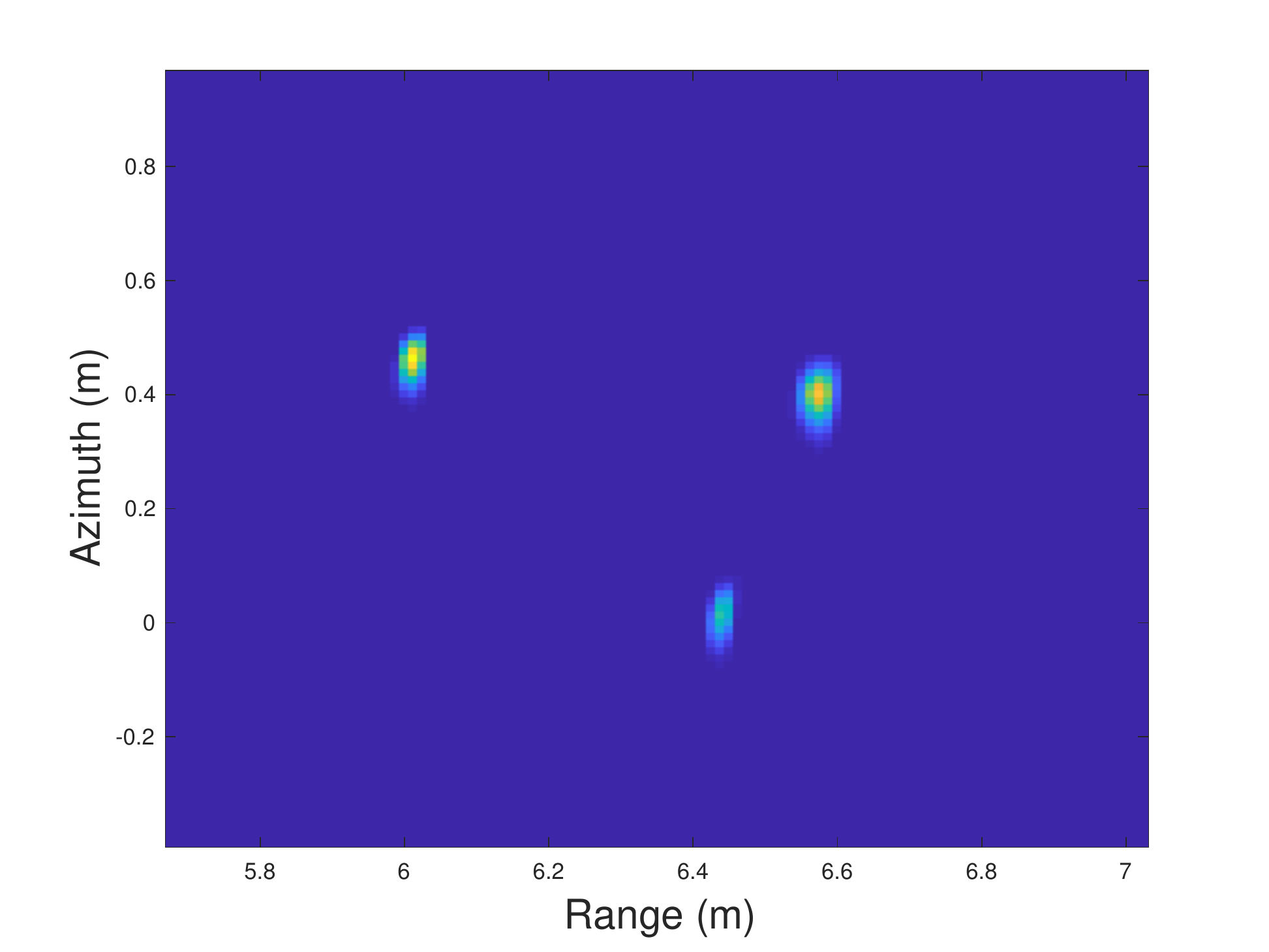}
\includegraphics[width=0.18\textwidth]{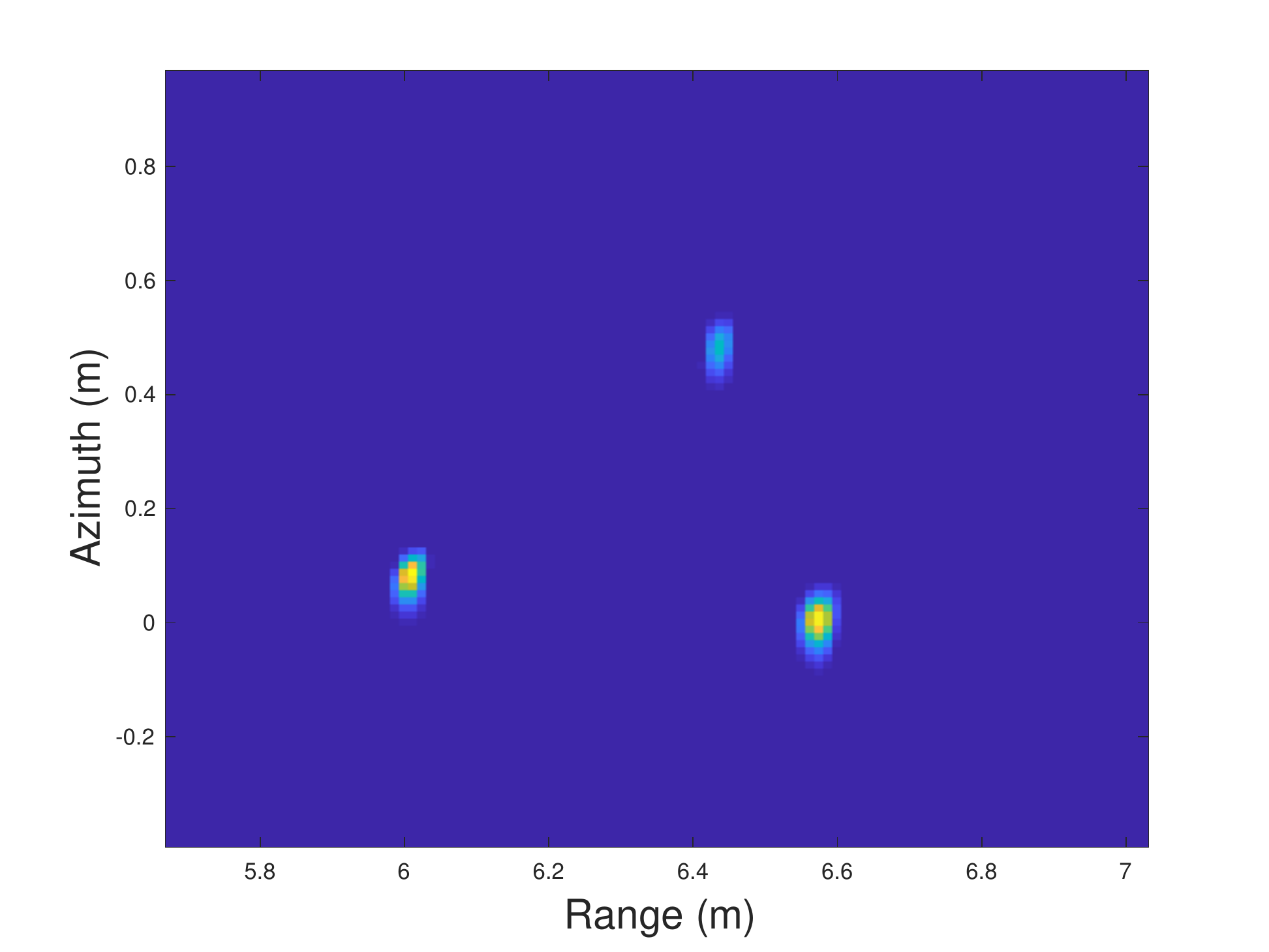}
\includegraphics[width=0.18\textwidth]{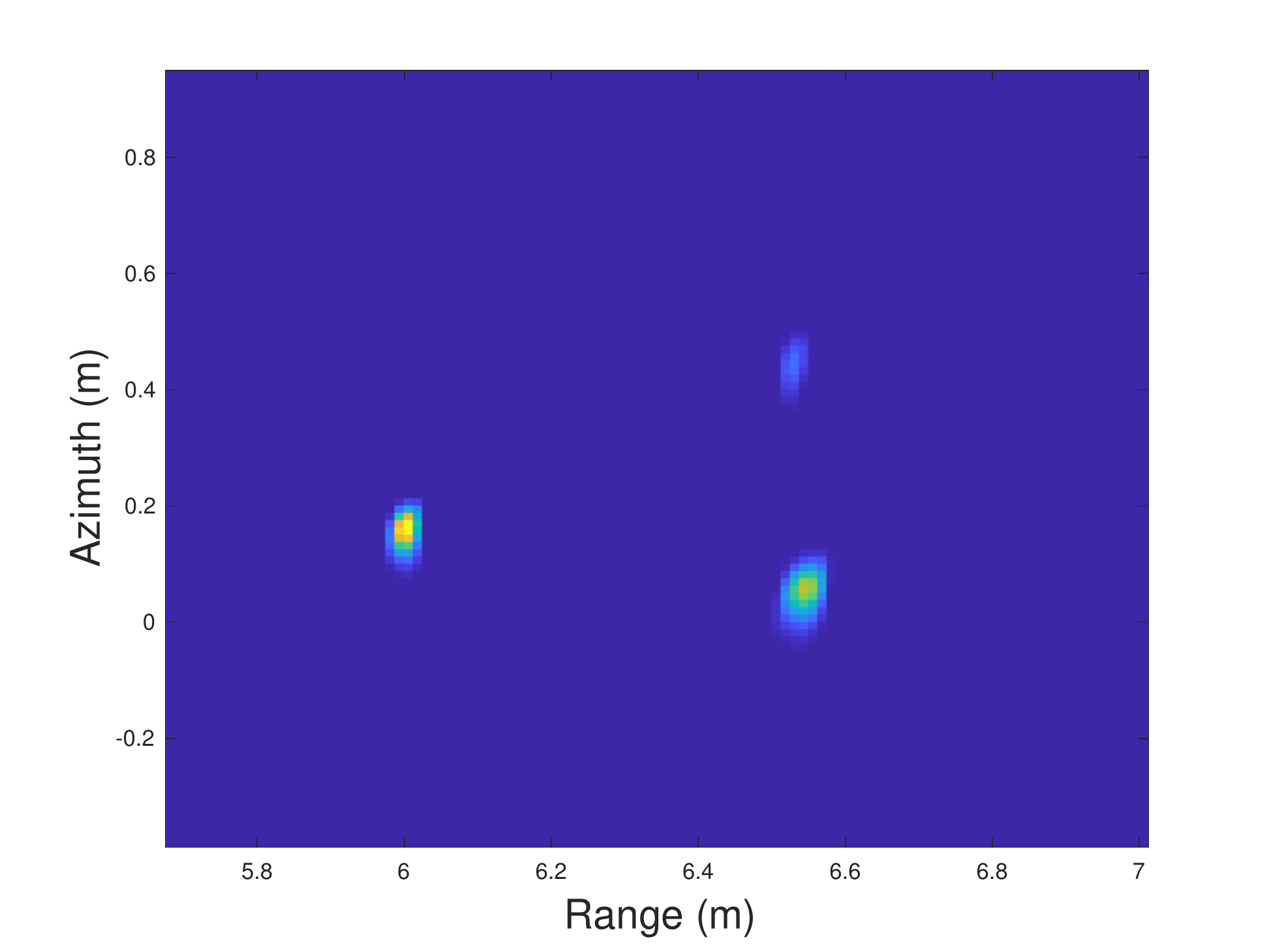}
\includegraphics[width=0.18\textwidth]{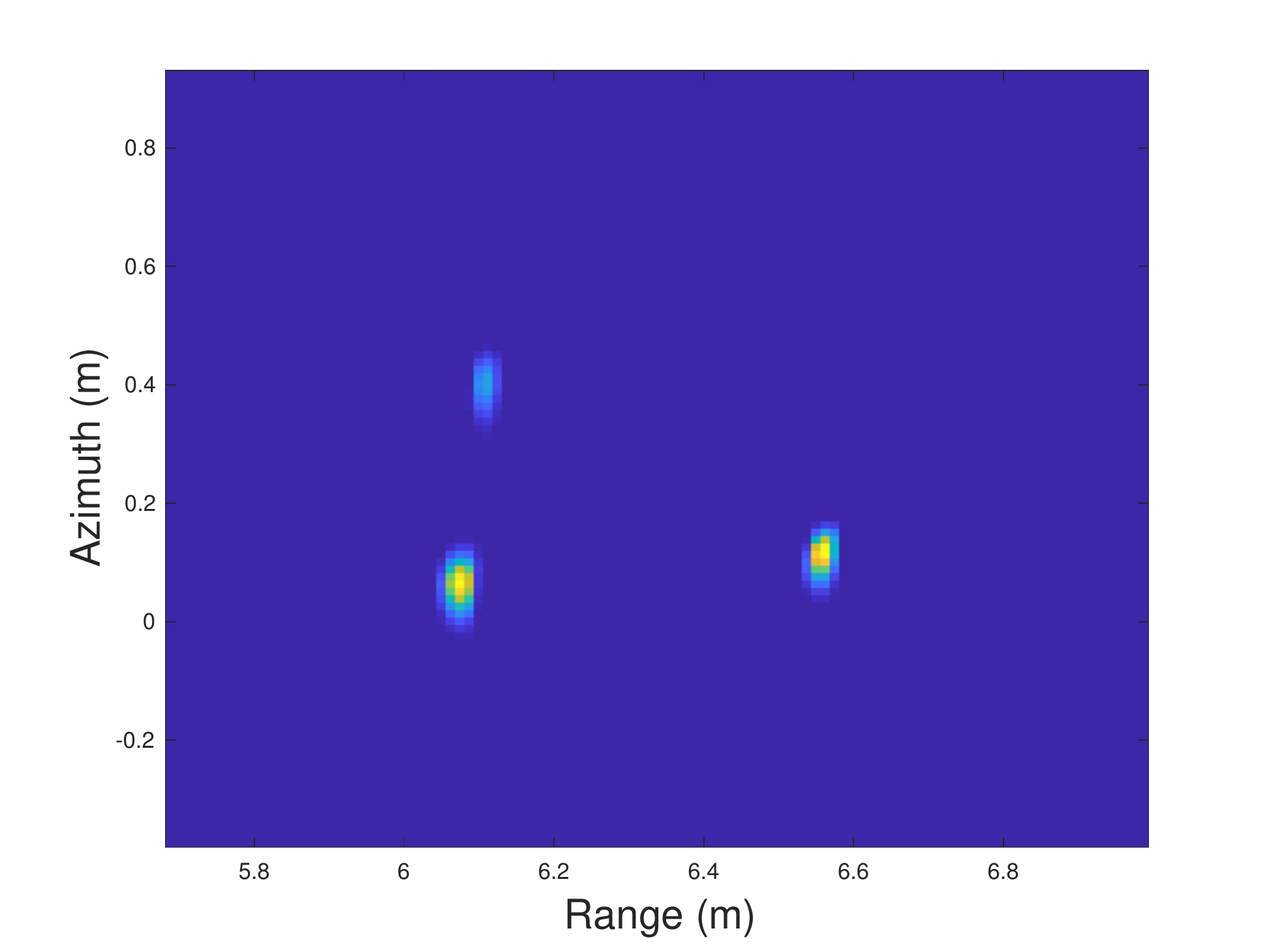}
}
\mbox{
\includegraphics[width=0.18\textwidth]{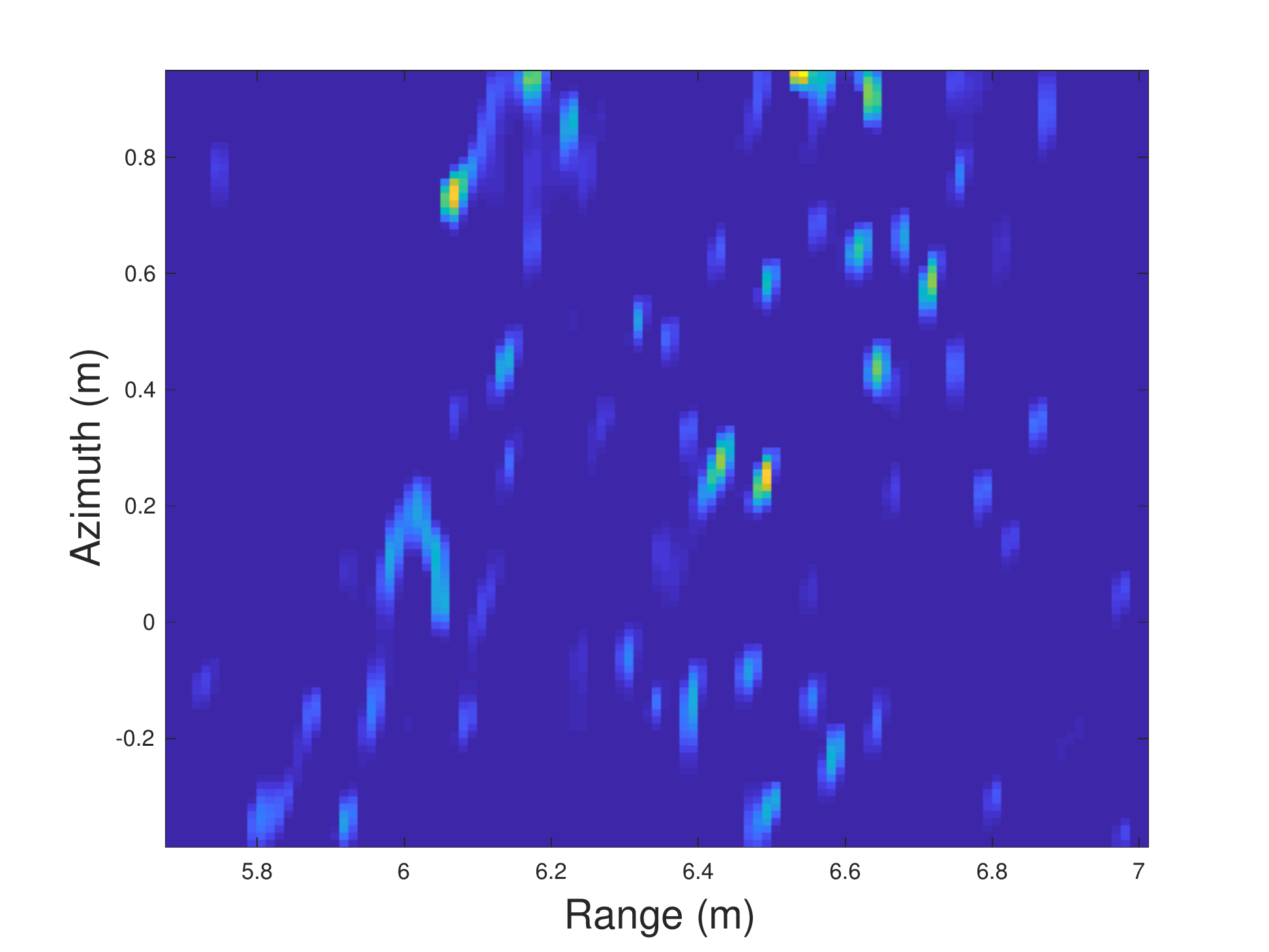}
\includegraphics[width=0.18\textwidth]{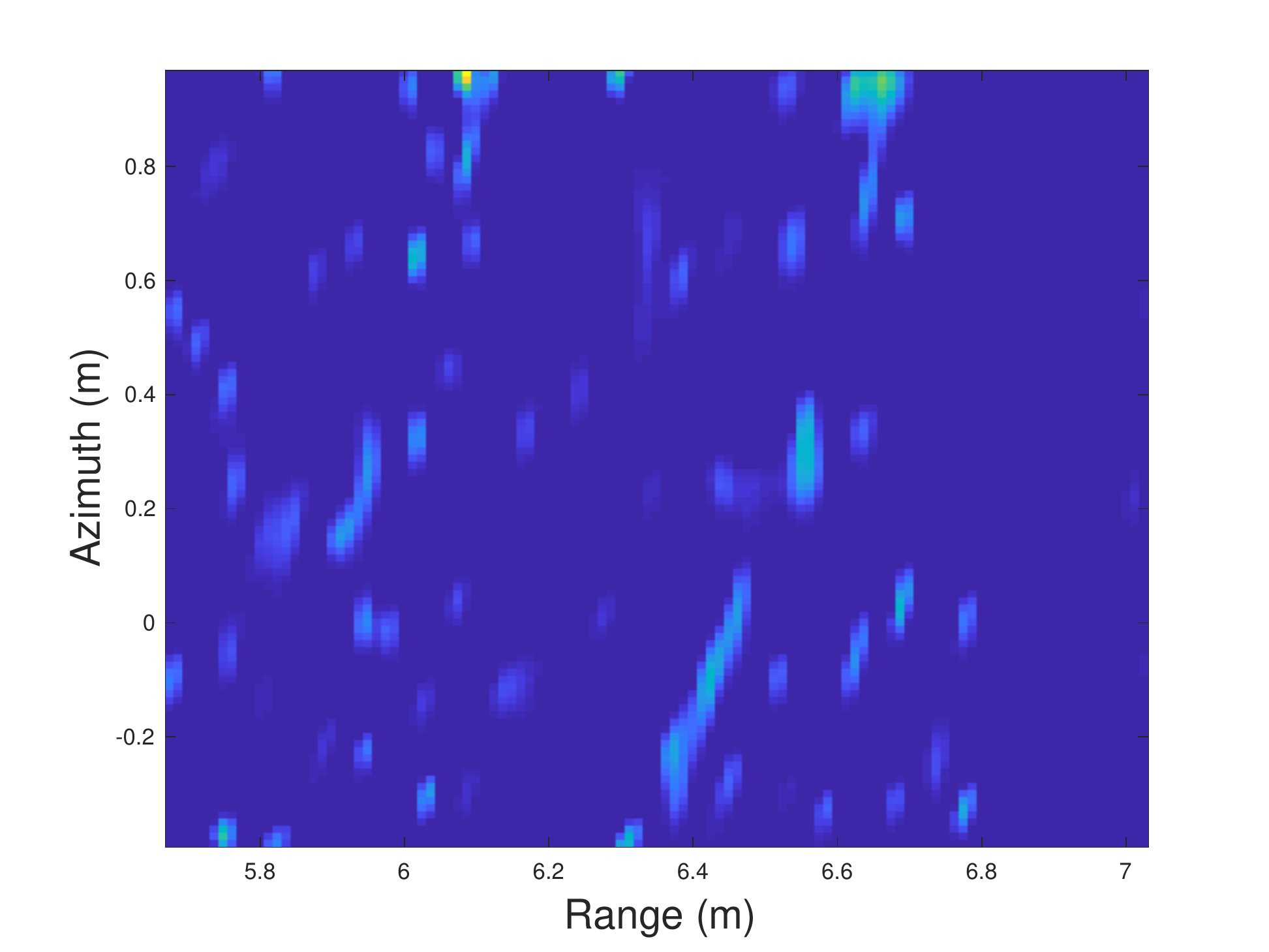}
\includegraphics[width=0.18\textwidth]{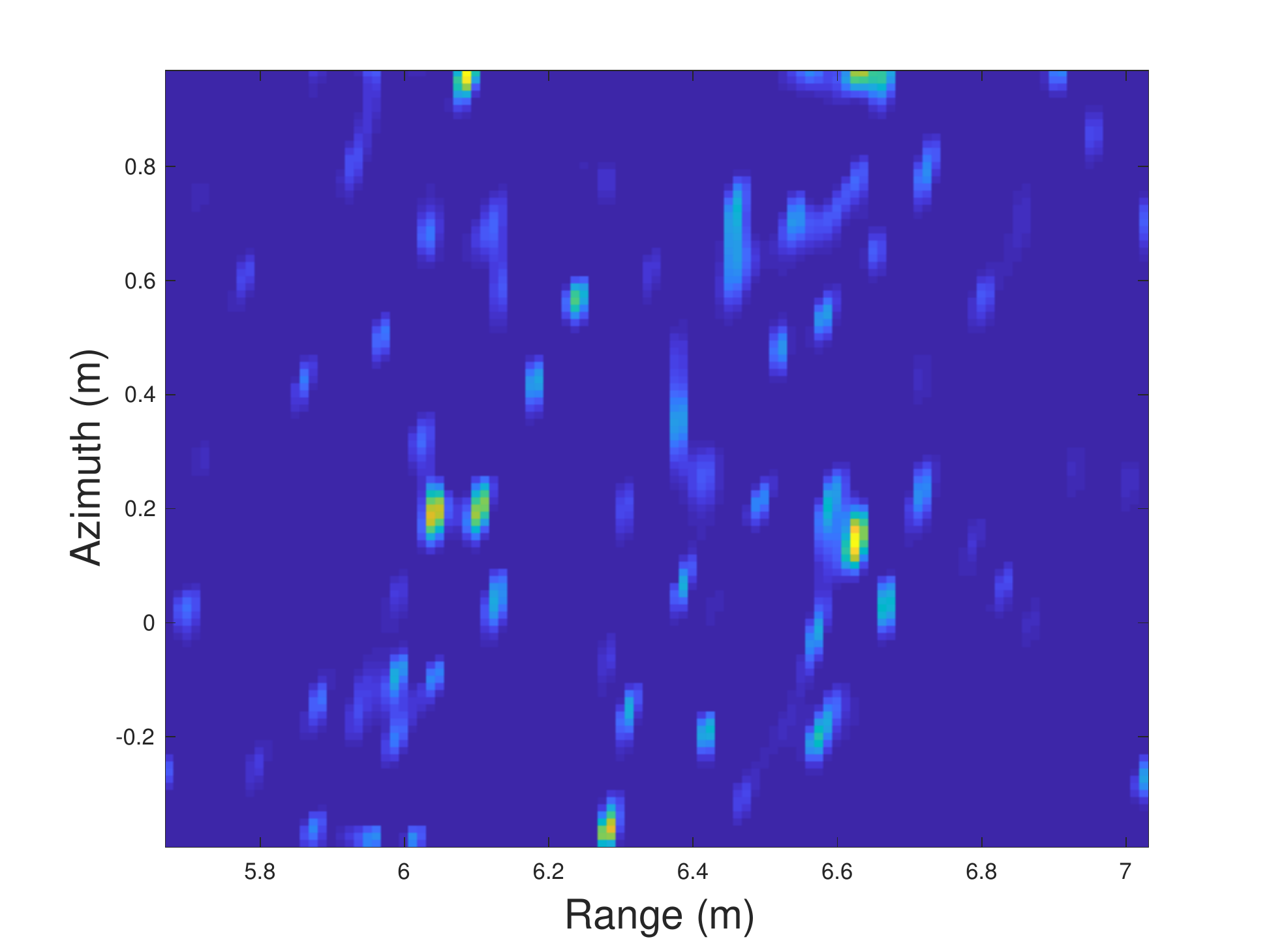}
\includegraphics[width=0.18\textwidth]{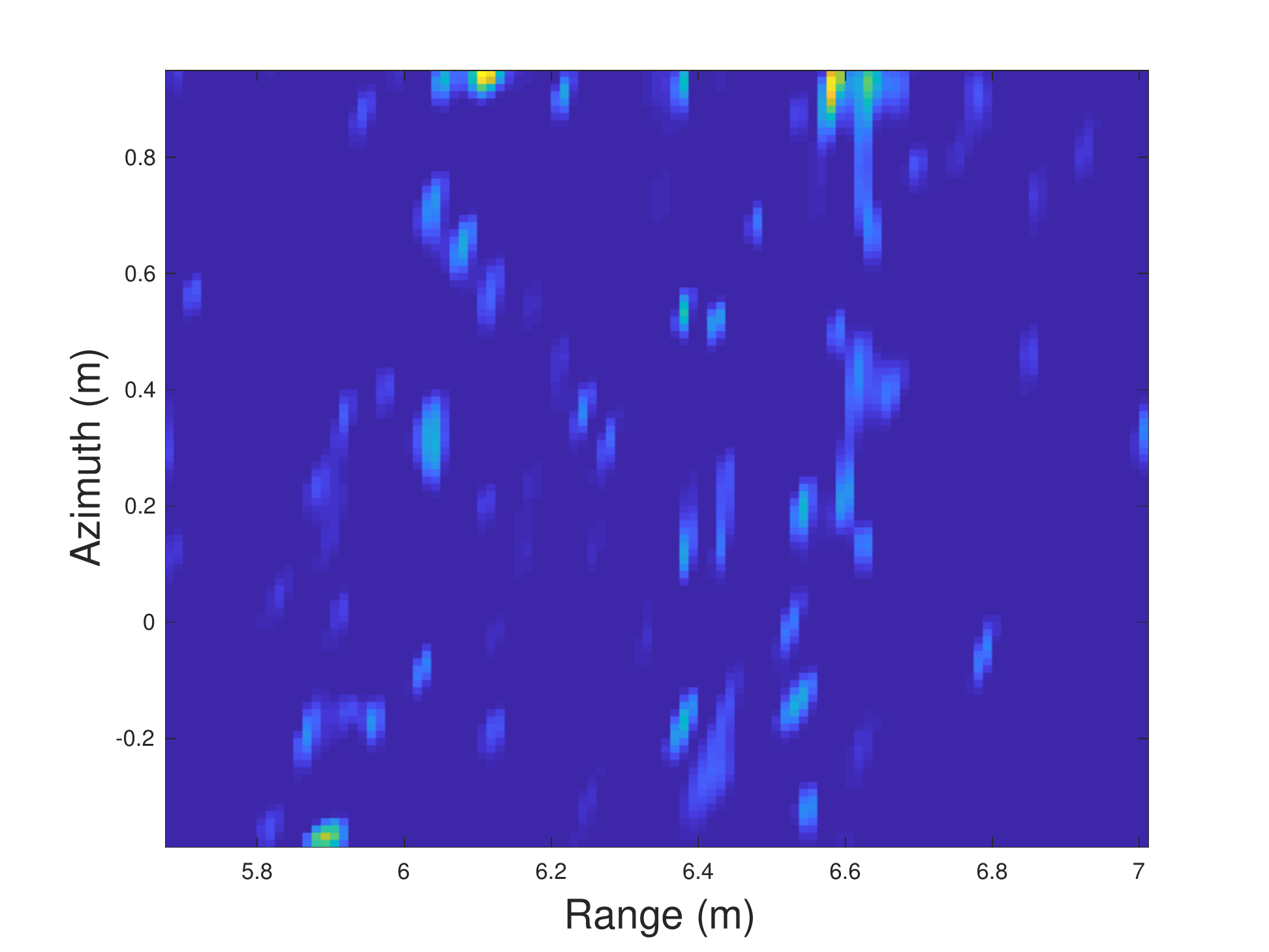}
\includegraphics[width=0.18\textwidth]{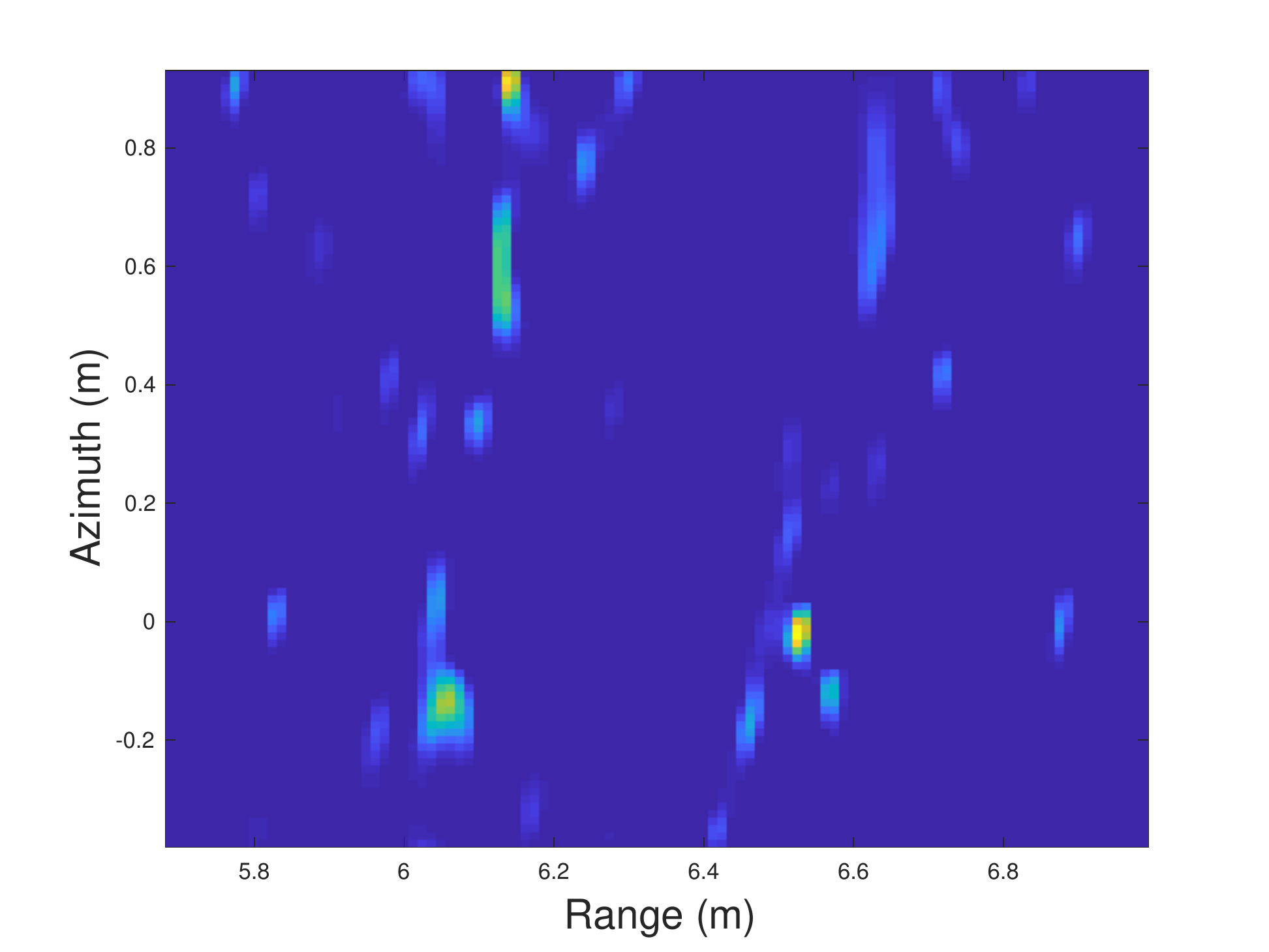}
}
\mbox{
\includegraphics[width=0.18\textwidth]{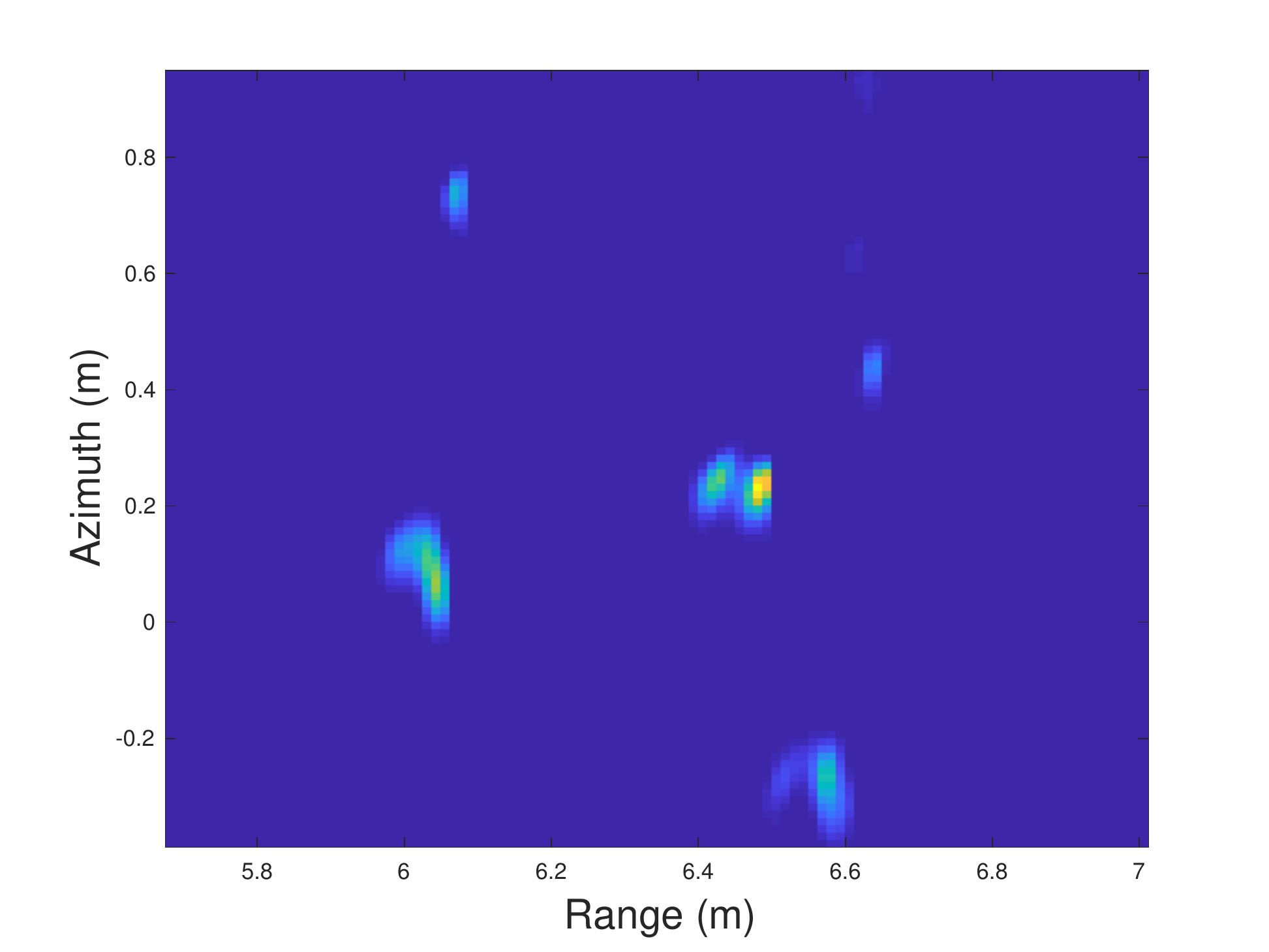}
\includegraphics[width=0.18\textwidth]{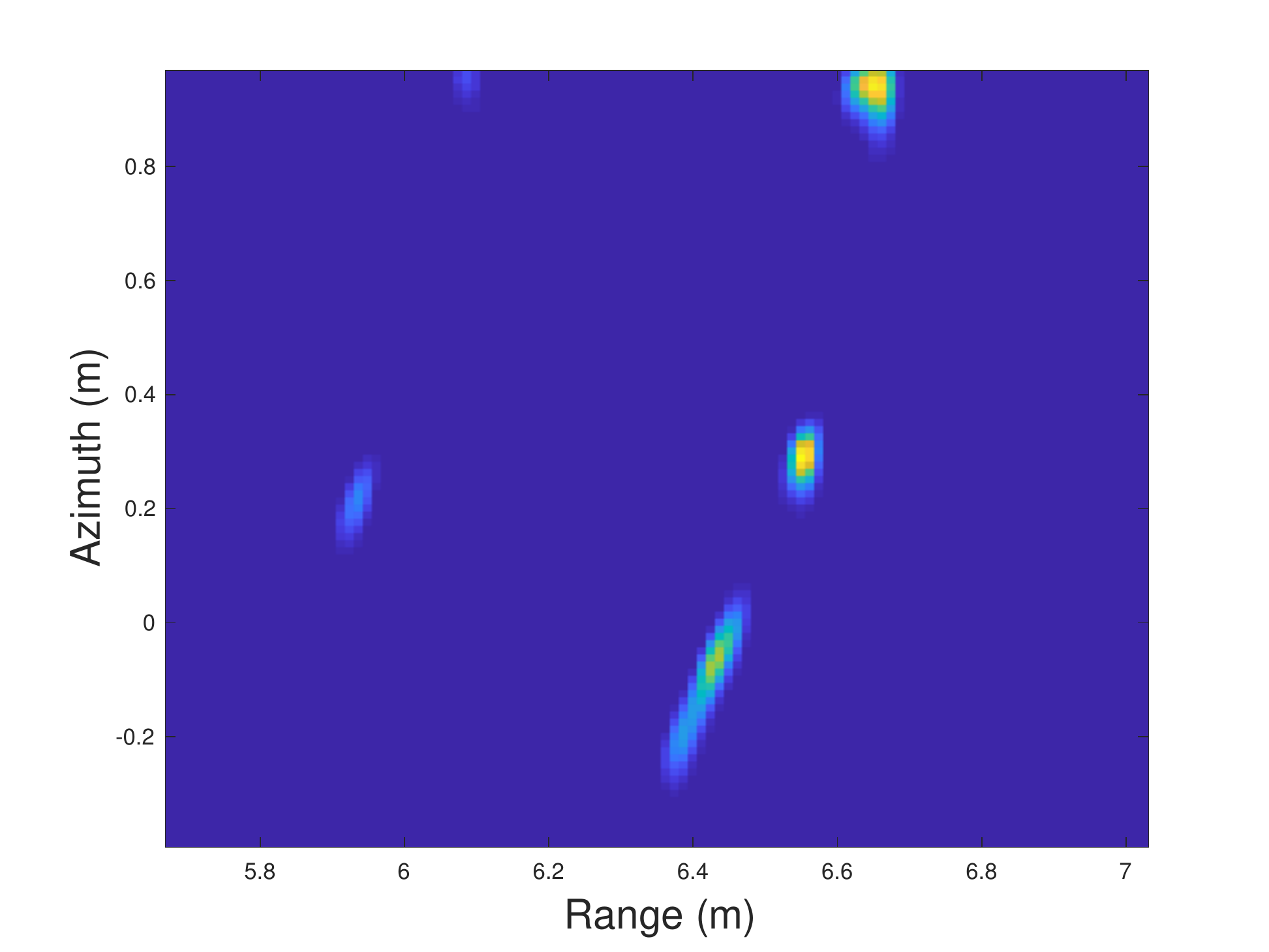}
\includegraphics[width=0.18\textwidth]{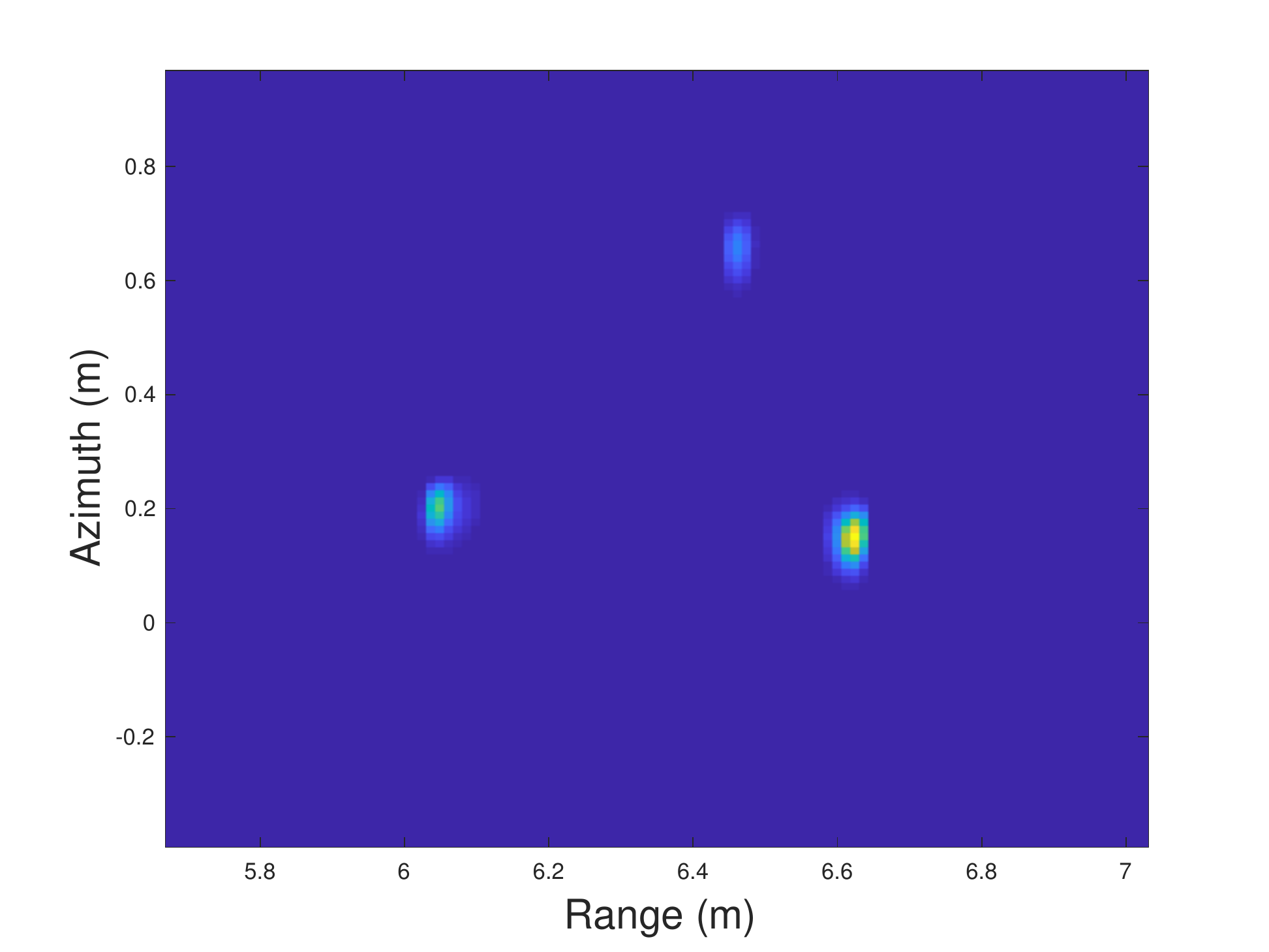}
\includegraphics[width=0.18\textwidth]{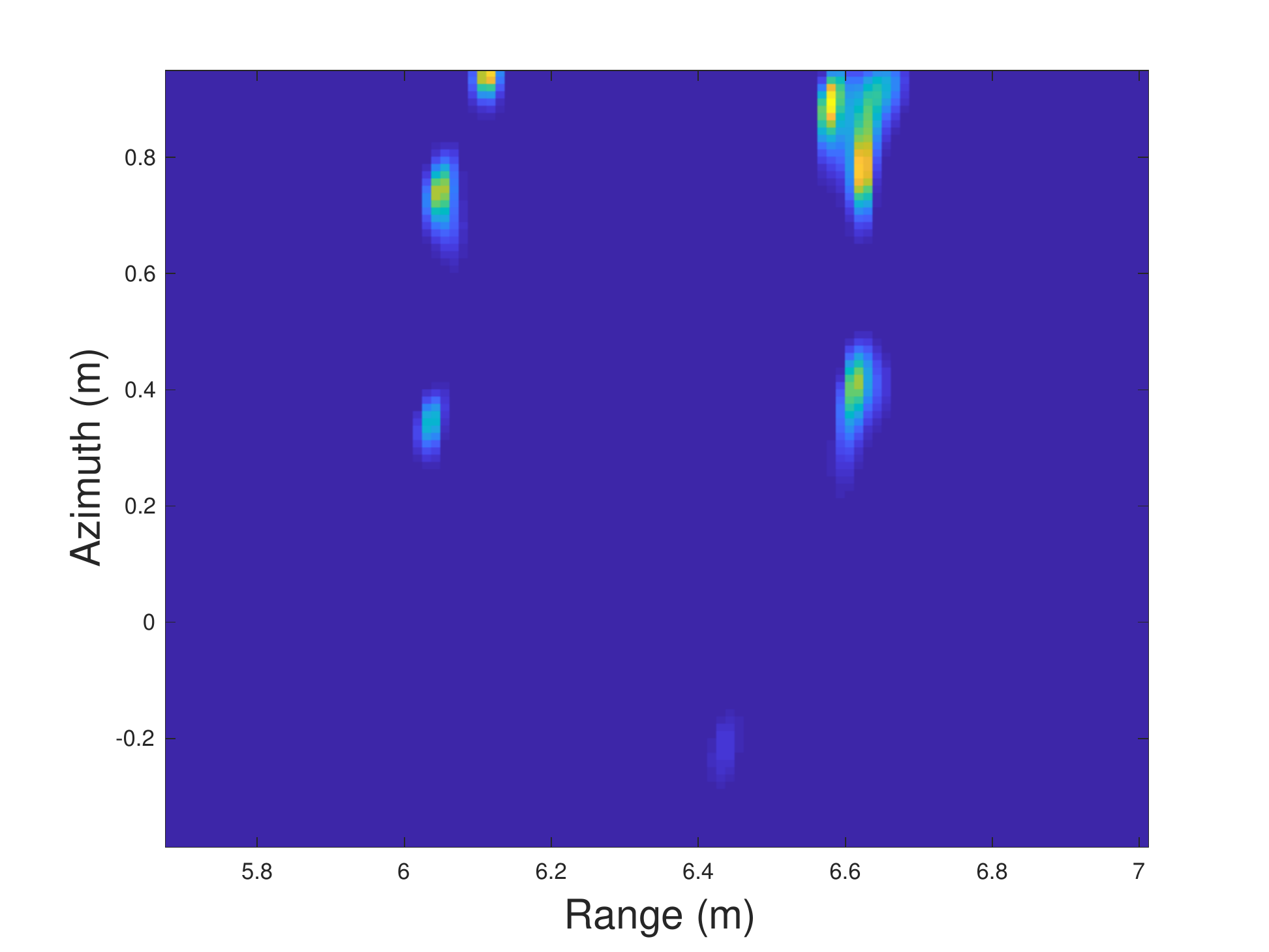}
\includegraphics[width=0.18\textwidth]{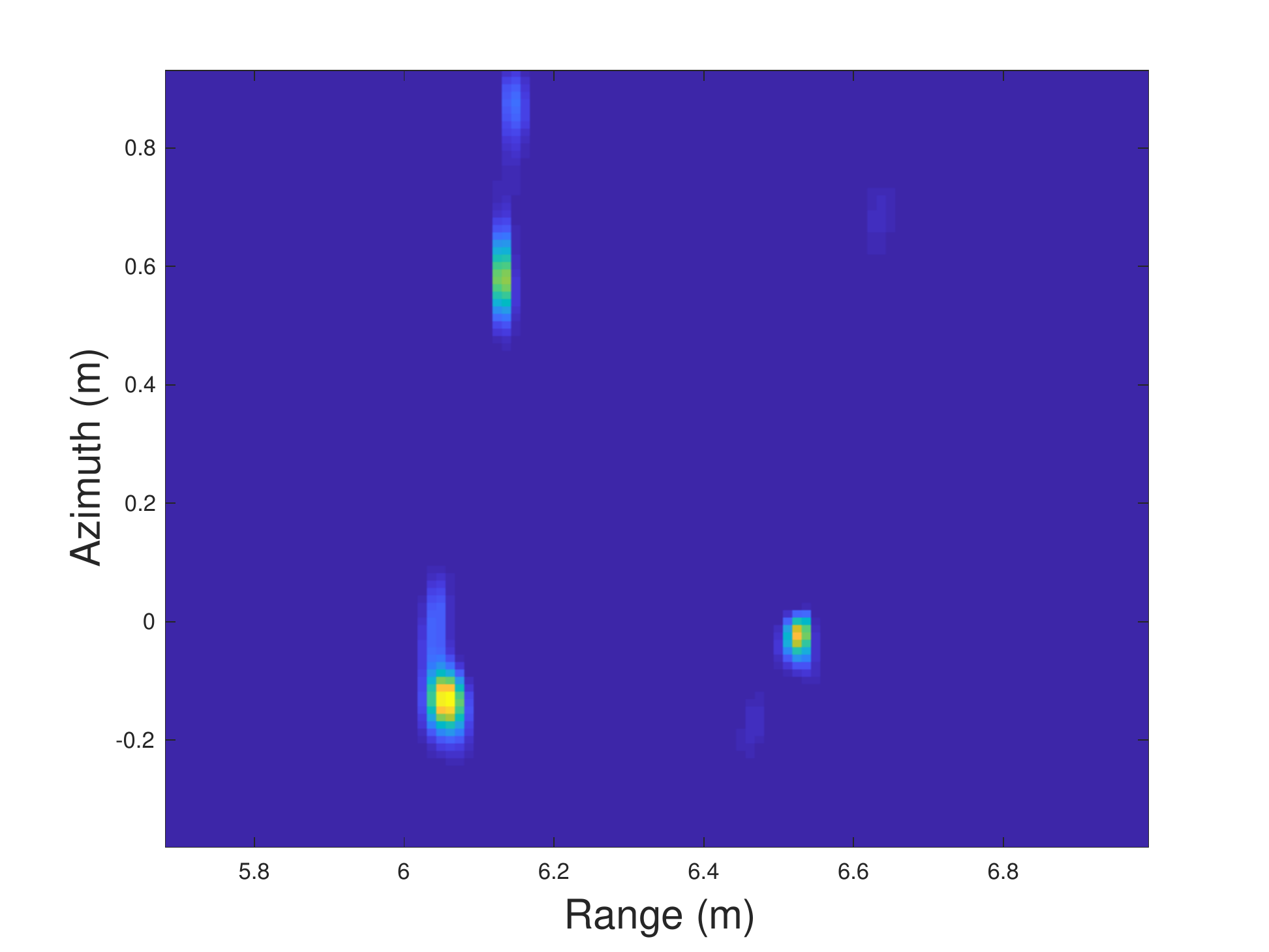}
}
\mbox{
\includegraphics[width=0.18\textwidth]{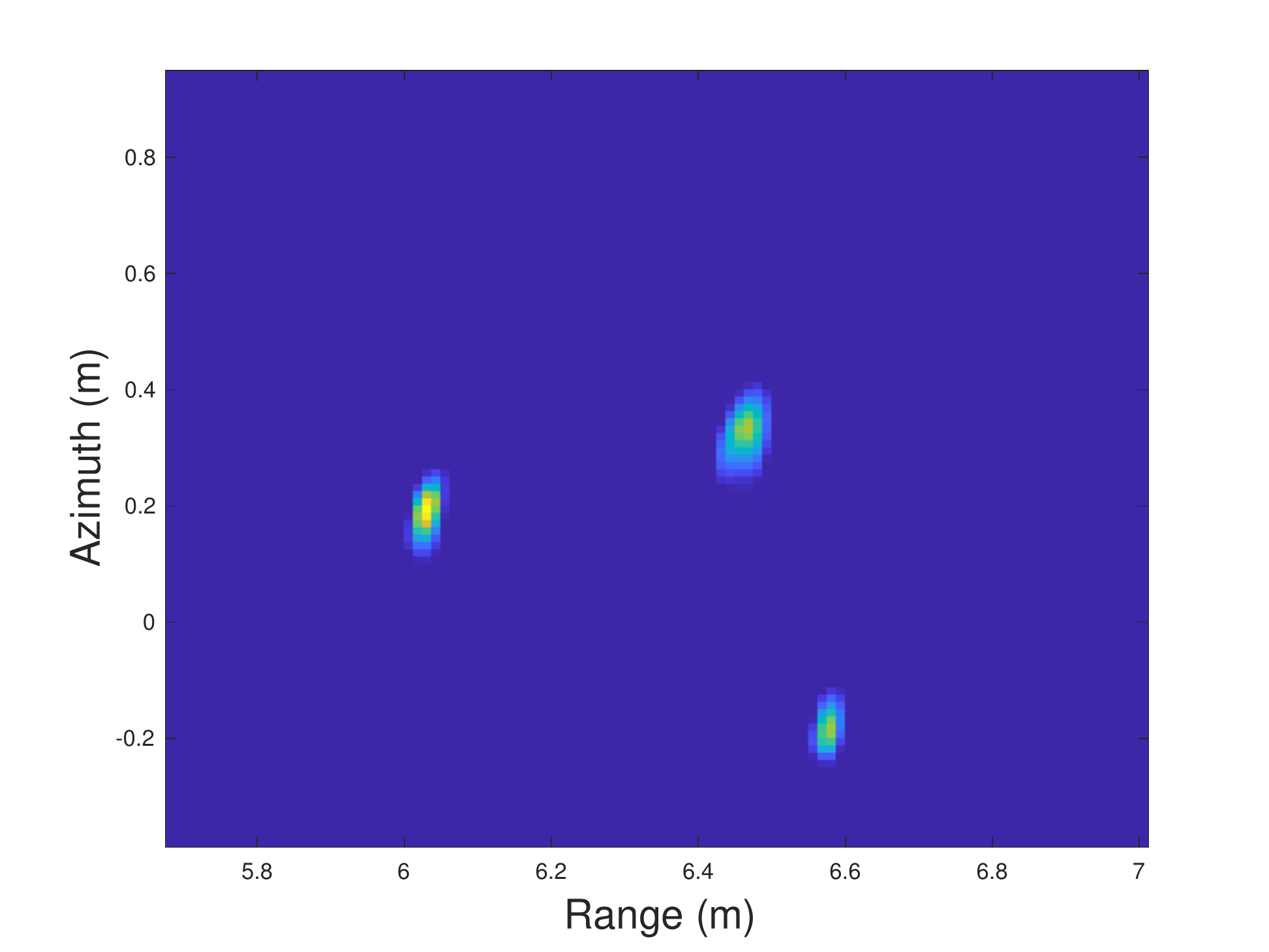}
\includegraphics[width=0.18\textwidth]{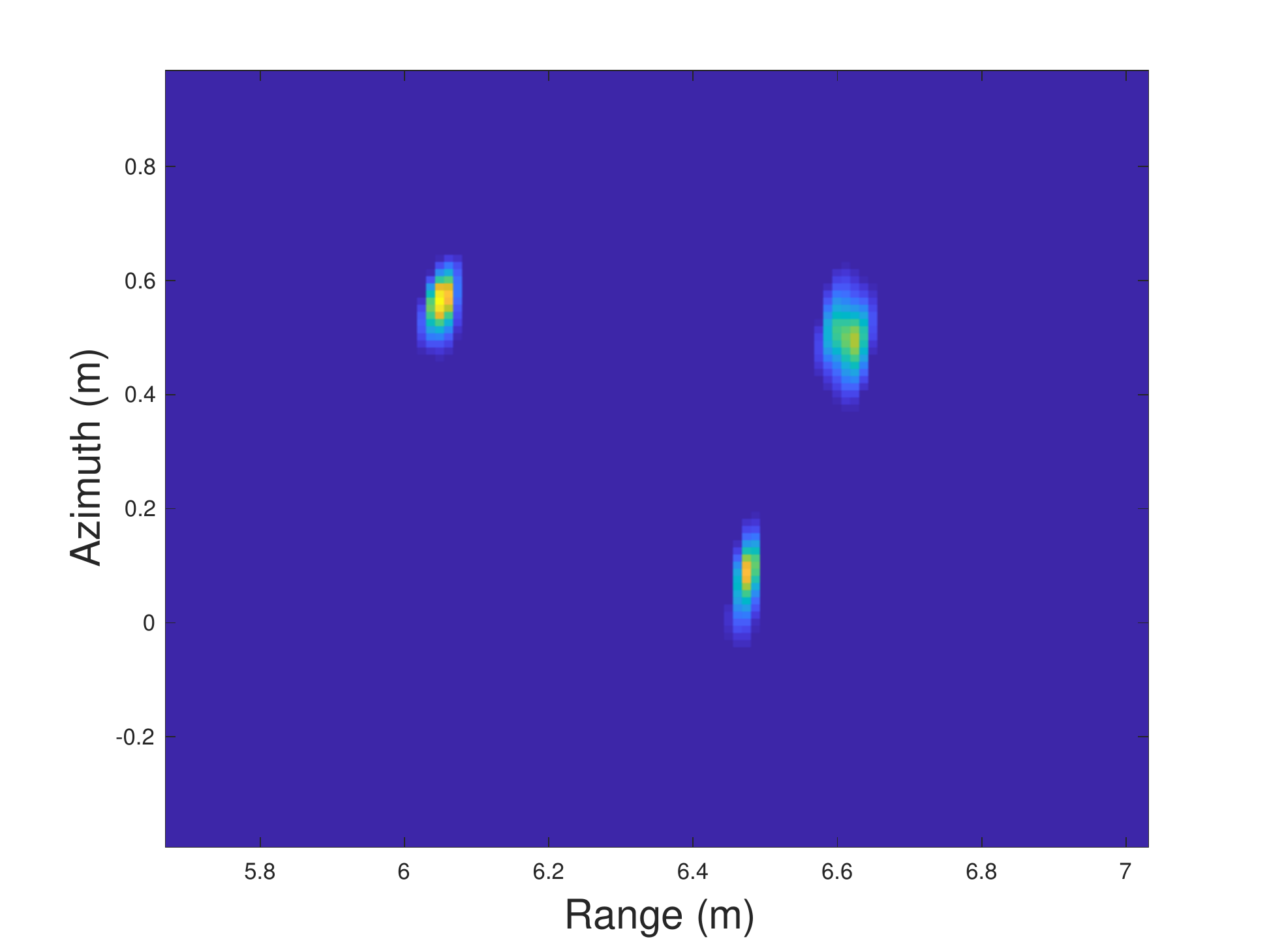}
\includegraphics[width=0.18\textwidth]{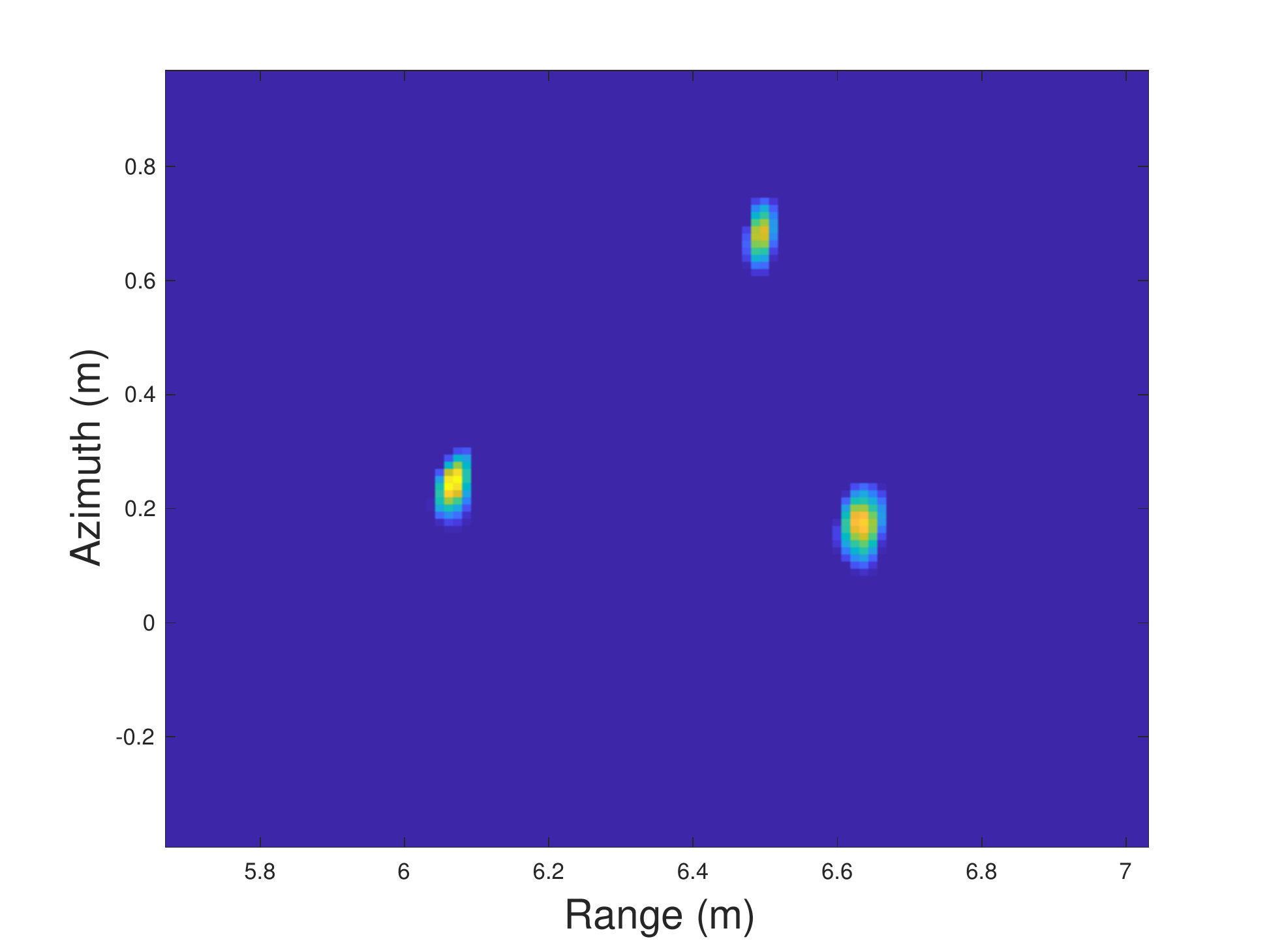}
\includegraphics[width=0.18\textwidth]{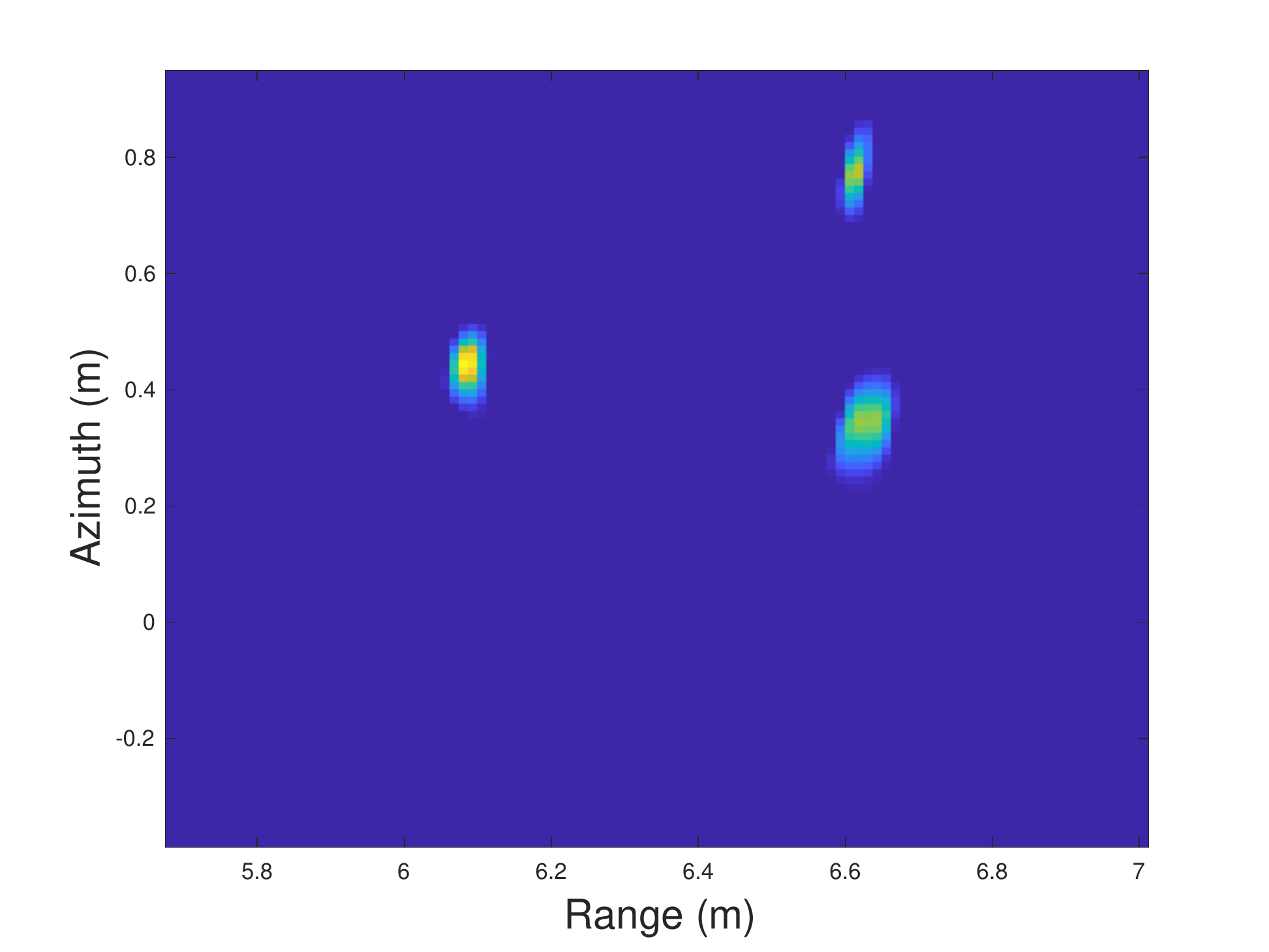}
\includegraphics[width=0.18\textwidth]{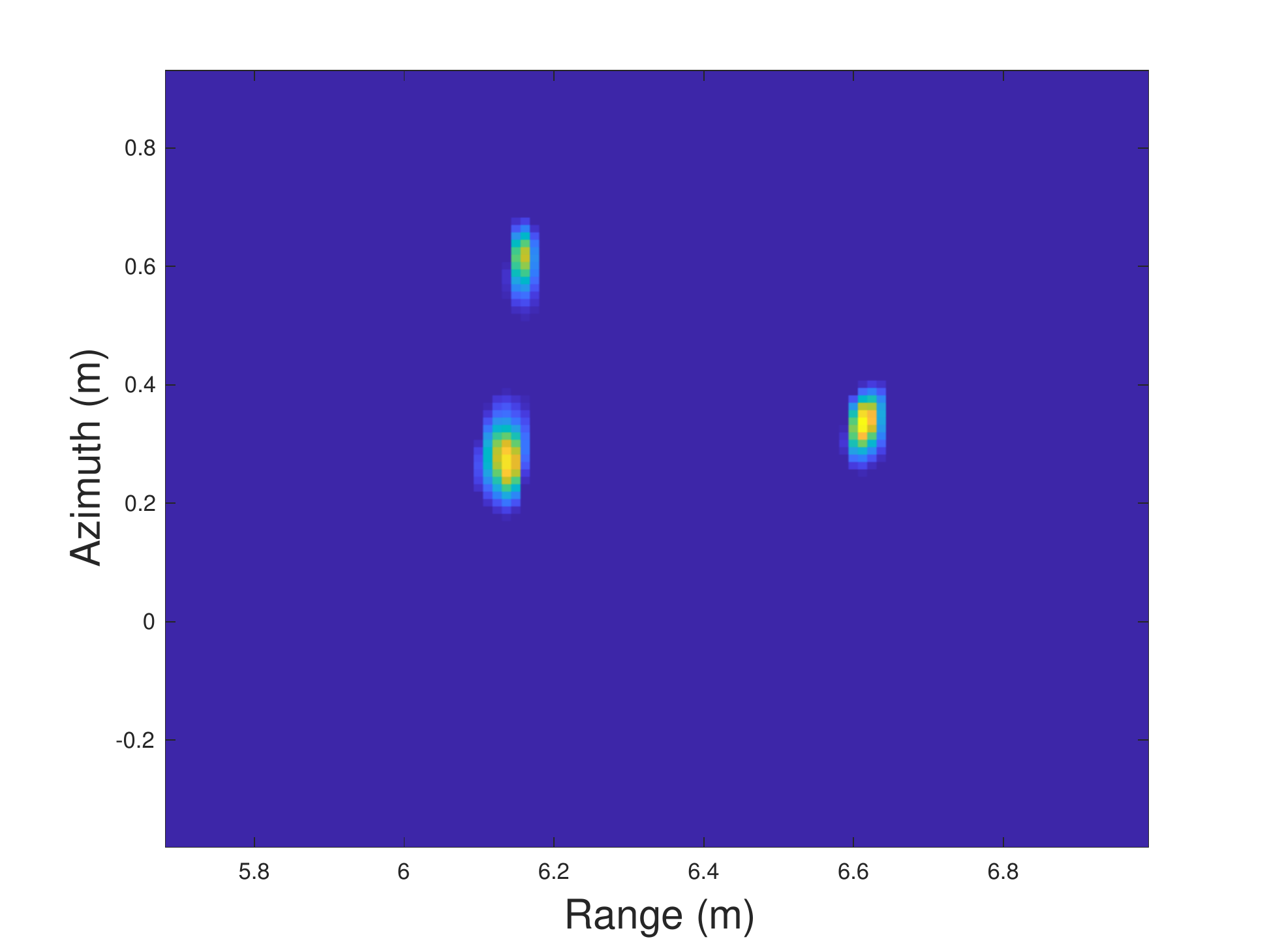}
}
\caption{\small Ground truth and reconstructed images from noisy measurements at 15dB PSNR. First row shows an illustration of the five different target layouts used in our simulations. Second row uses \emph{fused Lasso} reconstruction with the true antenna positions. Third row uses \emph{fused Lasso} reconstruction with wrong antenna positions. Fourth row is the solution to the measurement-domain blind deconvolution problem. Last row is the solution of our proposed method for the image-domain blind deconvolution problem.}\label{fig:ImageingResults}\vspace{-0.2in}
\end{figure*}
\begin{figure*}
\centering
\mbox{
\begin{subfigure}[b]{0.28\textwidth}
                \centering
                \includegraphics[width=\textwidth]{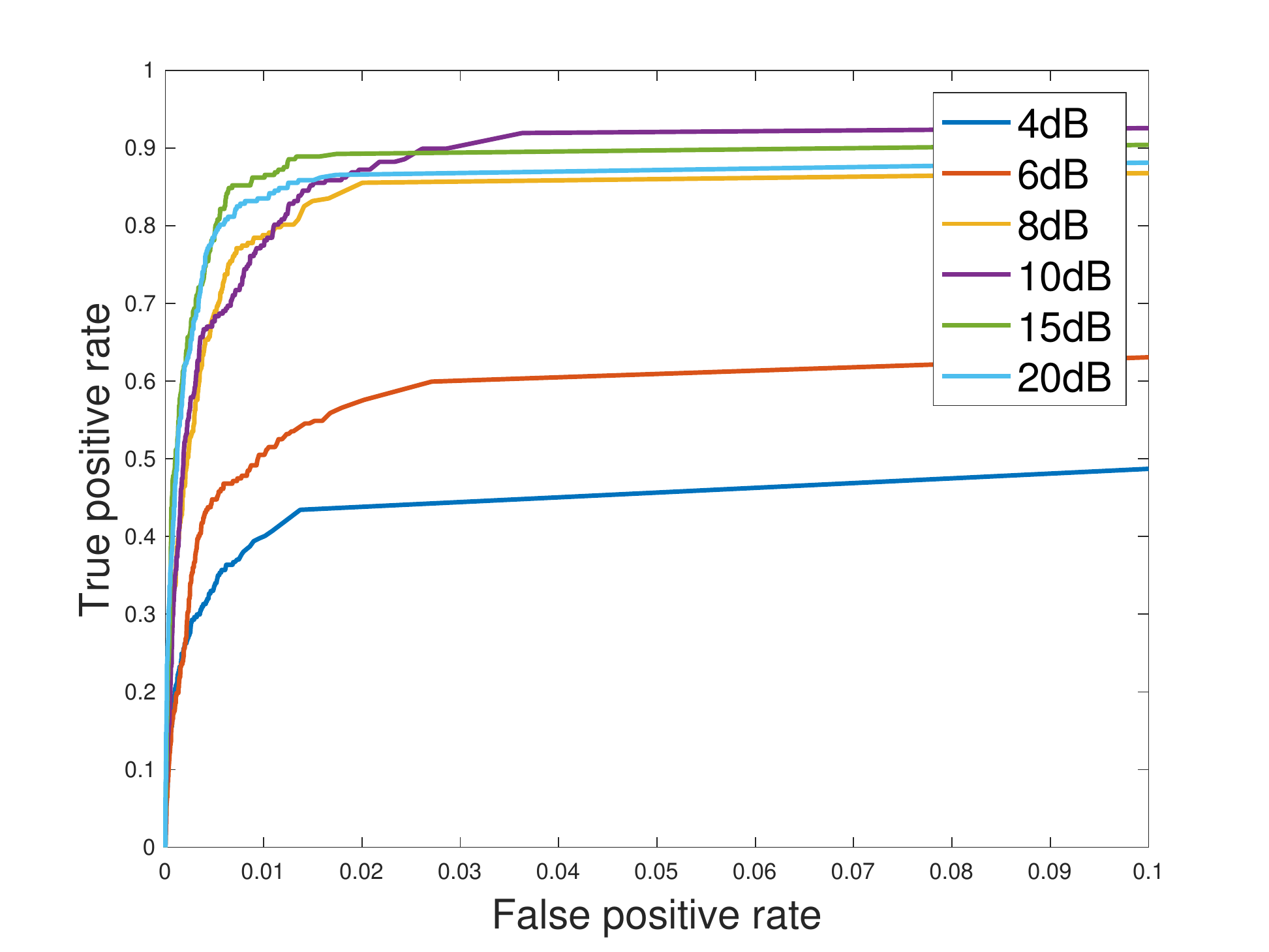}
                \caption{}
        \end{subfigure}%
        \begin{subfigure}[b]{0.28\textwidth}
                \centering
                \includegraphics[width=\textwidth]{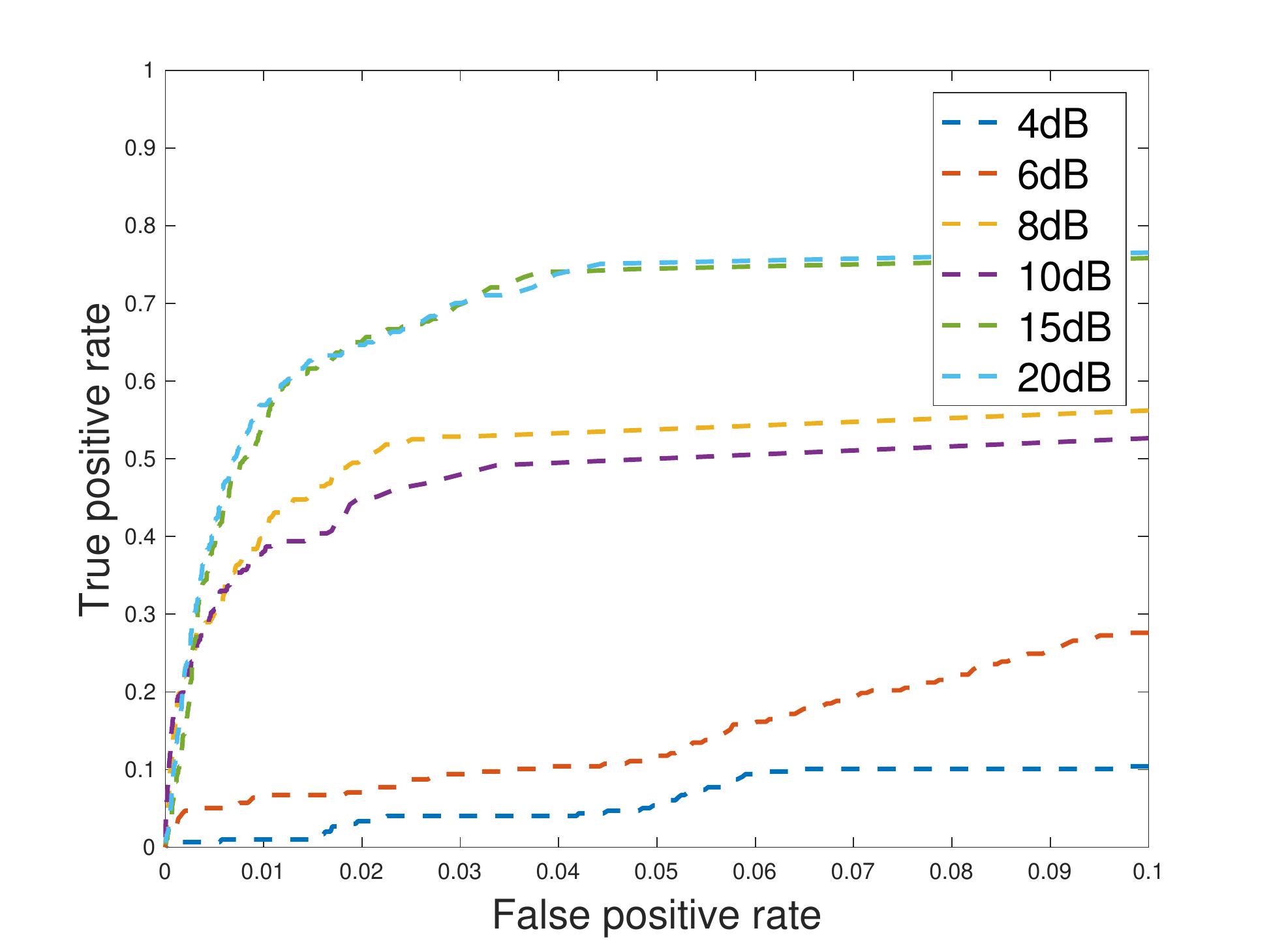}
                \caption{}
        \end{subfigure}%
        \begin{subfigure}[b]{0.28\textwidth}
                \centering
                \includegraphics[width=\textwidth]{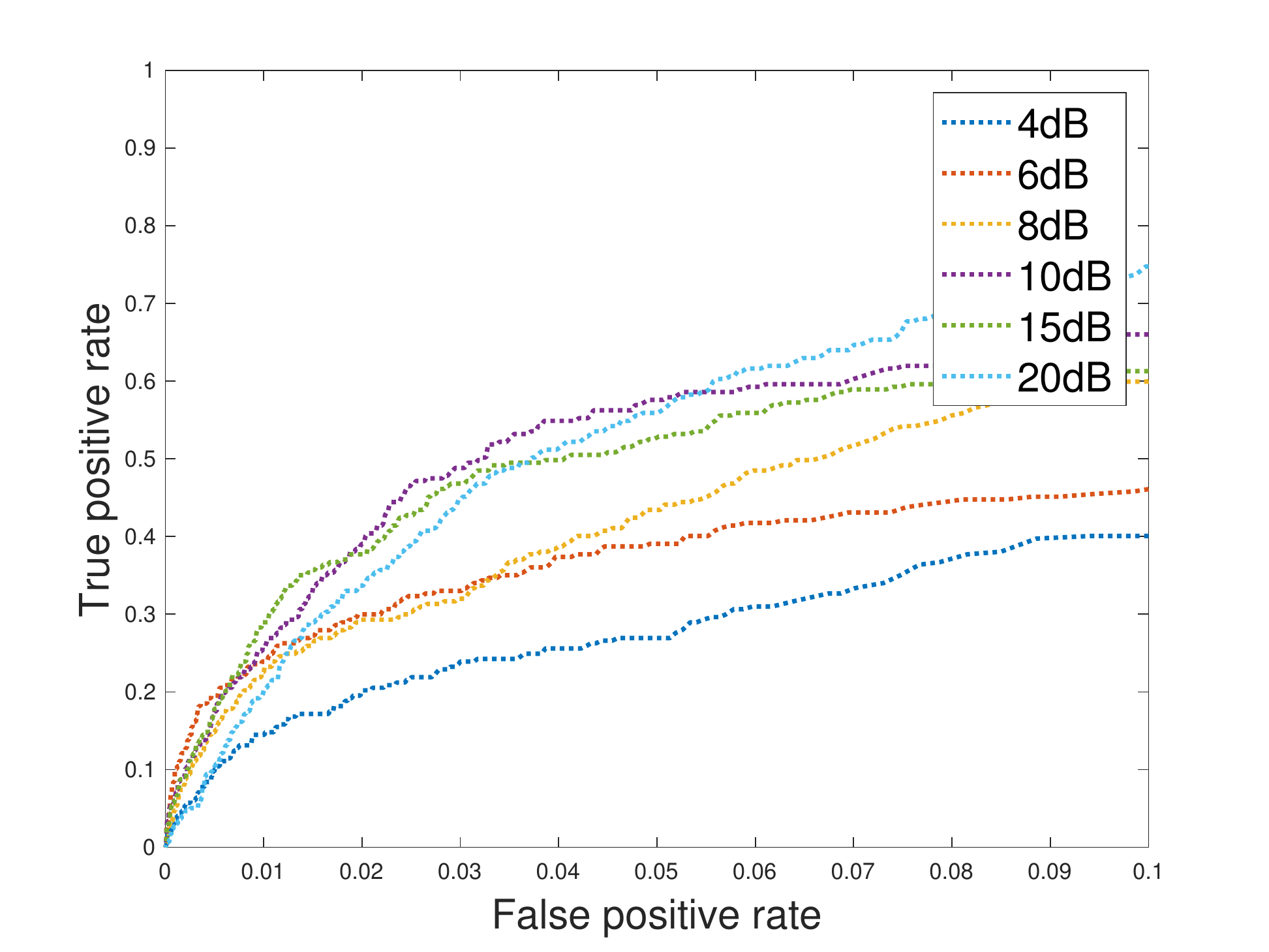}
                \caption{}
        \end{subfigure}%
}
\caption{\small Target detection ROC curves of the reconstructed images using (a) the proposed method, (b) solving the measurement-domain blind deconvolution problem, and (c) \emph{fused Lasso} recovery without compensating for position ambiguity.}\label{fig:Sims_ROC}
\end{figure*}

Next, we compare the robustness of our proposed approach to the state of the art iterative perturbation estimation scheme of~\cite{LKB:2016}. The method in~\cite{LKB:2016} leverages the sparsity of the radar scene as well as the proximity between consecutive antenna positions in order to estimate the antenna perturbations and consequently improve the reconstructed image quality. Since the scheme in~\cite{LKB:2016} relies on the antenna proximity to perform coherence analysis, we use additional measurements at antenna positions that interpolate the gaps between the $\times$'s for a total of 52 antenna positions per array. We compare the reconstruction performance in terms of the receiver operating characteristic (ROC) curves as shown in Figure~\ref{fig:ROC_52}. To generate the ROC curves, we simulate five different target positions as well as five different antenna perturbations and noise realizations. It can be seen from the figure that our proposed method is significantly more robust to measurement noise even at extreme noise levels. Finally, we recognize that the detection performance appears to be best for both methods at the 15dB PSNR level. We attribute this behavior to the limited number of noise realizations that have been used in our simulations which happened to be more favorable in the 15dB PSNR case.

\begin{figure}[h]
\centering
\mbox{
\begin{subfigure}[b]{0.25\textwidth}
                \centering
                \includegraphics[width=\textwidth]{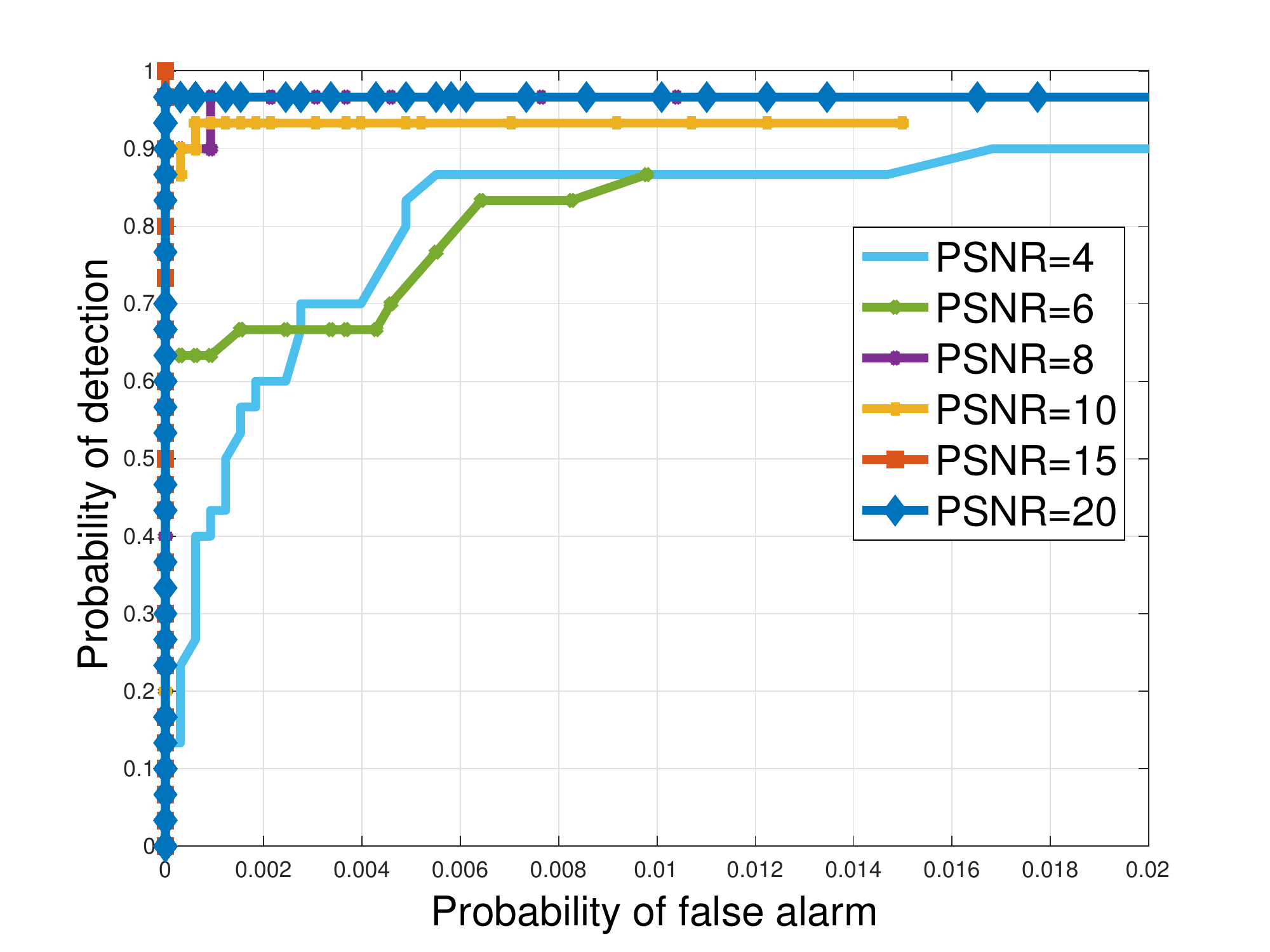}
                \caption{}
        \end{subfigure}%
        \begin{subfigure}[b]{0.25\textwidth}
                \centering
                \includegraphics[width=\textwidth]{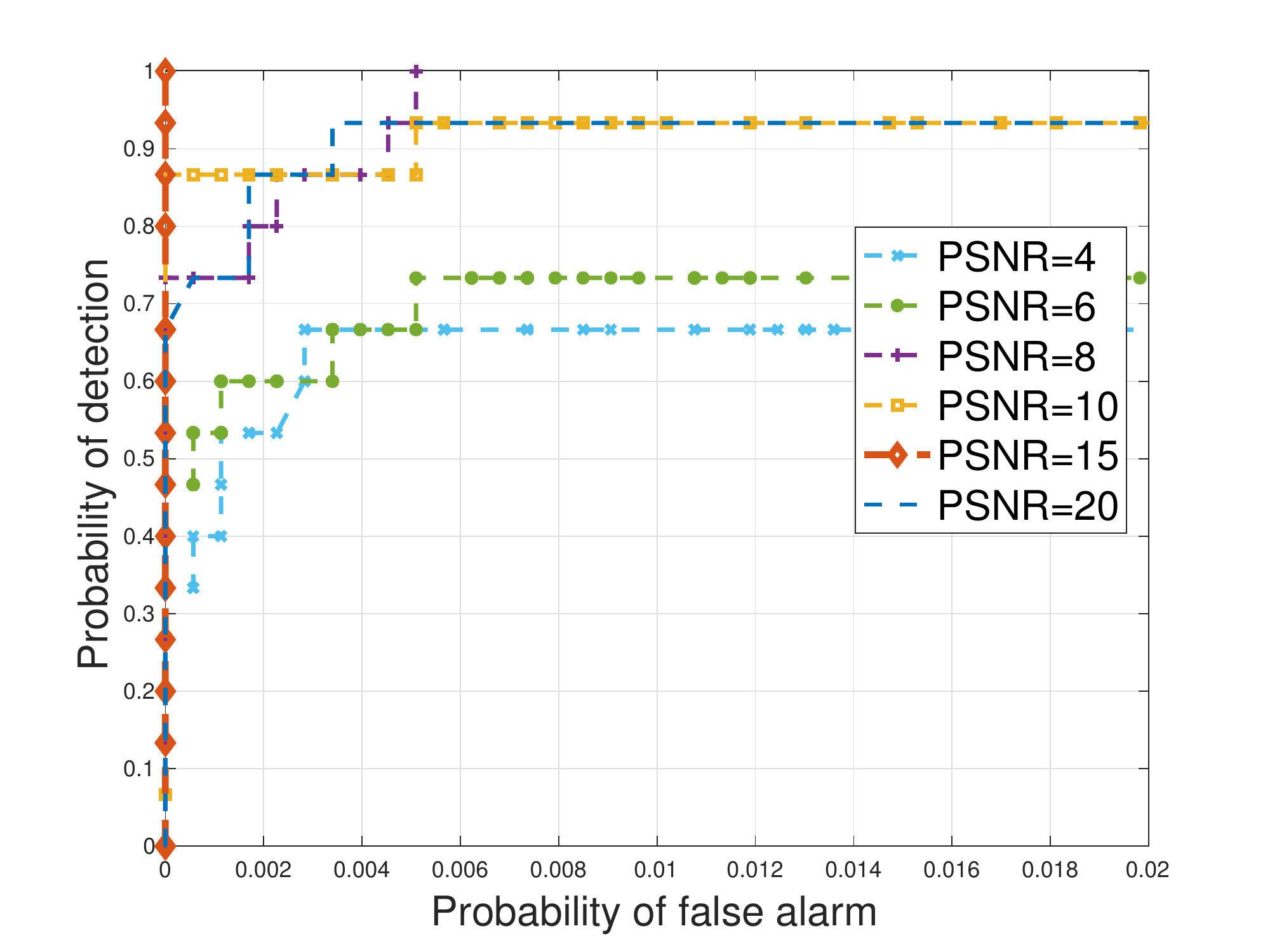}
                \caption{}
        \end{subfigure}%
}
\caption{\small ROC curves of the reconstructed images using (a) our proposed sparse blind deconvolution scheme, and (b) the iterative perturbation estimation scheme of~\cite{LKB:2016}.}\label{fig:ROC_52}
\end{figure}

\begin{figure}[h]
\centering
\mbox{
\begin{subfigure}[b]{0.25\textwidth}
                \centering
                \includegraphics[width=\textwidth]{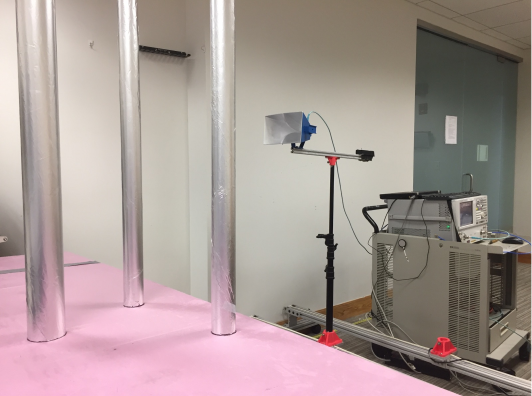}
                \caption{}
        \end{subfigure}%
 }
 \mbox{
        \begin{subfigure}[b]{0.3\textwidth}
                \centering
                \includegraphics[width=\textwidth]{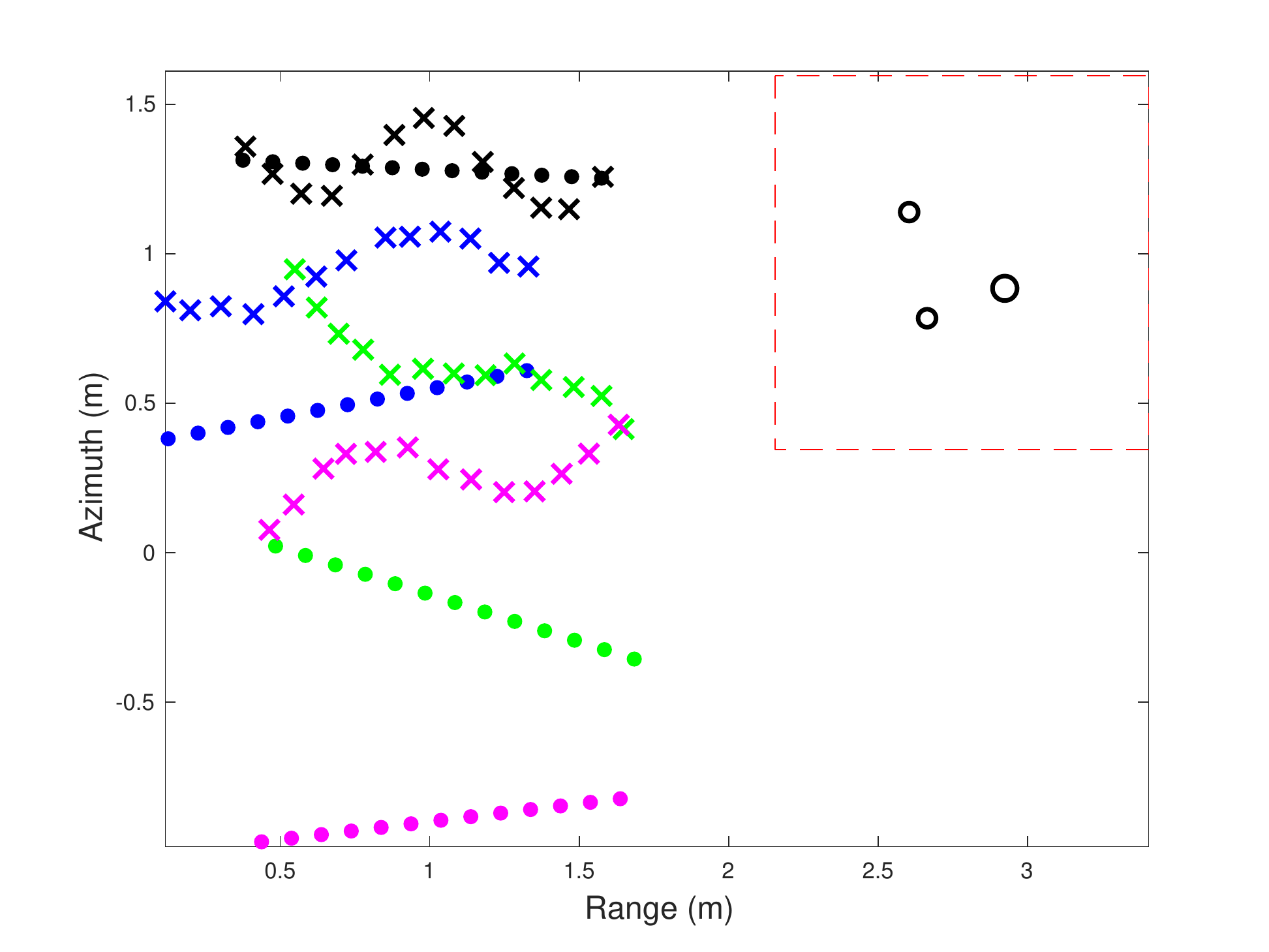}
                \caption{}
        \end{subfigure}%
        }
\caption{\small (a) Photograph of the horn antenna and three cylindrical targets used in our experimental setup. (b) Layout of the experimental distributed radar acquisition setup. The round dots indicate the assumed but erroneous antenna positions, while the $\times$'s indicate the true positions. The back circles in the top right corner indicate the target positions.}\label{fig:Exp_layout}\vspace{-0.2in}
\end{figure}

\subsection{Experimental data}
We built a radar setup using one horn antenna mounted on a platform that scans a scene that includes three cylindrical metal reflectors of diameter 6cm, 6cm, and 8.2 cm each, as shown in Figure~\ref{fig:Exp_layout} (a). The horn antenna was connected to a signal source, in this case an Agilent 2-Port PNA model 5230A, which measured the scattering parameters of the scene. The PNA was set to sweep over a frequency range from $1-10$ GHz with a 30 MHz frequency step and the port output power was set to 5 dBm. Over this range of frequencies the horn antennas have approximately a 40 degree main lobe beam width and a gain near 7 dBi. By moving the antenna positions and repeating the same experiment, we were able to collect radar measurements corresponding to four 13-element virtual arrays for a total of 52 antenna positions. A schematic of the experimental layout is shown in Figure~\ref{fig:Exp_layout} (b). The true antenna positions are indicated by the $\times$'s whereas the erroneous assumed positions are indicated by the dots. The horn of the antenna faces in the direction of the erroneous uniform linear array. With a 6 GHz center frequency and corresponding wavelength $\lambda_c = 5$cm, the maximum position error was equal to $1.28\lambda_c$ in the horizontal (range) direction and $3.43\lambda_c$ in the vertical (azimuth) direction for array 1. On the other hand, the position error for arrays 2, 3, and 4 was over $10.8\lambda_c$ in the vertical direction. The target locations are indicated by the black circles inside the region of interest marked by the dashed red line.

\begin{figure}[t]
\centering
\mbox{
\begin{subfigure}[b]{0.22\textwidth}
                \centering
                \includegraphics[width=\textwidth]{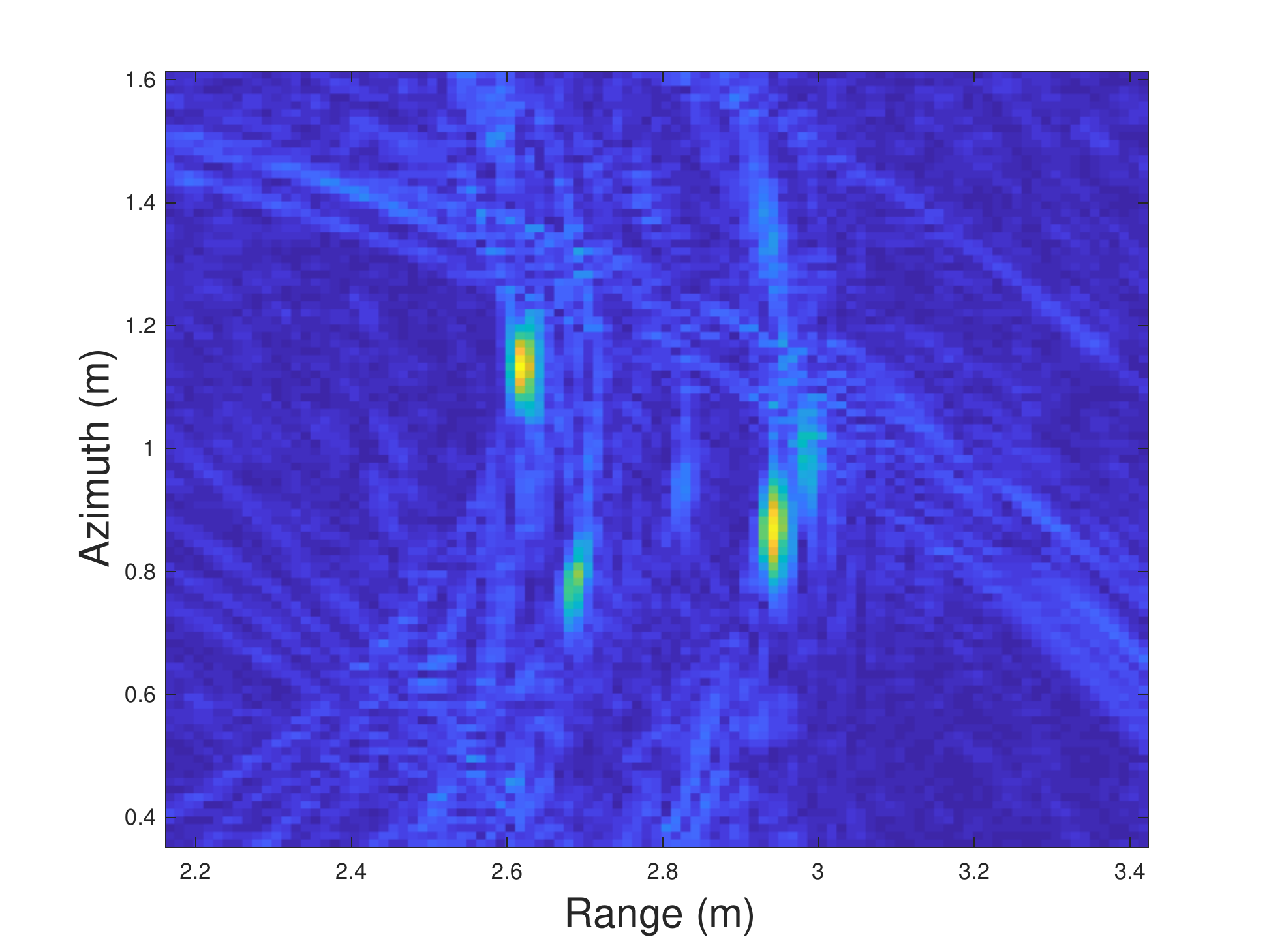}
                \caption{}
        \end{subfigure}%
        \begin{subfigure}[b]{0.22\textwidth}
                \centering
                \includegraphics[width=\textwidth]{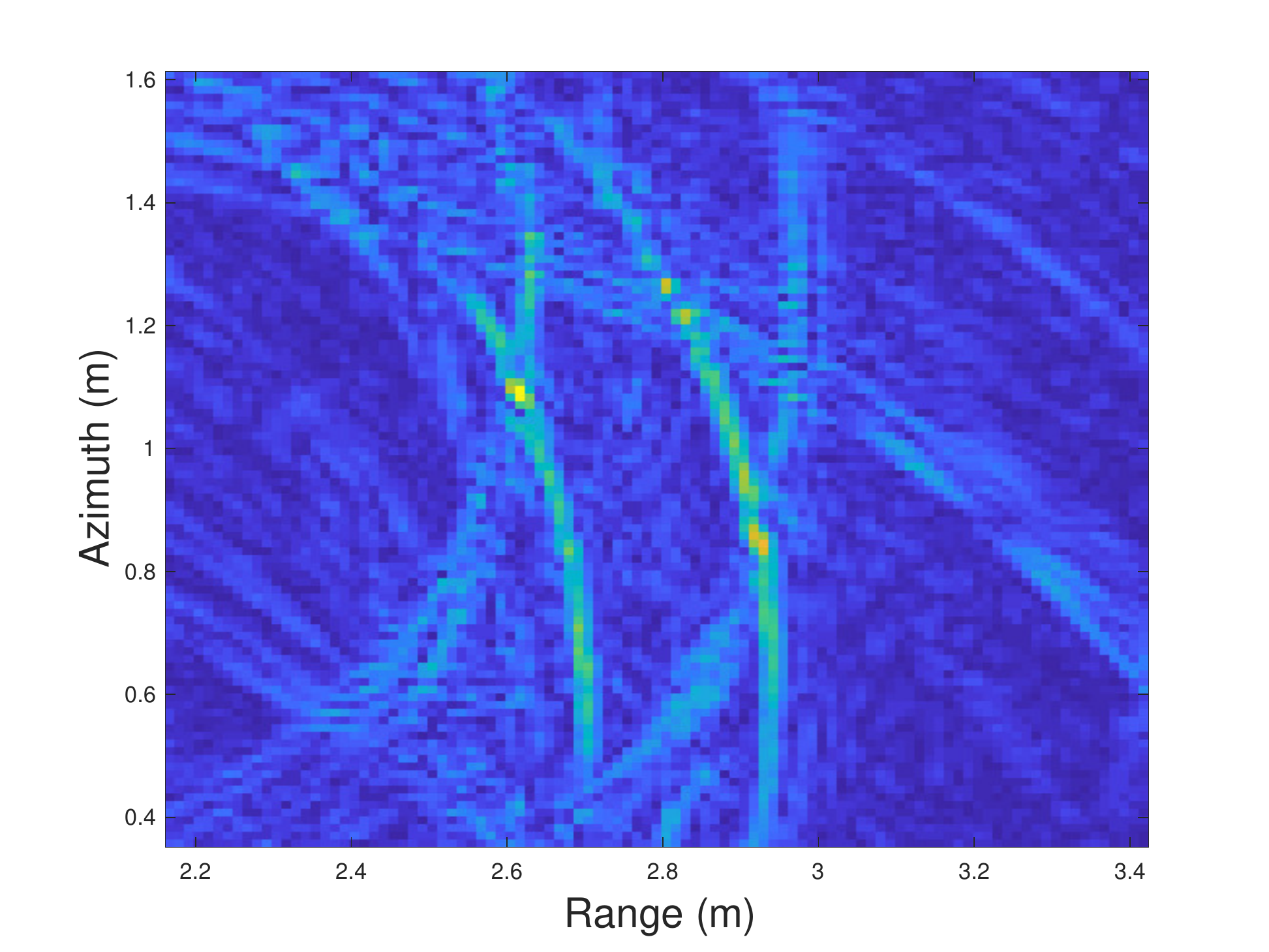}
                \caption{}
        \end{subfigure}%
}
\mbox{
\begin{subfigure}[b]{0.22\textwidth}
                \centering
                \includegraphics[width=\textwidth]{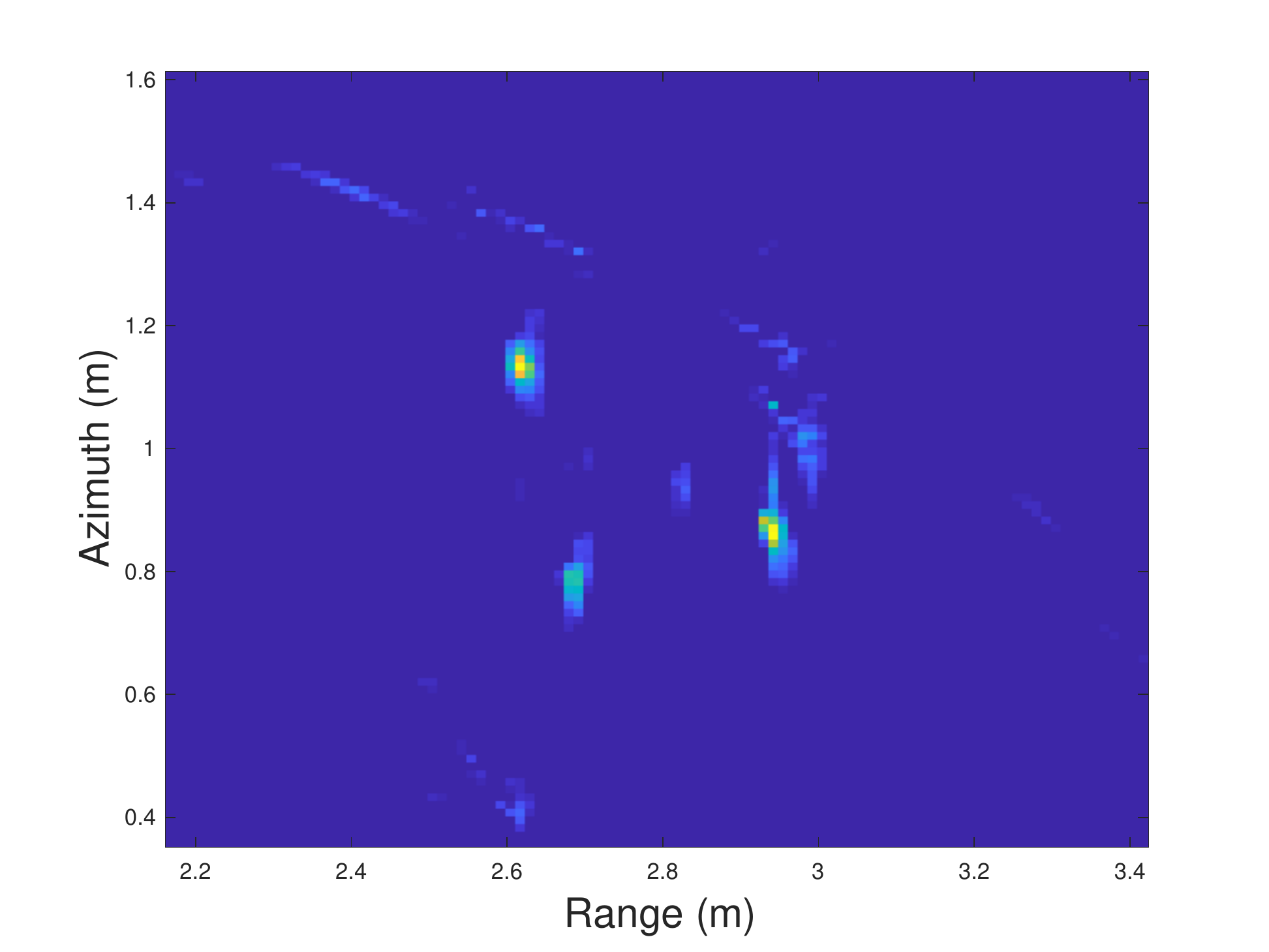}
                \caption{}
        \end{subfigure}%
        \begin{subfigure}[b]{0.22\textwidth}
                \centering
                \includegraphics[width=\textwidth]{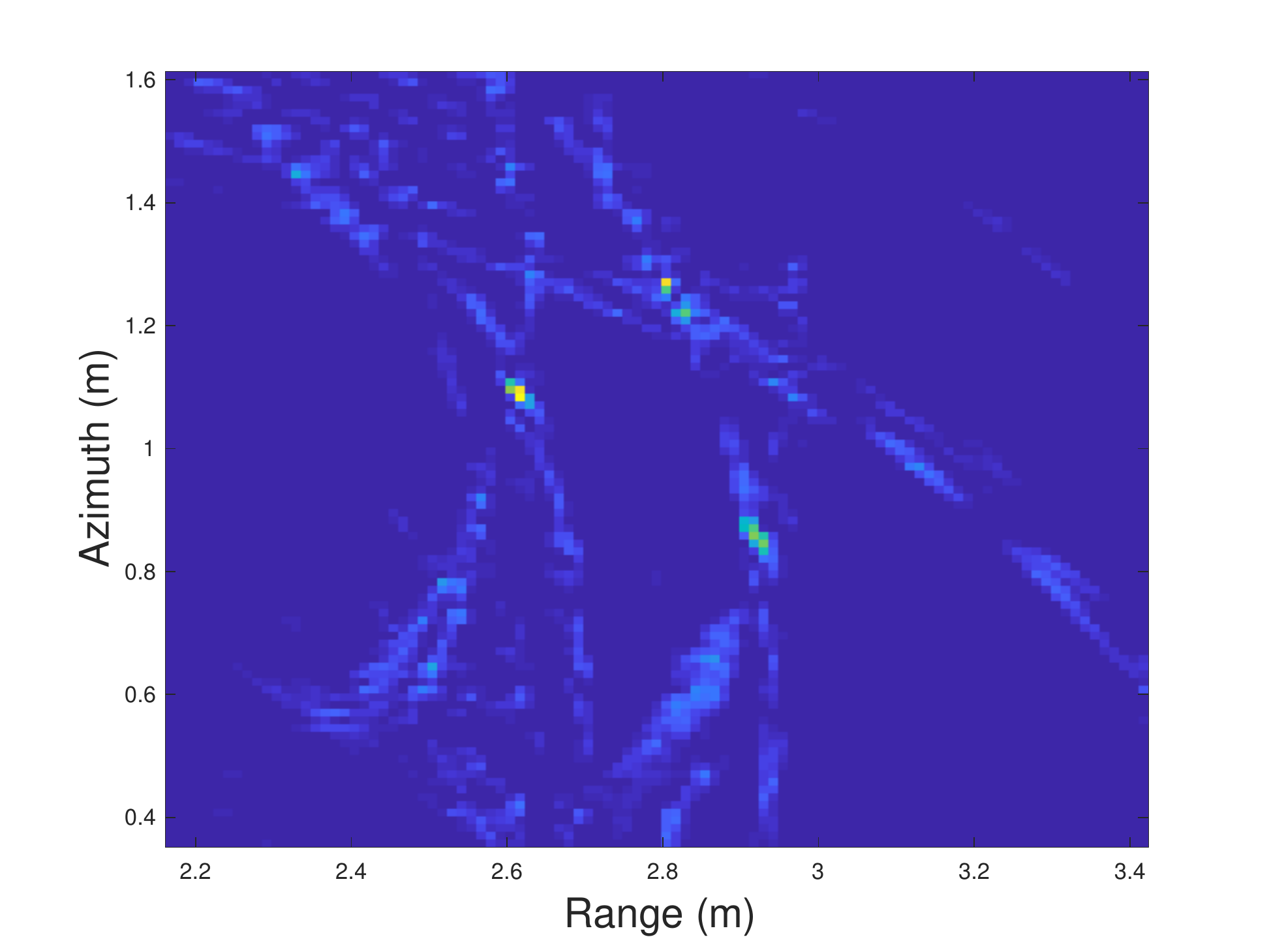}
                \caption{}
        \end{subfigure}%
}
\caption{\small Imaging results for the experimental data obtained from applying backprojection using (a) the correct imaging operator corresponding to the true antenna positions, and (b) the wrong imaging operator corresponding to the erroneous uniformly spaced positions. Results obtained using \emph{fused Lasso} regularized reconstruction with (c) the correct operator, and (d) the wrong operator.}\label{fig:Exp_conventional}\vspace{-0.2in}
\end{figure}
\begin{figure}[t]
\centering
\includegraphics[width=2.5in]{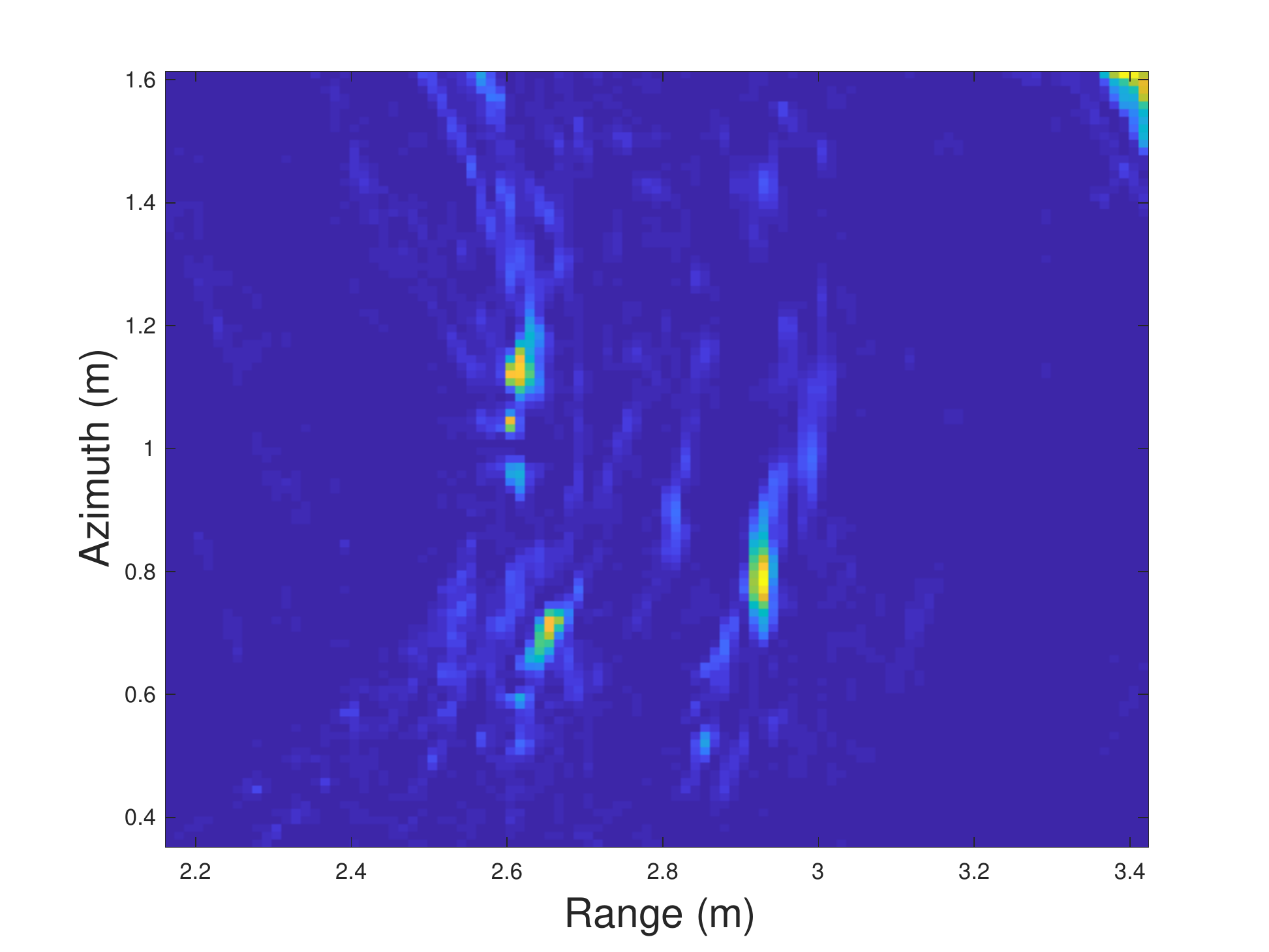}
\caption{\small Reconstructed image from the experimental data using the proposed blind deconvolution approach.}\label{fig:Exp_BlindDeconv}\vspace{-0.2in}
\end{figure}

\begin{figure*}[ht]
\centering
\includegraphics[width=6in]{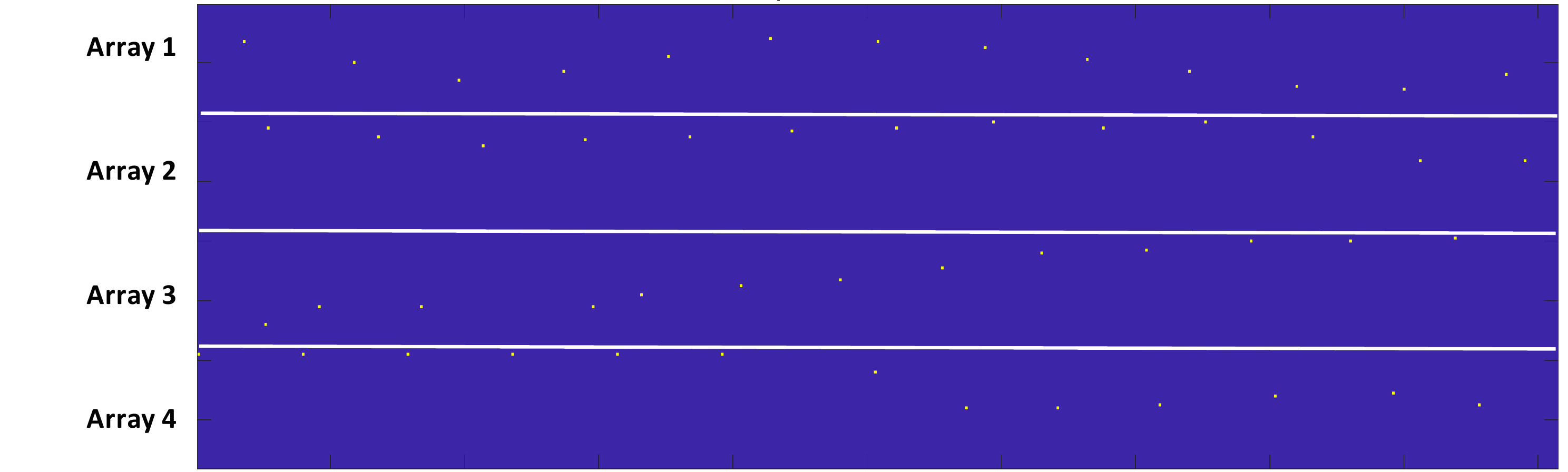}
\caption{\small Recovered convolution filters using the proposed blind deconvolution autofocus algorithm.}\label{fig:Exp_ConvFilters}\vspace{-0.25in}
\end{figure*}

We applied our blind deconvolution method to recover the radar image with parameters $\gamma = 0.5$, $\mu = 0.07$ and $\sigma = 1.2$. The radar image is discretized into a $101\times 101$ image with pixel resolution equal to $\lambda_c/4 = 1.25$cm. We also set the size of the convolution filters equal to $39\times 39$ pixels which is larger than the perturbations of array 1 but smaller than the perturbation of arrays 2, 3, and 4. Figure~\ref{fig:Exp_conventional} (a)--(d) shows the reconstructed images produced by the conventional backprojection scheme and a standard \emph{fused Lasso} regularized reconstruction which do not compensate for the position ambiguity. Notice that when the wrong antenna positions are used to build the radar operator, the quality of the reconstructed imaged degrades significantly. On the other hand, our proposed framework is capable of generating a focused radar image in Figure~\ref{fig:Exp_BlindDeconv} and compensates for the antenna perturbation by computing the spatial convolution filters shown in Figure~\ref{fig:Exp_ConvFilters}. It is quite remarkable how the shift kernels for arrays 1 and 2 are consistent with the true antenna perturbations. The shift kernels of arrays 3 and 4 are only mildly following the antenna perturbation due to the large perturbation in the true antenna position compared to the erroneous assumed positions.\vspace{-0.15in}

\section{Conclusion}

We developed a novel image-domain blind deconvolution framework for
recovering focused images from radar measurements that suffer from
position ambiguity. Contrary to existing convex approaches to blind
deconvolution, the proposed image-domain convolution model places the
linear measurement operator after applying the convolution
operation. We showed that this resulting model is exact in the setting
when the transmitter and receiver antennas are affected by the same
position error. To recover the scene, in the context of this
framework, we also developed a block coordinate descent algorithm that
alternates between recovering the target image under \emph{fused
  Lasso} penalty constraints, and estimating the sparse convolution
kernels of the antennas. The performance of the proposed formulation
was shown to be superior to state-of-the-art methods that have
addressed the radar autofocus problem. Future efforts will address
deriving a convex formulation of the image-domain blind deconvolution
problem.\vspace{-0.1in}


\bibliographystyle{IEEEtran}
\bibliography{refs}

\end{document}